\documentclass[12pt]{article}
\pdfoutput=1
\RequirePackage[algo2e,ruled]{algorithm2e}

\usepackage{mathtools}
\usepackage{amsthm}
\usepackage{array}
\usepackage{relsize}
\usepackage{amssymb}
\usepackage{amsmath}
\usepackage{comment}
\usepackage{tabularx}
\usepackage[margin=1in]{geometry}

\providecommand{\norm}[1]{\lVert#1\rVert}

\newcommand\numberthis{\addtocounter{equation}{1}\tag{\theequation}}
\usepackage{booktabs}
\usepackage{xr-hyper}
\newcommand{\refgen}[1]{~\ref*{general-#1} in \cite{schiffer2024stronger}}
\newcommand{\eqrefgen}[1]{~(\ref*{general-#1}) in \cite{schiffer2024stronger}}

\usepackage{hyperref}

\newtheorem{thm}{Theorem}
\newtheorem{problem}{Problem}

\newtheorem{assum}{Assumption}

\newtheorem{lemma}{Lemma}
\newtheorem{definition}{Definition}
\newtheorem{proposition}[thm]{Proposition}
\newtheorem{remark}[thm]{Remark}
\usepackage{xcolor}        
\usepackage{hyperref}
\usepackage{natbib}
\DeclareMathOperator*{\argmax}{arg\,max}
\DeclareMathOperator*{\argmin}{arg\,min}

\usepackage{algorithm}
\usepackage[noend]{algpseudocode}

\renewcommand{\P}{\mathbb{P}} 
\usepackage{amsmath}
\usepackage{amssymb}
\usepackage{amsfonts}   

\DeclareMathOperator*{\E}{\mathbb{E}}

\newcommand{\cond}{\; \middle| \;}

\pdfoutput=1

\title{Foundations of Safe Online Reinforcement Learning in the Linear Quadratic Regulator: $\sqrt{T}$-Regret}

\author{
	Benjamin Schiffer 
 and
    Lucas Janson 
}

\date{Department of Statistics, Harvard University}

\begin{document}

\maketitle
\begin{abstract}

Understanding how to efficiently learn while adhering to safety constraints is essential for using online reinforcement learning in practical applications. However, proving rigorous regret bounds for safety-constrained reinforcement learning is difficult due to the complex interaction between safety, exploration, and exploitation. In this work, we seek to establish foundations for safety-constrained reinforcement learning by studying the canonical problem of controlling a one-dimensional linear dynamical system with unknown dynamics. We study the safety-constrained version of this problem, where the state must with high probability stay within a safe region, and we provide the first safe algorithm that achieves regret of $\tilde{O}_T(\sqrt{T})$. Furthermore, the regret is with respect to the baseline of truncated linear controllers, a natural baseline of non-linear controllers that are well-suited for safety-constrained linear systems. In addition to introducing this new baseline, we also prove several desirable continuity properties of the optimal controller in this baseline. In showing our main result, we prove that whenever the constraints impact the optimal controller, the non-linearity of our controller class leads to a faster rate of learning than in the unconstrained setting.
\end{abstract}

\section{Introduction}

\subsection{Background and Motivation}
Online reinforcement learning (RL) algorithms are powerful tools for interacting with and learning about unknown environments \citep{levine2016end, lillicrap2015continuous,tewari2017ads}. The core idea behind many successful RL algorithms is carefully balancing exploration and exploitation. However, in many real world applications, online RL algorithms must satisfy a set of safety constraints. Importantly, these safety constraints must be satisfied even while the algorithm learns, leading to a complex interaction between safety and learning. Safety constraints reduce an algorithm's ability to explore because the algorithm must take actions that are known to be safe. Similarly, safety constraints reduce an algorithm's ability to exploit because actions that exploit known information may lead to unsafe states. As an example, consider a self-driving car that uses online RL to learn how to navigate a new environment in real time. To do this, an RL algorithm must make adjustments to speed and acceleration that account for unknown environmental factors such as wind speed and friction. However, the algorithm controlling a car in the real world must keep the car in safe states and avoid crashing into other objects. Therefore, it is critical that the algorithm learns while being safe. A better understanding of the relationship between learning and safety constraints is crucial for deploying online reinforcement learning algorithms in the real world. In this paper, we focus on understanding how safety and learning interact for a canonical learning problem in control theory known as \emph{online linear quadratic regulator (LQR) learning}. While online LQR learning is one of the simplest learning problems with a continuous action space, this problem highlights the inherent differences between learning without safety constraints and learning with safety constraints.

\subsection{Setting and Motivation}

In this paper, we study the problem of learning and controlling a discrete-time linear dynamical system when the dynamics of the system are unknown and safety must be maintained during online learning. At each time step, the algorithm observes the current state and chooses a control (action). The state at the next time step then depends on the current state, the chosen control, and random noise. The way in which the next state depends on the current state and chosen control is referred to as the \emph{dynamics of the system}. The goal of the problem is to choose actions that minimize a quadratic cost by keeping the state close to the origin while using minimal control. This model is used, e.g., in the field of robotics when a robot (drone, submarine, rocket, etc.) attempts to stay close to a single point while being subject to random environmental forces \citep{rubio2016optimal}. In practice, the dynamics of the system (such as air resistance) are not known a priori. Therefore, we study this problem when the dynamics are unknown, and the algorithm must minimize cost while learning the unknown dynamics. To model safety in this setting, we assume that the state must stay within a predefined `safe region.' For example, the robot described above cannot move to states that make the robot crash into other objects.

When the dynamics are known and there are no safety constraints, the optimal algorithm is the Linear Quadratic Regulator, which is well-studied in the field of control theory \citep{rawlings2012postface}. However, the addition of state constraints significantly complicates even this simple problem, and there no longer exists a closed-form solution for the constrained version of this problem with known dynamics \citep{rawlings2012postface}. In order to make the problem more tractable, we study this problem when both states and controls are one-dimensional; \citet{schiffer2024stronger} take the same approach in analyzing the one-dimensional constrained linear systems. One-dimensional linear systems have been frequently studied as a first step toward understanding other complex aspects of control theory, see e.g. \citet{fefferman2021optimal, abeille2017thompson}. Furthermore, some real-world problems can be represented as one-dimensional LQR problems. As an example, consider the simple setting of controlling the temperature of a room, a common problem in control \citep{oldewurtel2008tractable}. The possible actions include adding different amounts of hot air or cold air to the room, and a natural goal is to minimize costs (the amount of energy used) while also keeping the room close to a specific temperature. In this setting, state constraints would consist of constraining the temperature to stay within a `safe' region of temperatures that are not too hot and not too cold.

\subsection{Our Contribution}
The overarching goal of this paper is to provide foundations for analyzing safety-constrained LQR learning using non-linear baselines of controllers that are better suited for the constrained problem. Our main result is the first algorithm for safety-constrained one-dimensional LQR with unknown dynamics that with high probability guarantees $\tilde{O}_T(\sqrt{T})$ regret. In this setting, our work improves upon the previous best regret results, in particular \citet{li2021safe, dean2019safely} prove $\tilde{O}_T(T^{2/3})$ regret bounds and only for bounded noise distributions.

The rate of $\tilde{O}_T(\sqrt{T})$ matches the optimal rate of regret in the unconstrained LQR learning problem. Note that unconstrained LQR learning is a special case of constrained LQR learning with sufficiently loose constraints. Therefore, because the lower bound for unconstrained LQR learning is $\tilde{\Omega}_T(\sqrt{T})$ regret \citep{ziemann2024regret}, it is impossible to in general do better than $\tilde{\Omega}_T(\sqrt{T})$ regret for the constrained problem. In addition to improving the rate of regret, the $\tilde{O}_T(\sqrt{T})$ regret is also with respect to a stronger baseline than previous works. More specifically, the regret is defined with respect to the best controller from the baseline class of truncated linear controllers, which consists of linear controllers corrected to obey the safety constraints.  This baseline is naturally well-suited for safe control and is a significantly stronger baseline than studied in previous works (see Section \ref{sec:baseline_class} for more details). Because the controllers in this class are frequently non-linear, we also introduce new theoretical tools for analyzing this type of non-linear controller. Therefore, our $\tilde{O}_T(\sqrt{T})$ regret result is strictly stronger than the previous $\tilde{O}_T(T^{2/3})$ regret results of \citet{li2021safe, dean2019safely} applied to our setting, in both the regret baseline and the rate of regret. Note that these previous works also assume bounded noise distribution, while our results hold for any sub-gaussian distribution.

Informally, our main theorem can be stated as follows:

\begin{thm}[Informal]\label{thm:informal}
    For safety-constrained one-dimensional LQR with unknown dynamics and any sub-gaussian noise distribution, there exists an algorithm that with high probability is safe and has regret of $\tilde{O}_T(\sqrt{T})$ compared to the best truncated linear controller with known dynamics.
\end{thm}
To prove Theorem \ref{thm:informal}, we show that either the constraints are tight enough to give faster learning rates or loose enough that the problem is approximately unconstrained. This dichotomy is the main conceptual idea behind our algorithm being able to achieve $\tilde{O}_T(\sqrt{T})$ regret for all possible noise distributions. We also show that the class of truncated linear controllers satisfies multiple desirable continuity properties, which may be of independent interest.

\subsection{Related Work}\label{sec:related_work}

Safe reinforcement learning has been studied in many different contexts with  various definitions of safety, including reachability of safe sets and long term stability \citep{ganai2024iterative, garg2024learning, gu2022review, moldovan2012safe, wachi2018safe, wachi2024survey, yao2024constraint}. Specifically in control theory, there exist many methods that satisfy different notions of safety for specific control tasks \citep{fulton2018safe, cheng2019end, marvi2021safe, fisac2018general}. The work on safe RL in control has primarily focused on feasibility of safety, i.e. providing an algorithm that satisfies notions of safety such as returning to safe sets or remaining stable around the origin. However, these works do not study the theoretical regret analysis of their algorithms, and therefore do not provide bounds on how much worse the algorithm is when compared to the optimal baseline. Another general line of work that is related but less directly comparable to our results is the area of model predictive control and system identification \citep{bemporad2007robust, kohler2019linear, lu2021robust, oldewurtel2008tractable, mesbah2016stochastic, bemporad2002explicit, muthirayan2022online,lorenzen2019robust, simchowitz2018learning, zhao2022adaptive, mania2019certainty, li2023non}. However, the results in these areas tend to focus more on feasibility and empirical performance rather than theoretical regret bounds, and therefore are less directly related to our work.

The LQR learning problem has recently gained significant attention after \citet{abbasi2011regret} showed that $\tilde{O}_T(\sqrt{T})$ regret is possible in the unconstrained LQR learning problem. Subsequent works have built on these results with many variations and more efficient algorithms \citep{dean2018regret, mania2019certainty, mania2020active, simchowitz2018learning,cohen2019learning,wang2021exact,wang2022rate,mania2019certainty,abeille2017thompson,zheng2020non,sun2020finite,khosravi2020nonlinear,sattar2022non,faradonbeh2018input,faradonbeh2017finite,oymak2019non,ye2024online,athrey2024regret,ziemann2024regret,lee2024nonasymptotic}. One particular result from this line of work that we want to highlight is that certainty equivalent estimation gives the asymptotically best rate of regret for the LQR learning problem \citep{simchowitz2020naive,faradonbeh2018optimality,mania2019certainty,wang2022rate}. Certainty equivalence algorithms consist of estimating the true dynamics and finding the optimal controller for these estimated dynamics. Our main algorithm uses a certainty equivalence approach to achieve the same rate of regret in the safety-constrained LQR setting.

Less closely related to this paper, there is also a line of work studying optimal control with adversarial disturbances, where the goal is still to minimize regret but the system dynamics are known (see e.g. \citep{agarwal2019online, hazan2022introduction}). \cite{li2021online} also study optimal constrained control but again assume that the dynamics are known. The techniques and results of these lines of work with known dynamics are substantially different from our paper. This is because the key difficulty of our problem is that we do not know how to be safe apriori and must be safe while learning, which is not an issue with known dynamics.

Two previous works have focused on regret bounds for variants of the constrained LQR learning problem. \citet{dean2019safely} and \citet{li2021safe} both consider the problem of constrained LQR learning specifically with bounded noise distributions. These works both give algorithms that achieve $\tilde{O}_T(T^{2/3})$ regret for this problem, and their regret results are with respect to the baseline of the best safe linear controller. While the results in these works hold in higher dimensions, our work improves on these results in two ways. The first is that the regret rate we achieve is with respect to the baseline of the best truncated linear controller, which is a strictly stronger (and often significantly stronger) baseline than the best safe linear controller. Furthermore, the regret of our algorithm is $\tilde{O}_T(T^{1/2})$.

This paper is the second part of a two part series of papers on safe LQR learning. The first part of this series \citep{schiffer2024stronger} provides more general results but with weaker regret bounds that apply to any baseline class of controllers satisfying a set of assumptions. Specifically, \citet{schiffer2024stronger} shows that for any baseline class of controllers satisfying certain natural but abstract assumptions, it is possible to achieve $\tilde{O}_T(T^{2/3})$ regret with respect to that baseline. That paper also shows that $\tilde{O}_T(\sqrt{T})$ is possible for such a baseline in the special case when the noise distribution has sufficiently large support. Importantly, however, that paper does not provide any concrete examples of baselines satisfying its assumptions as doing so would have rendered its appendix unreadably long. This paper establishes just such a concrete example of a baseline class that satisfies the assumptions of \citet{schiffer2024stronger}, namely, the class of truncated linear controllers. This class of controllers is well-adapted to safe LQR yet, due to its nonlinearity, presents a number of significant technical challenges (see Appendices \ref{app:parameterization_assum3} and \ref{app:paramterization_assum2}). Furthermore, Theorem \ref{trunc_thm} is a strictly stronger result than those in \citet{schiffer2024stronger} for truncated linear controllers, and Algorithm \ref{alg:cap3} requires a number of new technical ideas and tools that are specific to the class of truncated linear controllers (see Section \ref{sec:theoretical_results} and Appendix \ref{sec:trunc_lin_sqrtt}).

\section{Preliminaries}

\subsection{Dynamics and Cost}

Let $T$ be the number of steps. For $t \in [T]$, we denote the state at time $t$ as $x_t \in \mathbb{R}$ and the control at time $t$ as $u_t \in \mathbb{R}$. Unless otherwise stated, we let $x_0 = 0$. Denote the (unknown) dynamics of the system as $\theta^* = (a^*,b^*) \in \mathbb{R}^2$. Then the state at time $t+1$ is $x_{t+1} = a^*x_t + b^*u_t + w_t$, where $w_t \stackrel{\text{i.i.d.}}{\sim} \mathcal{D}$ and $\mathcal{D}$ is a known continuous distribution with mean $0$, variance $\sigma_{\mathcal{D}}^2 = 1$, cumulative distribution function $F_{\mathcal{D}}$, and bounded probability density function $f_{\mathcal{D}}$ (bounded by constant $B_P$). Note that the assumptions that the noise distribution is mean $0$ and sub-gaussian are standard in LQR learning \citep{abbasi2011regret,li2021safe,dean2019safely}. The assumption of unit variance is made only for expositional simplicity, and our main results still hold for noise distributions with arbitrary variances. Define $W = \{w_t\}_{t=0}^{T-1}$ as the set of noise random variables for the $T$ steps. The goal of the algorithm is to minimize the total cost over all $T$ steps, where the cost at time $t$ is $qx_t^2 + ru_t^2$ for known $q, r \in \mathbb{R}_{>0}$. 

A controller $C$ at time $t$ chooses a control $u_t = C(H_t)$, where $H_t$ is the history up to time $t$ and is defined as $H_t := (x_0,u_0,...,u_{t-1}, x_t)$. The average cost over $T$ steps for controller $C$ starting at state $x_0$ under dynamics $\theta$ is defined as
\begin{equation}\label{eq:cost_def}
J(\theta, C, T, x_0, W)  := \frac{1}{T}\left(qx_T^2 +  \sum_{t=0}^{T-1} qx_t^2 + ru_t^2\right),
\end{equation}
\[
 \text{where }  u_t = C(H_t),\;\,  x_{t+1} = ax_t + bu_t + w_t,\;\, w_t \stackrel{i.i.d.}{\sim} \mathcal{D}.
\]
$J(\cdot)$ is an average cost, and therefore the total cost over $T$ steps of controller $C$ is $T \cdot J(\theta, C, T, x_0, W)$. We also define the expected cost of controller $C$ as $J^*(\theta, C, T, x_0) = \E[J(\theta,C,T,x_0,W) \mid \theta,C,T,x_0]$. Finally, for ease of notation we define $J^*(\theta,C,T) = J^*(\theta,C,T,0)$.

\subsection{Safety Constraints}\label{sec:constraints}

As described in the introduction, the key difficulty of our problem is learning the unknown dynamics efficiently while maintaining safety. In this paper, we formulate safety as constraints on the expected state. In this section, we formally introduce our safety definition and show that our definition is \emph{strictly more general} than the safety definitions studied in previous works \citep{li2021safe, dean2019safely}. More specifically, when the noise distribution is bounded, our safety definition is equivalent to the safety definition in \cite{li2021safe,dean2019safely}. However, our safety definition can generalize to unbounded noise distributions unlike the safety definitions in \citep{li2021online,dean2019safely}.

Because $w_t$ is a mean-$0$ random variable, we know that the conditional expectation of the next state given the current state and control is $\E[x_{t+1} \mid x_t, u_t] = a^*x_t + b^*u_t$. The safety constraints as defined in Definition \ref{def:safe} constrain this expected state to always stay within a known safe region between $D_{\mathrm{L}}^{\E[x]}$ and $D_{\mathrm{U}}^{\E[x]})$.  

\begin{definition}\label{def:safe}
    A series of controls $\{u_t\}_{t=0}^{T-1}$ are \textit{safe} for dynamics $\theta^*$ and boundaries $(D_{\mathrm{L}}^{\E[x]},D_{\mathrm{U}}^{\E[x]})$  if for all $t$, 
    \begin{equation}\label{eq:intro_safety}
    D_{\mathrm{L}}^{\E[x]} \le a^*x_t + b^*u_t \le D_{\mathrm{U}}^{\E[x]}.
\end{equation}
    Similarly, a controller $C$ is \textit{safe} for dynamics $\theta^*$ and boundaries $(D_{\mathrm{L}}^{\E[x]},D_{\mathrm{U}}^{\E[x]})$ if the resulting controls $\{C(H_t)\}_{t=0}^{T-1}$ under true dynamics $\theta^*$ are safe for dynamics $\theta^*$.
\end{definition}
\begin{assum}\label{assum:constraints}
    The safety constraint boundaries satisfy that $D_{\mathrm{L}}^{\E[x]} < 0 < D_{\mathrm{U}}^{\E[x]}$, that $D_{\mathrm{L}}^{\E[x]}, D_{\mathrm{U}}^{E[x]} = O_T(1)$, and that $D^{\E[x]}_{\mathrm{U}} - D_{\mathrm{L}}^{\E[x]} \ge \frac{1}{\log(T)}$. 
\end{assum}
The assumptions that the origin is in the safe set and that the boundaries are bounded above by constants are standard for safety-constrained LQR learning \citep{li2021safe, dean2019safely}.

Other works such as \citet{li2021safe,dean2019safely} consider a similar constrained LQR problem but require that the controller satisfies strict constraints on the state. In these works, the algorithm must choose controls such that for all $t$, $D_{\mathrm{L}}^x \le x_t \le D_{\mathrm{U}}^x$ for some $D_{\mathrm{L}}^x < 0 < D_{\mathrm{U}}^x$. However, these works also require that the noise distribution is bounded. When the noise distribution $\mathcal{D}$ is a bounded distribution (i.e. $\mathcal{D}$ satisfies $\bar{w}_{\mathrm{L}} := \inf_{w \sim \mathcal{D}} w > -\infty$ and $\bar{w}_{\mathrm{U}} := \sup_{w \sim \mathcal{D}} w < \infty$), then there exists a one-to-one mapping between Definition \ref{def:safe} and strict state constraints. Formally, when $\mathcal{D}$ is a bounded distribution, the expected-state safety constraints in Definition \ref{def:safe} are equivalent to the strict state constraints that $D_{\mathrm{L}}^{\E[x]} - \bar{w}_{\mathrm{L}} \le x_t \le D_{\mathrm{U}}^{\E[x]} - \bar{w}_{\mathrm{U}}$ for all $t \in [T]$. Therefore, the expected-state constraint formulation is \emph{strictly more general} than the safety formulation studied in these previous works.

The reason that we study expected state constraints is that they allow for more general noise distributions. For example, if $\mathcal{D}$ is normally distributed with mean $0$ and variance $1$, then $D_{\mathrm{L}}^x \le x_t \le D_{\mathrm{U}}^x$ is impossible to satisfy with probability $1-o_T(1)$ for any constant $D_{\mathrm{L}}^x, D_{\mathrm{U}}^x$. Therefore, for distributions with unbounded support, the expected state constraints are a natural way to make the problem feasible. For notational simplicity, we will often use $D = (D_{\mathrm{L}}, D_{\mathrm{U}}) = (D^{\E[x]}_{\mathrm{L}}, D^{\E[x]}_{\mathrm{U}})$ to represent the bounds for the expected-state constraints.

\subsection{Baseline Class}\label{sec:baseline_class}
In both \citet{li2021safe} and \citet{dean2019safely}, the regret baseline for the $\tilde{O}_T(T^{2/3})$ results is the total cost of the best stationary linear controller of the form $u_t = -Kx_t$ that is safe for $\theta^*$ with probability $1$. We will refer to the class of stationary linear controllers that are safe for $\theta^*$ with probability $1$ as the class of safe linear controllers. Since not all linear controllers are safe for dynamics $\theta^*$, this is restricted to $K$ that will maintain safety for $\theta^*$ for any realization of the noise, and therefore can be a very weak baseline. For example, when $D_{\mathrm{U}}$ and $D_{\mathrm{L}}$ are not symmetric, the best linear controller must still behave symmetrically. Symmetric behavior may be far from optimal for $D_{\mathrm{U}}$ and $D_{\mathrm{L}}$ that are not symmetric, yet linear controllers lack the flexibility to behave asymmetrically. As another example, when the noise distribution is unbounded, there only exists a single safe linear controller (the $K = \frac{a^*}{b^*}$ controller).

To evaluate our algorithm, we instead use the baseline of the class of \emph{truncated linear controllers}. The class of truncated linear controllers for dynamics $\theta=(a,b) \in \Theta$ is defined as $\mathcal{C}_{\mathrm{tr}}^{\theta} = \{C_K^\theta\}_{K \in [\frac{a-1}{b}, \frac{a}{b}]}$, where $C^{\theta}_K$ is defined as
\begin{equation}\label{eq:def_of_trunc_linear}
C^\theta_K(x)=
    \begin{cases}
        -Kx & \text{if } D^{\E[x]}_{\mathrm{L}} \le (a-bK)x \le D^{\E[x]}_{\mathrm{U}}\\
        \frac{D^{\E[x]}_{\mathrm{U}}-ax}{b} & \text{if } (a-bK)x > D^{\E[x]}_{\mathrm{U}} \\
        \frac{D^{\E[x]}_{\mathrm{L}} - ax}{b} & \text{if } (a-bK)x < D^{\E[x]}_{\mathrm{L}}.
    \end{cases}
\end{equation}
Note that every controller in the class of truncated linear controllers for dynamics $\theta$ is safe with probability $1$ for dynamics $\theta$. Furthermore, the class of truncated linear controllers for dynamics $\theta$ contains every linear controller that is with probability $1$ safe for dynamics $\theta$. Therefore, the class of truncated linear controllers is a strict superset of the class of safe linear controllers.
We use the class of truncated linear controllers as a baseline because these controllers are computationally tractable while also being better suited for constrained LQR than standard linear controllers. For example, truncated linear controllers can effectively handle asymmetric constraints. As noted above, every controller in the baseline class $\mathcal{C}_{\mathrm{tr}}^{\theta^*}$ is safe, and 
therefore this is a fair baseline for our safe algorithm.

To evaluate our algorithm, we compare the total cost of the algorithm to the expected total cost of the best truncated linear controller when the dynamics of the system are known. Define 
\[
     K_{\mathrm{opt}}(\theta,T) := \arg\min_{K \in [\frac{a-1}{b}, \frac{a}{b}]} J^*(\theta, C^{\theta}_K, T) .
\]
Then the expected total cost of the best truncated linear controller for dynamics $\theta^*$ is 
\begin{equation}\label{eq:fullformula}
\begin{array}{ll@{}ll}
\displaystyle\min_{C \in \mathcal{C}_{\mathrm{tr}}^{\theta^*}}  & T \cdot J^*(\theta^*, C, T)= T \cdot J^*(\theta^*, C^{\theta^*}_{K_{\mathrm{opt}}(\theta^*,T)}, T).
\end{array}
\end{equation}
Therefore, the regret of an algorithm with controller $C_{\mathrm{alg}}$ is
\begin{equation}\label{eq:regret_rv}
  \text{Regret}(C_{\text{alg}}) := T \cdot  J(\theta, C_{\text{alg}}, T, 0, W) -  T \cdot J^*(\theta^*,  C^{\theta^*}_{K_{\mathrm{opt}}(\theta^*,T)}, T).
\end{equation}
Note that as is typical in LQR learning problems \citep{abbasi2011regret,li2021safe}, the regret as defined above is a random variable that depends on $\{w_t\}_{t=0}^T$ and any randomness in $C_{\text{alg}}$. Therefore, in our results we will bound regret with high probability. Note that these bounds also imply the same bounds on the expected regret due to standard concentration inequalities and the subgaussian assumption on the noise random variables. 

\subsection{Initial Uncertainty}
Without any prior knowledge about the unknown dynamics $\theta^*$, it is impossible for any algorithm to satisfy Definition \ref{def:safe} for all possible $\theta^* \in \mathbb{R}^2$ for any non-trivial noise distribution. For example, if the noise is normally distributed, then with probability $1$ any choice of control at time $t=1$ will violate Definition \ref{def:safe} for some $\theta^* \in \mathbb{R}^2$. Therefore, we must make some assumptions about the initial uncertainty in $\theta^*$ in order for the problem to be feasible. As is standard in LQR learning problems \citep{abbasi2011regret, li2021safe}, we will assume that there exists some known initial uncertainty set $\Theta \subseteq \mathbb{R}^2$ such that $\theta^* \in \Theta$.
\begin{assum}\label{assum_init}
    There exists known $\Theta = \Theta_a \times \Theta_b =  [\underline{a}, \bar{a}] \times [\underline{b}, \bar{b}]$ such that $\theta^* \in \Theta$ and $\bar{b} \ge \underline{b} > 0$ and $\bar{a} \ge \underline{a} > 0$.
\end{assum}
We define the size of the initial uncertainty set $\Theta$ as $\mathrm{size}(\Theta) = \max(\bar{a} - \underline{a}, \bar{b} - \underline{b})$. Note that the assumption that $a^*,b^* > 0$ is made only to simplify the proofs, and the same results hold for general $\theta^*$ such that $b^* \ne 0$ ($b^*=0$ corresponds to a degenerate case). In addition to assuming knowledge of $\Theta$, we also assume access to a controller $C^{\mathrm{init}}$ that allows for some amount of initial safe exploration. As shown in \citet{schiffer2024stronger}, this assumption is asymptotically only slightly stronger than assuming that the problem is feasible. Furthermore, if the noise distribution $\mathcal{D}$ is bounded (with bound $\bar{w}$), then Assumption \ref{assum:initial} holds for a simple linear controller $C^{\mathrm{init}}$ as long as $\Theta$ satisfies $\mathrm{size}(\Theta) \le \frac{\min(D_{\mathrm{U}}^{\E[x]}, D_{\mathrm{L}}^{E[x]})}{2(1+\frac{\bar{a}}{\bar{b}} \left(\norm{D^{\E[x]}}_{\infty} + \bar{w}\right)}$.

\begin{assum}\label{assum:initial}
    There exists a known controller $C^{\mathrm{init}}$ such that \\ $\forall x \in \left[ D_{\mathrm{L}}^{\E[x]} + F_{\mathcal{D}}^{-1}(\frac{1}{T^4}), D_{\mathrm{U}}^{\E[x]} + F_{\mathcal{D}}^{-1}(1-\frac{1}{T^4})\right] $,
    \begin{equation}\label{eq:assum_initial}
        D_{\mathrm{L}}^{\E[x]} + \frac{b^*}{\log(T)} \le  a^*x + b^*C^{\mathrm{init}}(x) \le D_{\mathrm{U}}^{\E[x]} - \frac{b^*}{\log(T)}.
    \end{equation}
\end{assum}

\subsection{Problem Statement}
Putting everything together, the formal problem statement is the following:

\begin{problem}[Safe LQR Learning]\label{problem_main}
    Find an algorithm $C^{\mathrm{alg}}$ that takes as input $D, \mathcal{D}, \Theta$, and $T$ that satisfy Assumptions \ref{assum:constraints}--\ref{assum:initial}, and achieves regret under linear dynamics with respect to baseline $\mathcal{C}_{\mathrm{tr}}^{\theta^*}$ that is as low as possible, while also satisfying 
    \[
    \sup_{\theta \in \Theta} \P\left(C^{\mathrm{alg}} \text{ is safe with respect to $\theta$}\right) = 1-o_T(1/T).
    \]
\end{problem}
Informally, $\sup_{\theta \in \Theta} \P\left(C^{\mathrm{alg}} \text{ is safe with respect to $\theta$}\right)= 1-o_T(1/T)$ is equivalent to saying that for any $\theta \in \Theta$, if $\theta^* = \theta$ then using $C^{\mathrm{alg}}$ will result in a series of controls that satisfy Definition \ref{def:safe} with high probability. Note that in Problem \ref{problem_main}, we only require that $C_{\mathrm{alg}}$ is safe with high probability rather than safe with probability $1$. The reason for this is that requiring safety with probability $1$ would mean that $C^{\mathrm{alg}}$ is unable to use any conclusions about $\theta^*$ learned from the history that do not hold with probability $1$. For example in the case of unbounded noise distributions, making any statement about $\theta^*$ from historical data that holds with probability $1$ is impossible. Therefore, we allow a vanishing $o_T(1/T)$ probability of the algorithm not being safe to allow the algorithm to use historical information when choosing safe controls. Note that the choice of $o_T(1/T)$ is made for expositional purposes, and an equivalent result holds when $1-o_T(1/T)$ is replaced with $1-\delta$ for $\delta < 1$. Throughout this paper, we use $O_T(\cdot)$ and other big-O notation to represent equations that hold for sufficiently large $T$, where equations with $O_T(\cdot)$ hold for sufficiently large $T$ and contain unwritten constants that are independent of $T$ and any other variables included in the parentheses. For expositional purposes in the proofs, we will also assume that $\log(T^{1/12})$ is an integer, however simple modifications to the algorithm allow the same result to hold for all $T$. More discussion of notation and definitions can be found in Appendix \ref{app:notation}.

\section{Theoretical Results}\label{sec:theoretical_results}
We now formally state our main result on truncated linear controllers and provide some general intuition for the proof and algorithm. We present a more detailed proof sketch of Theorem \ref{trunc_thm} in Section \ref{sec:trunc_lin_cont} and the full proof in Appendix \ref{sec:trunc_lin_sqrtt}.

\begin{thm}\label{trunc_thm}
    In the setting of Problem \ref{problem_main}, there exists an algorithm $C^{\mathrm{alg}}$ (Algorithm \ref{alg:cap3}) that with probability $1-o_T(1/T)$ achieves  $\tilde{O}_T(\sqrt{T})$ regret with respect to baseline $\mathcal{C}_{\mathrm{tr}}^{\theta^*}$ while also satisfying $\sup_{\theta \in \Theta} \P\left(C^{\mathrm{alg}} \text{ is safe with respect to $\theta$}\right) = 1-o_T(1/T)$.
\end{thm}

The intuition of Algorithm \ref{alg:cap3} is outlined in Algorithm \ref{alg:cap3_intuit}. The algorithm first explores for $\tilde{\Theta}_T(\sqrt{T})$ steps using $C^{\mathrm{init}}$ from Assumption \ref{assum:initial}. Using the data from this exploration, the algorithm calculates a regularized least-squares estimate of $\theta^*$ (denoted $\hat{\theta}_{\mathrm{wu}}$) that is accurate up to $\tilde{O}_T(T^{-1/4})$. Based on this least-squares estimate, the algorithm then decides if the support of the noise distribution $\mathcal{D}$ is small or large relative to the constraint boundary $D$.
In the small noise case, the algorithm uses the best unconstrained controller for dynamics $\hat{\theta}_{\mathrm{wu}}$ with small modifications to the control as needed to guarantee constraint satisfaction with high probability. Because the noise is small in this case, the modification is only needed a small fraction of the time. Therefore, in this case the regret of the algorithm is only slightly more than the regret of the optimal unconstrained controller for $\hat{\theta}_{\mathrm{wu}}$, which can be shown to be $\tilde{O}_T(\sqrt{T})$ using standard certainty equivalence results. In the large noise case, the algorithm takes inspiration from \citet{schiffer2024stronger}, and uses a truncated certainty equivalence approach that guarantees $\tilde{O}_T(\sqrt{T})$ regret with high probability. Intuitively, in this case the noise is large enough to force the algorithm to a constant fraction of the time be non-linear by a constant amount. This non-linearity allows the algorithm to learn the unknown dynamics at a faster rate of $1/\sqrt{t}$, which in turn leads to regret of $\tilde{O}_T(\sqrt{T})$.

\begin{algorithm}[ht]
\caption{Outline of Algorithm \ref{alg:cap3} for proof of Theorem \ref{trunc_thm}}\label{alg:cap3_intuit}

Explore for $\tilde{\Theta}_T(\sqrt{T})$ steps using controller $C^{\mathrm{init}}$ from Assumption \ref{assum:initial}.

 $\hat{\theta}_{\mathrm{wu}} \gets $ regularized least-squares estimate of $\theta^*$.

Using $\hat{\theta}_{\mathrm{wu}}$, determine if support of noise distribution $\mathcal{D}$ is large or small relative to boundary $D$.

\If{support of $\mathcal{D}$ is \textbf{small} relative to $D$}{
For the rest of the steps, use the optimal unconstrained linear controller for dynamics $\hat{\theta}_{\mathrm{wu}}$ with small modifications to the control as necessary to enforce constraint satisfaction w.h.p.
}
\If{support of $\mathcal{D}$ is \textbf{large} relative to $D$}{

\nl  \label{line:intuit_start} \For{$s \in [0:\log(\sqrt{T})-1]$}{ 

 $\hat{\theta}_s \gets $ regularized least-squares estimate of $\theta^*$ using data seen so far

 $\epsilon_s \gets$ high probability bound on $\norm{\theta^* - \hat{\theta}_s}_{\infty}$

 $C_s^{\mathrm{alg}} \gets$ optimal truncated linear controller for dynamics $\hat{\theta}_s$

  \nl  \label{line:intuit_end} For next $\sqrt{T}2^s$ steps, use controller $C_s^{\mathrm{alg}}$ modified at each step to be safe for all dynamics $\theta$ satisfying $\norm{\theta - \hat{\theta}_s}_{\infty} \le \epsilon_s$
}
}
\end{algorithm}

In proving Theorem \ref{trunc_thm}, we also show that the class of truncated linear controllers satisfies two natural assumptions of continuity first proposed in \citet{schiffer2024stronger}, formalized in the following two lemmas. Informally, Lemma \ref{parameterization_assum2} says that the cost of the optimal truncated linear controller is Lipschitz continuous in the dynamics. Therefore, using the optimal controller for dynamics $\theta$ that are close to the true dynamics $\theta^*$ does not incur significantly higher cost.

\begin{lemma}\label{parameterization_assum2}
    There exists $\epsilon_{\mathrm{L}\ref{parameterization_assum2}} = \tilde{\Omega}_T(1)$ such that for any $\norm{\theta - \theta^*}_{\infty}  \le \epsilon_{\mathrm{L}\ref{parameterization_assum2}}$ and $t \le T$,
    \[
        |J^*(\theta^*,C_{K_{\mathrm{opt}}(\theta, t)}^\theta,t) - J^*(\theta^*,C_{K_{\mathrm{opt}}(\theta^*, t)}^{\theta^*},t)| \le \tilde{O}_T\left(\norm{\theta - \theta^*}_{\infty} + \frac{1}{T^2}\right).
    \]
   
\end{lemma}
The proof of Lemma \ref{parameterization_assum2} can be found in Appendix \ref{app:paramterization_assum2}. Next, informally, Lemma \ref{parameterization_assum3} says that the cost of using a truncated linear controller is Lipschitz continuous in the starting state. Therefore, if $|x-y|$ is sufficiently small, then the difference in total cost of starting at $x$ versus $y$ is linear in $|x-y|$.
\begin{lemma}\label{parameterization_assum3}
     There exist $\epsilon_{\mathrm{L}\ref{parameterization_assum3}}, \delta_{\mathrm{L}\ref{parameterization_assum3}} = \tilde{\Omega}_T(1)$ such that for any $\theta$ satisfying $\norm{\theta - \theta^*}_{\infty} \le \epsilon_{\mathrm{L}\ref{parameterization_assum3}}$ the following holds. For $t < T$, let $W' = \{w_i\}_{i=0}^{t-1}$. Then for any $K \in [\frac{a-1}{b}, \frac{a}{b}]$, there exists a set $\mathcal{Y}_{\mathrm{L}\ref{parameterization_assum3}} \in \mathbb{R}^{t}$ that depends only on $C_K^\theta$ such that the following holds. Define $E_{\mathrm{L}\ref{parameterization_assum3}}\left(C_K^\theta, W'\right)$ as the event that $W' \in \mathcal{Y}_{\mathrm{L}\ref{parameterization_assum3}}$. Then $\P(E_{\mathrm{L}\ref{parameterization_assum3}}\left(C_K^\theta, W'\right)) \ge 1-o_T(1/T^{10})$ and for any  $|x|, |y| \le 4\log^2(T)$ such that $|x-y| \le \delta_{\mathrm{L}\ref{parameterization_assum3}}$, conditional on event $E_{\mathrm{L}\ref{parameterization_assum3}}\left(C_K^\theta, W'\right)$,
    \begin{equation}\label{eq:param_assum3}
        \left|t \cdot J(\theta^*,C_{K}^\theta,t,x, W') - t \cdot J(\theta^*,C^{\theta}_{K},t, y, W')\right| \le \tilde{O}_T(|x-y| + \norm{\theta - \theta^*}_{\infty}).
    \end{equation}

\end{lemma}
The proof of Lemma \ref{parameterization_assum3} can be found in Appendix \ref{app:parameterization_assum3}. 
Lemmas \ref{parameterization_assum2} and \ref{parameterization_assum3} give a concrete instantiation of the results of \citet{schiffer2024stronger}. Because \citet{schiffer2024stronger} does not give any concrete baselines for their framework, these two lemmas are necessary to show the applicability of their framework.  However, also note that Theorem \ref{trunc_thm} is strictly stronger than the theorems of \citet{schiffer2024stronger} would be for truncated linear controllers.

As discussed above, truncated linear controllers are a natural extension of linear controllers better suited for problems with safety constraints. Because truncated linear controllers are not linear, the analysis of this class requires new theoretical tools (see Appendices \ref{app:paramterization_assum2} and \ref{app:parameterization_assum3}). These proofs and results may be independently interesting in that non-linear controllers have not been well-studied in this setting and therefore little was previously known about properties of such controller classes.

\section{Proof Sketch of Theorem \ref{trunc_thm}}\label{sec:trunc_lin_cont}
The full proof of Theorem \ref{trunc_thm} can be found in Appendix \ref{sec:trunc_lin_sqrtt}. Before presenting the algorithm for Theorem \ref{trunc_thm}, we need additional notation. Define $\mathcal{C}^{\mathrm{unc}} = \{C_K^{\mathrm{unc}}\}_{K \in \mathbb{R}}$ as the class of untruncated linear controllers, so $C_K^{\mathrm{unc}}(x) = -Kx$. For any controller $C$ and dynamics $\theta$, define $J^*(\theta, C) =  \lim_{T \longrightarrow \infty} J^*(\theta, C, T)$. Define  $K_{\mathrm{opt}}(\theta) = \arg\max_{K} J^*(\theta, C_K^\theta)$ and $F_{\mathrm{opt}}(\theta) = \arg\max_{K} J^*(\theta, C_K^{\mathrm{unc}})$. Finally, define $C_{\mathrm{switch}} = \frac{c_{\mathrm{E}\ref{eq:Fhat_approx}}D_{\mathrm{U}}}{c_{\mathrm{L}\ref{j_bounded_from_0}}^2}$ where $c_{\mathrm{E}\ref{eq:Fhat_approx}} = \tilde{O}_T(1)$ and is from Equation \eqref{eq:Fhat_approx} and $c_{\mathrm{L}\ref{j_bounded_from_0}} = \Omega(1)$ from Lemma \ref{j_bounded_from_0}. The algorithm that achieves the regret bound of Theorem \ref{trunc_thm} is Algorithm \ref{alg:cap3}.

\begin{algorithm}[ht!]
\caption{Truncated Linear Controller Safe LQR}\label{alg:cap3}
\KwIn{$D, \mathcal{D}, \Theta, C^{\mathrm{init}}, T, \lambda$}
\For{$t \gets 0$ to $\sqrt{T}-1$} {
        $\phi_t \sim \mathrm{Rademacher}(0.5)$
        Use control $u_t = C^{\text{init}}(x_t) + \frac{\phi_t}{\log(T)}$
}
\nl \label{line:thetawu} $\hat{\theta}_{\mathrm{wu}} \gets (Z_{{\sqrt{T}}}^{\top}Z_{{\sqrt{T}}}+\lambda I)^{-1}Z_{{\sqrt{T}}}^{\top}X_{{\sqrt{T}}}$ 

\For{$s \gets 0$ to $\log_2(\sqrt{T}) - 1$}{ 
   \nl \label{line:exploit_start_tl}  $T_s \gets 2^s\sqrt{T}$
     
   \nl \label{line:epsilon_tl} $\epsilon_s \gets B_{T_s}\sqrt{\frac{\max\left(V_{T_s}^{22}, V_{T_s}^{11}\right)}{V_{T_s}^{11}V_{T_s}^{22} - (V_{T_s}^{12})^2}}$  
   
    $\hat{\theta}_s^{\mathrm{pre}} \gets (Z_{{T_s}}^{\top}Z_{{T_s}}+\lambda I)^{-1}Z_{{T_s}}^{\top}X_{{T_s}}$
   
    $\hat{\theta}_s \gets \argmax_{\norm{\theta - \hat{\theta}^{\text{pre}}_s} \le \epsilon_s}  a - bK_{\mathrm{opt}}(\theta)$
   
   \nl   \label{line:split}  $C_s^{\mathrm{alg}} \gets
                \begin{cases}
                    C_{F_{\mathrm{opt}}(\hat{\theta}_{\mathrm{wu}})}^{\mathrm{unc}} & \text{if } \bar{w} +D_{\mathrm{U}} - \frac{D_{\mathrm{U}}}{\hat{a}_{\mathrm{wu}}-\hat{b}_{\mathrm{wu}}F_{\mathrm{opt}}(\hat{\theta}_{\mathrm{wu}})} \le C_{\mathrm{switch}}T^{-1/4}\\
                    C_{ K_{\mathrm{opt}}(\hat{\theta}_s)}^{\hat{\theta}_s} & \text{otherwise}
                \end{cases}$

    \For{$t \gets T_s$ to $2T_s - 1$}{
        \If{$\bar{w} +D_{\mathrm{U}} - \frac{D_{\mathrm{U}}}{\hat{a}_{\mathrm{wu}}-\hat{b}_{\mathrm{wu}}F_{\mathrm{opt}}(\hat{\theta}_{\mathrm{wu}})} \le C_{\mathrm{switch}}T^{-1/4}$}{
       
        $u_t^{\mathrm{safeU}} \gets \max \left\{ u:  \displaystyle\max_{\norm{\theta - \hat{\theta}^{\mathrm{wu}}}_{\infty} \le \epsilon_0} ax_t + bu \le D_{\mathrm{U}}\right\}$
       
       $u_t^{\mathrm{safeL}} \gets \min \left\{ u: \displaystyle\min_{\norm{\theta - \hat{\theta}^{\mathrm{wu}}}_{\infty} \le \epsilon_0} ax_t + bu \ge D_{\mathrm{L}}\right\}$
        }
        \Else{
       
      \nl \label{line:safeU_tl} $u_t^{\mathrm{safeU}} \gets \max \left\{u : \displaystyle\max_{\norm{\theta - \hat{\theta}_s}_{\infty} \le \epsilon_s} ax_t + bu \le D_{\mathrm{U}}\right\}$ 
       
         \nl \label{line:safeL_tl}      $u_t^{\mathrm{safeL}} \gets \min \left\{u :  \displaystyle\min_{\norm{\theta - \hat{\theta}_s}_{\infty} \le \epsilon_s} ax_t + bu \ge D_{\mathrm{L}}\right\}$
       
       \nl \label{line:safety_tl}  Use control $u_t = \max\left(\min\left(C_s^{\mathrm{alg}}(x_t), u^{\mathrm{safeU}}_t\right),u^{\mathrm{safeL}}_t\right)$ 
        }
    }
}
\end{algorithm}

    \paragraph{Algorithm \ref{alg:cap3} Intuition}
    The main intuition behind the proof of Theorem \ref{trunc_thm} is to design an algorithm that combines the faster learning rates under tight constraints from \citet{schiffer2024stronger} with the observation that $\tilde{O}_T(\sqrt{T})$ regret is possible in unconstrained LQR learning with unknown dynamics. Algorithm \ref{alg:cap3} is broken into two phases. The first phase is a warm-up exploration phase that allows the algorithm to learn about the unknown dynamics quickly but potentially incurs high per-step cost. The second phase of the algorithm uses a form of certainty equivalence. The key is to split the choice of $C_s^{\mathrm{alg}}$ into two cases (Line \ref{line:split}) depending on the estimated dynamics ($\hat{\theta}_{\mathrm{wu}}$) at the end of the warm-up period. The first case in Line \ref{line:split} corresponds to when the support of the noise is sufficiently small so that we can bound the regret of the algorithm using the observation that  $\tilde{O}_T(\sqrt{T})$ regret is possible in the unconstrained setting. More specifically, this case is when the boundaries are far enough away from the origin compared to the magnitude of the noise, and therefore the algorithm can use a controller very close to the optimal unconstrained controller.  The second case in Line \ref{line:split} corresponds to when the support of the noise is sufficiently large so that we can use the faster learning rate from \citet{schiffer2024stronger}. More specifically, in this case we argue that the uncertainty bound $\epsilon_s$ will decrease at a rate of $\tilde{O}_T(1/\sqrt{T_s})$ (Proposition \ref{prop_Kcase}). We give more details on the $\tilde{O}_T(\sqrt{T})$ regret of these two cases separately below. The warm-up phase of Algorithm \ref{alg:cap3} satisfies the safety constraints with probability $1-o_T(1/T)$ by Assumption \ref{assum:initial}. The second phase satisfies the safety constraint with probability $1-o_T(1/T)$ because of the final choice of $u_t$ in Line \ref{line:safety_tl}. With high probability, $\norm{\theta^* - \hat{\theta}^{\mathrm{wu}}}_{\infty} \le \epsilon_0$ and $\norm{\theta^* - \hat{\theta}_s}_{\infty} \le \epsilon_s$, and therefore with high probability $u_t^{\mathrm{safeU}}$ and $u_t^{\mathrm{safeL}}$ provide upper and lower bounds on a set of safe controls. Therefore, the choice of $u_t$ is safe with probability $1-o_T(1/T)$ for all steps in the exploration phase.

    \paragraph{Sufficiently small noise case}
     In this case, we let  $C_s^{\mathrm{alg}} = C_{F_{\mathrm{opt}}(\hat{\theta}_{\mathrm{wu}})}^{\mathrm{unc}}$, i.e. the optimal unconstrained controller based on the data in the warm-up period. First, we show that the controller $ C_{F_{\mathrm{opt}}(\hat{\theta}_{\mathrm{wu}})}^{\mathrm{unc}}$ has $\tilde{O}_T(\sqrt{T})$ more expected total cost for $T_s$ steps than the baseline controller $C_{K_{\mathrm{opt}}(\theta^*, T_s)}^{\theta^*}$ (Lemma \ref{close_J2}). Intuitively, this follows from the fact that $ C_{F_{\mathrm{opt}}(\hat{\theta}_{\mathrm{wu}})}^{\mathrm{unc}}$ has similar expected cost to the best infinite-time unconstrained controller for $\theta^*$, and the best infinite-time controller and the best finite-time controller for $T_s$ steps have similar expected cost. Because $C_{F_{\mathrm{opt}}(\hat{\theta}_{\mathrm{wu}})}^{\mathrm{unc}}$ is an unconstrained linear controller, we can also show that the realized total cost of using this controller concentrates to within $\tilde{O}_T(\sqrt{T})$ of the expected total cost with high probability (Lemma \ref{close_safe_J}). 
    
    The last (and most subtle) part of this case is to show that enforcing safety in Line \ref{line:safety_tl} only contributes $\tilde{O}_T(\sqrt{T})$ regret (Lemma \ref{safety_is_cheap}). This is where we use the fact that $\bar{w} +D_{\mathrm{U}} - \frac{D_{\mathrm{U}}}{\hat{a}_{\mathrm{wu}}-\hat{b}_{\mathrm{wu}}F_{\mathrm{opt}}(\hat{\theta}_{\mathrm{wu}})} \le C_{\mathrm{switch}}T^{-1/4}$. When this equation holds, the probability that the algorithm uses control $u_t = u_t^{\mathrm{safeU}}$ or $u_t = u_t^{\mathrm{safeL}}$ is at most $\tilde{O}_T(T^{-1/4})$ for any $t$. Furthermore, each time these controls are used, the extra cost compared to using control $u_t = C^{\mathrm{alg}}_s(x_t)$ is $\tilde{O}_T(T^{-1/4})$. Combining these two facts, the total extra regret from using controls $u_t^{\mathrm{safeU}}$ or $u_t^{\mathrm{safeL}}$ is $\tilde{O}_T(\sqrt{T})$ with probability $1-o_T(1/T)$. The warm-up period has regret of $\tilde{O}_T(\sqrt{T})$ with probability $1-o_T(1/T)$ because the algorithm is safe with high probability and the length of warm-up is $\sqrt{T}$ steps. Putting this all together, we have that with probability $1-o_T(1/T)$, the total regret of the algorithm in this case is $ \tilde{O}_T(\sqrt{T})$.    
 
 \paragraph{Sufficiently large noise case}

 In this case, we have that $C_s^{\mathrm{alg}} = C_{ K_{\mathrm{opt}}(\hat{\theta}_s)}^{\hat{\theta}_s}$. To prove that the regret is $\tilde{O}_T(\sqrt{T})$ in this case, we will show that with probability $1-o_T(1/T)$, the uncertainty bound satisfies $\epsilon_s = \tilde{O}_T(1/\sqrt{T_s})$ for every $s$. To show this, we use Lemma \ref{boundary_uncertainty_cont_b}, an uncertainty bound that is based on Lemma \refgen{boundary_uncertainty}. Lemma \ref{boundary_uncertainty_cont_b} says that $\epsilon_s$ is upper bounded by $\tilde{O}_T(1/\sqrt{|S_{T_s}|})$ with probability $1-o_T(1/T)$, where $|S_{T_s}|$ is the number of times $t < T_s$ that the algorithm uses control $u_t^{\mathrm{safeU}}$ and such that the probability of using the control $u_t^{\mathrm{safeU}}$ conditional on the history up until that point is lower-bounded by a constant. To use this lemma, we show that with probability $1-o_T(1/T)$, we have $|S_{T_s}| \ge \Omega_T(T_s)$ for all $s$ (Lemma \ref{sufficiently_many_boundaries}). 
 
 In this case, the key observation is that when using the controller $ C_{ K_{\mathrm{opt}}(\hat{\theta}_s)}^{\hat{\theta}_s}$, there exist constants $\epsilon, d_\epsilon > 0$ such that at every time step $t$ when the control is not $u_t^{\mathrm{safeU}}$, there is  an $\epsilon$ probability that the state increases by $d_{\epsilon}$ (Lemma \ref{eventually_go_to_boundary}). Informally, this says that at every step, either $u_t = u_t^{\mathrm{safeU}}$ or the state will increase by a constant amount with a constant probability. Therefore, because $D$ is a constant relative to $T$, we have that with high probability, every $\Omega(1)$ steps the state will exceed $P(\theta^*,  K_{\mathrm{opt}}(\hat{\theta}_s), D_{\mathrm{U}})$ or there will be a $t$ such that $u_t = u_t^{\mathrm{safeU}}$. The control at any time $t$ where $x_t \ge P(\theta^*,  K_{\mathrm{opt}}(\hat{\theta}_s), D_{\mathrm{U}})$ is $u_t = u_t^{\mathrm{safeU}}$. Therefore, with high probability every $\Omega(1)$ steps there will exist a $t$ such that the algorithm uses control $u_t = u_t^{\mathrm{safeU}}$, and we further show that this happens with constant probability. This implies that $|S_{T_s}| \ge \Omega(T_s)$ for every $s$ with high probability.  Combining with Lemma \ref{boundary_uncertainty_cont_b} gives that with probability $1-o_T(1/T)$, $\epsilon_s \le \tilde{O}_T(1/\sqrt{T_s})$.

We finish by bounding each source of regret. The first source of regret is the regret from using certainty equivalence, i.e. using $\hat{\theta}_s$ instead of using $\theta^*$ in finding $C^{\mathrm{alg}}_s$. Using Lemma \ref{parameterization_assum2}, the expected cost of using $   C_{ K_{\mathrm{opt}}(\hat{\theta}_s)}^{\hat{\theta}_s} $ instead of $   C_{ K_{\mathrm{opt}}(\theta^*)}^{\theta^*} $ is $\tilde{O}_T(T_s\norm{\hat{\theta}_s - \theta^*}_{\infty} + 1/T)$. Because $\norm{\hat{\theta}_s - \theta^*}_{\infty} \le \epsilon_s \le \tilde{O}_T(1/\sqrt{T_s})$ with high probability, this source of regret is $\tilde{O}_T(\sqrt{T})$ with high probability. The second source of regret is the regret from randomness in the regret random variable, which can be bounded by $\tilde{O}_T(\sqrt{T})$ by a variant of McDiarmids Inequality. The third source of regret is the regret of enforcing safety with $u_t^{\mathrm{safeU}}$ and $u_t^{\mathrm{safeL}}$ in the choice of $u_t$. By construction $u_t$ differs from $C_s^{\mathrm{alg}}$ by $\tilde{O}_T(\epsilon_s) = \tilde{O}_T(1/\sqrt{T_s})$ at every time step. Therefore by Lemma \ref{parameterization_assum3}, the regret of enforcing safety by using $u_t$ is $\tilde{O}_T(\sqrt{T})$ with high probability. The warm-up period has regret $\tilde{O}_T(\sqrt{T})$ as in the small noise case. Finally, there is one additional component of regret in this proof, as we are using the best infinite time controller rather than the best $T_s$-step controller in round $s$. However, we can show that this only adds at most $\tilde{O}_T(\sqrt{T})$ extra cost, and therefore the total regret is with probability $1-o_T(1/T)$ still $\tilde{O}_T(\sqrt{T})$ (Lemma \ref{fin_to_inf}). See Appendix \ref{app:proof_of_lemmas} for a proof sketch of Lemmas \ref{parameterization_assum2} and \ref{parameterization_assum3}.

\section{Discussion}\label{sec:discussion}
In this section we discuss a few limitations of our results and some open questions. In this work, we focus on state constraints rather than constraints on the actions themselves. We expect that very minor modifications to Algorithm \ref{alg:cap3} will naturally extend these results to also apply to the setting where the controls $u_t$ must satisfy constraints. More specifically, we would need to choose $C_s^{\mathrm{alg}}$ in Algorithm \ref{alg:cap3} to only choose controls that satisfy control constraints with an extra buffer of $\tilde{\Theta}_T(\epsilon_s)$. See \citet{schiffer2024stronger} for more details on how results regarding state constraints can generalize to problems with control constraints as well. We leave formal derivations of this to future work.

Another natural extension of our results is to higher dimensional linear control problems. Our results focus on the one-dimensional case, but we expect that similar algorithmic ideas will extend to higher dimensional control problems. While we leave formal study of the higher dimensional case to future work, we highlight here a few interesting open questions regarding safety constrained control in higher dimensions. In higher dimensions, the system may not be one-step controllable, and therefore returning to the safe region in one step may be impossible. Therefore, for unbounded noise distributions there is not a clear definition of safety for these systems in higher dimensions. One simple case in which we do expect the results of this paper to easily generalize is when the system is one-step controllable and the constraints are symmetric around the origin. However, the question of whether $\tilde{O}_T(\sqrt{T})$ regret is possible for all noise distributions in higher dimensions is an open question for future work. In this paper, we also introduced the class of truncated linear controllers and proved some desirable properties of this class of controllers. We expect these properties to still hold in higher dimensions, but we leave formal study of this to future work.

\section*{Acknowledgements}

The authors would like to thank Na Li and Shahriar Talebi for helpful discussions. B.S. and L.J. received funding from NSF grant CBET-2112085 and B.S. received funding from the National Science Foundation Graduate Research Fellowship grant DGE 2140743.

\bibliographystyle{plainnat}

\bibliography{reference_arxiv}

\begin{thebibliography}{57}
\providecommand{\natexlab}[1]{#1}
\providecommand{\url}[1]{\texttt{#1}}
\expandafter\ifx\csname urlstyle\endcsname\relax
  \providecommand{\doi}[1]{doi: #1}\else
  \providecommand{\doi}{doi: \begingroup \urlstyle{rm}\Url}\fi

\bibitem[Abbasi-Yadkori and Szepesv{\'a}ri(2011)]{abbasi2011regret}
Yasin Abbasi-Yadkori and Csaba Szepesv{\'a}ri.
\newblock Regret bounds for the adaptive control of linear quadratic systems.
\newblock In \emph{Proceedings of the 24th Annual Conference on Learning Theory}, pages 1--26. JMLR Workshop and Conference Proceedings, 2011.

\bibitem[Abeille and Lazaric(2017)]{abeille2017thompson}
Marc Abeille and Alessandro Lazaric.
\newblock Thompson sampling for linear-quadratic control problems.
\newblock In \emph{Artificial intelligence and statistics}, pages 1246--1254. PMLR, 2017.

\bibitem[Agarwal et~al.(2019)Agarwal, Bullins, Hazan, Kakade, and Singh]{agarwal2019online}
Naman Agarwal, Brian Bullins, Elad Hazan, Sham Kakade, and Karan Singh.
\newblock Online control with adversarial disturbances.
\newblock In \emph{International Conference on Machine Learning}, pages 111--119. PMLR, 2019.

\bibitem[Anderson and Moore(2007)]{anderson2007optimal}
Brian~DO Anderson and John~B Moore.
\newblock \emph{Optimal control: linear quadratic methods}.
\newblock Courier Corporation, 2007.

\bibitem[Athrey et~al.(2024)Athrey, Mazhar, Guo, De~Schutter, and Shi]{athrey2024regret}
Archith Athrey, Othmane Mazhar, Meichen Guo, Bart De~Schutter, and Shengling Shi.
\newblock Regret analysis of learning-based linear quadratic gaussian control with additive exploration.
\newblock In \emph{2024 European Control Conference (ECC)}, pages 1795--1801. IEEE, 2024.

\bibitem[Bemporad and Morari(2007)]{bemporad2007robust}
Alberto Bemporad and Manfred Morari.
\newblock Robust model predictive control: A survey.
\newblock In \emph{Robustness in identification and control}, pages 207--226. Springer, 2007.

\bibitem[Bemporad et~al.(2002)Bemporad, Morari, Dua, and Pistikopoulos]{bemporad2002explicit}
Alberto Bemporad, Manfred Morari, Vivek Dua, and Efstratios~N Pistikopoulos.
\newblock The explicit linear quadratic regulator for constrained systems.
\newblock \emph{Automatica}, 38\penalty0 (1):\penalty0 3--20, 2002.

\bibitem[Cheng et~al.(2019)Cheng, Orosz, Murray, and Burdick]{cheng2019end}
Richard Cheng, G{\'a}bor Orosz, Richard~M Murray, and Joel~W Burdick.
\newblock End-to-end safe reinforcement learning through barrier functions for safety-critical continuous control tasks.
\newblock In \emph{Proceedings of the AAAI conference on artificial intelligence}, volume~33, pages 3387--3395, 2019.

\bibitem[Cohen et~al.(2019)Cohen, Koren, and Mansour]{cohen2019learning}
Alon Cohen, Tomer Koren, and Yishay Mansour.
\newblock Learning linear-quadratic regulators efficiently with only sqrtt regret.
\newblock pages 1300--1309, 2019.

\bibitem[Dean et~al.(2018)Dean, Mania, Matni, Recht, and Tu]{dean2018regret}
Sarah Dean, Horia Mania, Nikolai Matni, Benjamin Recht, and Stephen Tu.
\newblock Regret bounds for robust adaptive control of the linear quadratic regulator.
\newblock \emph{Advances in Neural Information Processing Systems}, 31, 2018.

\bibitem[Dean et~al.(2019)Dean, Tu, Matni, and Recht]{dean2019safely}
Sarah Dean, Stephen Tu, Nikolai Matni, and Benjamin Recht.
\newblock Safely learning to control the constrained linear quadratic regulator.
\newblock In \emph{2019 American Control Conference (ACC)}, pages 5582--5588. IEEE, 2019.

\bibitem[Faradonbeh et~al.(2017)Faradonbeh, Tewari, and Michailidis]{faradonbeh2017finite}
Mohamad Kazem~Shirani Faradonbeh, Ambuj Tewari, and George Michailidis.
\newblock Finite time analysis of optimal adaptive policies for linear-quadratic systems.
\newblock \emph{arXiv preprint arXiv:1711.07230}, 2017.

\bibitem[Faradonbeh et~al.(2018{\natexlab{a}})Faradonbeh, Tewari, and Michailidis]{faradonbeh2018input}
Mohamad Kazem~Shirani Faradonbeh, Ambuj Tewari, and George Michailidis.
\newblock Input perturbations for adaptive regulation and learning.
\newblock \emph{arXiv preprint arXiv:1811.04258}, 2018{\natexlab{a}}.

\bibitem[Faradonbeh et~al.(2018{\natexlab{b}})Faradonbeh, Tewari, and Michailidis]{faradonbeh2018optimality}
Mohamad Kazem~Shirani Faradonbeh, Ambuj Tewari, and George Michailidis.
\newblock On optimality of adaptive linear-quadratic regulators.
\newblock \emph{arXiv preprint arXiv:1806.10749}, 2018{\natexlab{b}}.

\bibitem[Fefferman et~al.(2021)Fefferman, Pegueroles, Rowley, and Weber]{fefferman2021optimal}
Charles Fefferman, Bernat~Guill{\'e}n Pegueroles, Clarence~W Rowley, and Melanie Weber.
\newblock Optimal control with learning on the fly: a toy problem.
\newblock \emph{Revista matem{\'a}tica iberoamericana}, 38\penalty0 (1):\penalty0 175--187, 2021.

\bibitem[Fisac et~al.(2018)Fisac, Akametalu, Zeilinger, Kaynama, Gillula, and Tomlin]{fisac2018general}
Jaime~F Fisac, Anayo~K Akametalu, Melanie~N Zeilinger, Shahab Kaynama, Jeremy Gillula, and Claire~J Tomlin.
\newblock A general safety framework for learning-based control in uncertain robotic systems.
\newblock \emph{IEEE Transactions on Automatic Control}, 64\penalty0 (7):\penalty0 2737--2752, 2018.

\bibitem[Fulton and Platzer(2018)]{fulton2018safe}
Nathan Fulton and Andr{\'e} Platzer.
\newblock Safe reinforcement learning via formal methods: Toward safe control through proof and learning.
\newblock In \emph{Proceedings of the AAAI Conference on Artificial Intelligence}, volume~32, 2018.

\bibitem[Ganai et~al.(2024)Ganai, Gong, Yu, Herbert, and Gao]{ganai2024iterative}
Milan Ganai, Zheng Gong, Chenning Yu, Sylvia Herbert, and Sicun Gao.
\newblock Iterative reachability estimation for safe reinforcement learning.
\newblock \emph{Advances in Neural Information Processing Systems}, 36, 2024.

\bibitem[Garg et~al.(2024)Garg, Zhang, So, Dawson, and Fan]{garg2024learning}
Kunal Garg, Songyuan Zhang, Oswin So, Charles Dawson, and Chuchu Fan.
\newblock Learning safe control for multi-robot systems: Methods, verification, and open challenges.
\newblock \emph{Annual Reviews in Control}, 57:\penalty0 100948, 2024.

\bibitem[Gu et~al.(2022)Gu, Yang, Du, Chen, Walter, Wang, and Knoll]{gu2022review}
Shangding Gu, Long Yang, Yali Du, Guang Chen, Florian Walter, Jun Wang, and Alois Knoll.
\newblock A review of safe reinforcement learning: Methods, theory and applications.
\newblock \emph{arXiv preprint arXiv:2205.10330}, 2022.

\bibitem[Hazan and Singh(2022)]{hazan2022introduction}
Elad Hazan and Karan Singh.
\newblock Introduction to online nonstochastic control.
\newblock \emph{arXiv preprint arXiv:2211.09619}, 2022.

\bibitem[Khosravi and Smith(2020)]{khosravi2020nonlinear}
Mohammad Khosravi and Roy~S Smith.
\newblock Nonlinear system identification with prior knowledge on the region of attraction.
\newblock \emph{IEEE Control Systems Letters}, 5\penalty0 (3):\penalty0 1091--1096, 2020.

\bibitem[K{\"o}hler et~al.(2019)K{\"o}hler, Andina, Soloperto, M{\"u}ller, and Allg{\"o}wer]{kohler2019linear}
Johannes K{\"o}hler, Elisa Andina, Raffaele Soloperto, Matthias~A M{\"u}ller, and Frank Allg{\"o}wer.
\newblock Linear robust adaptive model predictive control: Computational complexity and conservatism.
\newblock In \emph{2019 IEEE 58th Conference on Decision and Control (CDC)}, pages 1383--1388. IEEE, 2019.

\bibitem[Lee et~al.(2024)Lee, Rantzer, and Matni]{lee2024nonasymptotic}
Bruce Lee, Anders Rantzer, and Nikolai Matni.
\newblock Nonasymptotic regret analysis of adaptive linear quadratic control with model misspecification.
\newblock In \emph{6th Annual Learning for Dynamics \& Control Conference}, pages 980--992. PMLR, 2024.

\bibitem[Levine et~al.(2016)Levine, Finn, Darrell, and Abbeel]{levine2016end}
Sergey Levine, Chelsea Finn, Trevor Darrell, and Pieter Abbeel.
\newblock End-to-end training of deep visuomotor policies.
\newblock \emph{The Journal of Machine Learning Research}, 17\penalty0 (1):\penalty0 1334--1373, 2016.

\bibitem[Li et~al.(2021{\natexlab{a}})Li, Das, and Li]{li2021online}
Yingying Li, Subhro Das, and Na~Li.
\newblock Online optimal control with affine constraints.
\newblock In \emph{Proceedings of the AAAI Conference on Artificial Intelligence}, volume~35, pages 8527--8537, 2021{\natexlab{a}}.

\bibitem[Li et~al.(2021{\natexlab{b}})Li, Das, Shamma, and Li]{li2021safe}
Yingying Li, Subhro Das, Jeff Shamma, and Na~Li.
\newblock Safe adaptive learning-based control for constrained linear quadratic regulators with regret guarantees.
\newblock \emph{arXiv preprint arXiv:2111.00411}, 2021{\natexlab{b}}.

\bibitem[Li et~al.(2023)Li, Zhang, Das, Shamma, and Li]{li2023non}
Yingying Li, Tianpeng Zhang, Subhro Das, Jeff Shamma, and Na~Li.
\newblock Non-asymptotic system identification for linear systems with nonlinear policies.
\newblock \emph{arXiv preprint arXiv:2306.10369}, 2023.

\bibitem[Lillicrap et~al.(2015)Lillicrap, Hunt, Pritzel, Heess, Erez, Tassa, Silver, and Wierstra]{lillicrap2015continuous}
Timothy~P Lillicrap, Jonathan~J Hunt, Alexander Pritzel, Nicolas Heess, Tom Erez, Yuval Tassa, David Silver, and Daan Wierstra.
\newblock Continuous control with deep reinforcement learning.
\newblock \emph{arXiv preprint arXiv:1509.02971}, 2015.

\bibitem[Lorenzen et~al.(2019)Lorenzen, Cannon, and Allg{\"o}wer]{lorenzen2019robust}
Matthias Lorenzen, Mark Cannon, and Frank Allg{\"o}wer.
\newblock Robust mpc with recursive model update.
\newblock \emph{Automatica}, 103:\penalty0 461--471, 2019.

\bibitem[Lu et~al.(2021)Lu, Cannon, and Koksal-Rivet]{lu2021robust}
Xiaonan Lu, Mark Cannon, and Denis Koksal-Rivet.
\newblock Robust adaptive model predictive control: Performance and parameter estimation.
\newblock \emph{International Journal of Robust and Nonlinear Control}, 31\penalty0 (18):\penalty0 8703--8724, 2021.

\bibitem[Mania et~al.(2019)Mania, Tu, and Recht]{mania2019certainty}
Horia Mania, Stephen Tu, and Benjamin Recht.
\newblock Certainty equivalence is efficient for linear quadratic control.
\newblock \emph{Advances in Neural Information Processing Systems}, 32, 2019.

\bibitem[Mania et~al.(2020)Mania, Jordan, and Recht]{mania2020active}
Horia Mania, Michael~I Jordan, and Benjamin Recht.
\newblock Active learning for nonlinear system identification with guarantees.
\newblock \emph{arXiv preprint arXiv:2006.10277}, 2020.

\bibitem[Marvi and Kiumarsi(2021)]{marvi2021safe}
Zahra Marvi and Bahare Kiumarsi.
\newblock Safe reinforcement learning: A control barrier function optimization approach.
\newblock \emph{International Journal of Robust and Nonlinear Control}, 31\penalty0 (6):\penalty0 1923--1940, 2021.

\bibitem[McDiarmid et~al.(1989)]{mcdiarmid1989method}
Colin McDiarmid et~al.
\newblock On the method of bounded differences.
\newblock \emph{Surveys in combinatorics}, 141\penalty0 (1):\penalty0 148--188, 1989.

\bibitem[Mesbah(2016)]{mesbah2016stochastic}
Ali Mesbah.
\newblock Stochastic model predictive control: An overview and perspectives for future research.
\newblock \emph{IEEE Control Systems Magazine}, 36\penalty0 (6):\penalty0 30--44, 2016.

\bibitem[Moldovan and Abbeel(2012)]{moldovan2012safe}
Teodor~Mihai Moldovan and Pieter Abbeel.
\newblock Safe exploration in markov decision processes.
\newblock \emph{arXiv preprint arXiv:1205.4810}, 2012.

\bibitem[Muthirayan et~al.(2022)Muthirayan, Yuan, Kalathil, and Khargonekar]{muthirayan2022online}
Deepan Muthirayan, Jianjun Yuan, Dileep Kalathil, and Pramod~P Khargonekar.
\newblock Online learning for predictive control with provable regret guarantees.
\newblock In \emph{2022 IEEE 61st Conference on Decision and Control (CDC)}, pages 6666--6671. IEEE, 2022.

\bibitem[Oldewurtel et~al.(2008)Oldewurtel, Jones, and Morari]{oldewurtel2008tractable}
Frauke Oldewurtel, Colin~N Jones, and Manfred Morari.
\newblock A tractable approximation of chance constrained stochastic mpc based on affine disturbance feedback.
\newblock In \emph{2008 47th IEEE conference on decision and control}, pages 4731--4736. IEEE, 2008.

\bibitem[Oymak and Ozay(2019)]{oymak2019non}
Samet Oymak and Necmiye Ozay.
\newblock Non-asymptotic identification of lti systems from a single trajectory.
\newblock In \emph{2019 American control conference (ACC)}, pages 5655--5661. IEEE, 2019.

\bibitem[Rawlings and Mayne(2009)]{rawlings2012postface}
J.B. Rawlings and D.Q. Mayne.
\newblock \emph{Model Predictive Control: Theory and Design}.
\newblock Nob Hill Pub., 2009.
\newblock ISBN 9780975937709.

\bibitem[Rubio et~al.(2016)Rubio, Seuret, Ariba, and Mannisi]{rubio2016optimal}
Alicia~Arce Rubio, Alexandre Seuret, Yassine Ariba, and Alessio Mannisi.
\newblock Optimal control strategies for load carrying drones.
\newblock \emph{Delays and Networked Control Systems}, pages 183--197, 2016.

\bibitem[Sattar and Oymak(2022)]{sattar2022non}
Yahya Sattar and Samet Oymak.
\newblock Non-asymptotic and accurate learning of nonlinear dynamical systems.
\newblock \emph{The Journal of Machine Learning Research}, 23\penalty0 (1):\penalty0 6248--6296, 2022.

\bibitem[Schiffer and Janson(2024)]{schiffer2024stronger}
Benjamin Schiffer and Lucas Janson.
\newblock Foundations of safe online reinforcement learning in the linear quadratic regulator: Generalized baselines.
\newblock \emph{arXiv preprint arXiv:2410.21081v2}, 2024.

\bibitem[Simchowitz and Foster(2020)]{simchowitz2020naive}
Max Simchowitz and Dylan Foster.
\newblock Naive exploration is optimal for online lqr.
\newblock In \emph{International Conference on Machine Learning}, pages 8937--8948. PMLR, 2020.

\bibitem[Simchowitz et~al.(2018)Simchowitz, Mania, Tu, Jordan, and Recht]{simchowitz2018learning}
Max Simchowitz, Horia Mania, Stephen Tu, Michael~I Jordan, and Benjamin Recht.
\newblock Learning without mixing: Towards a sharp analysis of linear system identification.
\newblock In \emph{Conference On Learning Theory}, pages 439--473. PMLR, 2018.

\bibitem[Sun et~al.(2020)Sun, Oymak, and Fazel]{sun2020finite}
Yue Sun, Samet Oymak, and Maryam Fazel.
\newblock Finite sample system identification: Optimal rates and the role of regularization.
\newblock In \emph{Learning for dynamics and control}, pages 16--25. PMLR, 2020.

\bibitem[Tewari and Murphy(2017)]{tewari2017ads}
Ambuj Tewari and Susan~A Murphy.
\newblock From ads to interventions: Contextual bandits in mobile health.
\newblock \emph{Mobile health: sensors, analytic methods, and applications}, pages 495--517, 2017.

\bibitem[Wachi et~al.(2018)Wachi, Sui, Yue, and Ono]{wachi2018safe}
Akifumi Wachi, Yanan Sui, Yisong Yue, and Masahiro Ono.
\newblock Safe exploration and optimization of constrained mdps using gaussian processes.
\newblock In \emph{Proceedings of the AAAI Conference on Artificial Intelligence}, volume~32, 2018.

\bibitem[Wachi et~al.(2024)Wachi, Shen, and Sui]{wachi2024survey}
Akifumi Wachi, Xun Shen, and Yanan Sui.
\newblock A survey of constraint formulations in safe reinforcement learning.
\newblock \emph{arXiv preprint arXiv:2402.02025}, 2024.

\bibitem[Wang and Janson(2021)]{wang2021exact}
Feicheng Wang and Lucas Janson.
\newblock Exact asymptotics for linear quadratic adaptive control.
\newblock \emph{The Journal of Machine Learning Research}, 22\penalty0 (1):\penalty0 12136--12247, 2021.

\bibitem[Wang and Janson(2022)]{wang2022rate}
Feicheng Wang and Lucas Janson.
\newblock Rate-matching the regret lower-bound in the linear quadratic regulator with unknown dynamics.
\newblock \emph{arXiv preprint arXiv:2202.05799}, 2022.

\bibitem[Yao et~al.(2024)Yao, Liu, Cen, Zhu, Yu, Zhang, and Zhao]{yao2024constraint}
Yihang Yao, Zuxin Liu, Zhepeng Cen, Jiacheng Zhu, Wenhao Yu, Tingnan Zhang, and Ding Zhao.
\newblock Constraint-conditioned policy optimization for versatile safe reinforcement learning.
\newblock \emph{Advances in Neural Information Processing Systems}, 36, 2024.

\bibitem[Ye et~al.(2024)Ye, Chi, Liu, and Gupta]{ye2024online}
Lintao Ye, Ming Chi, Zhi-Wei Liu, and Vijay Gupta.
\newblock Online actuator selection and controller design for linear quadratic regulation with unknown system model.
\newblock \emph{IEEE Transactions on Automatic Control}, 2024.

\bibitem[Zhao and Li(2022)]{zhao2022adaptive}
Zichen Zhao and Qianxiao Li.
\newblock Adaptive sampling methods for learning dynamical systems.
\newblock In \emph{Mathematical and Scientific Machine Learning}, pages 335--350. PMLR, 2022.

\bibitem[Zheng and Li(2020)]{zheng2020non}
Yang Zheng and Na~Li.
\newblock Non-asymptotic identification of linear dynamical systems using multiple trajectories.
\newblock \emph{IEEE Control Systems Letters}, 5\penalty0 (5):\penalty0 1693--1698, 2020.

\bibitem[Ziemann and Sandberg(2024)]{ziemann2024regret}
Ingvar Ziemann and Henrik Sandberg.
\newblock Regret lower bounds for learning linear quadratic gaussian systems.
\newblock \emph{IEEE Transactions on Automatic Control}, 2024.

\end{thebibliography}

\newpage

\appendix

\section{Notation}\label{app:notation}
We use the same general notation as in \cite{schiffer2024stronger}.

\subsection{Equation Notation}
Throughout this paper, we use notation such as $o_T(\cdot)$, $O_T(\cdot)$, $\omega_T(\cdot)$, $\Omega_T(\cdot)$, where the subscript $T$ highlights that these equations hold for sufficiently large $T$. The following ways we use $O$-notation are relatively standard, and we include them here for completeness. We also use $\Omega$-notation that is defined equivalently in the other direction. When using this notation, the functions $f(T)$ and $g(t)$ will always be non-negative.

\begin{itemize}
    \item $f(T) = O_T(g(T))$ if there exists $T_0$ and $M \in \mathbb{R}$ such that for $T \ge T_0$, $f(T) \le M\cdot g(T)$.
    \item $f(T) = \Omega_T(g(T))$ if there exists $T_0$ and $M \in \mathbb{R}$ such that for $T \ge T_0$, $f(T) \ge M\cdot g(T)$.
    \item $f(T) = o_T(g(T))$ if for every constant $\epsilon > 0$ there exists $T_0$ such that  for all $T \ge T_0$,  $f(T) \le \epsilon \cdot g(T)$.
    \item $f(T) = \omega_T(g(T))$ if for every constant $\epsilon > 0$ there exists $T_0$ such that  for all $T \ge T_0$,  $f(T) \ge \epsilon \cdot g(T)$.
    \item $f(T) = \tilde{O}_T(g(T))$ if there exists $T_0$ and $k, M \in \mathbb{R}$ such that for $T \ge T_0$, $f(T) \le M\cdot g(T) \cdot \log^k(T)$.
\end{itemize}

Whenever equations or inequalities involve random variables, the results hold with almost surely unless specified otherwise.

\subsection{Problem Specifications}
Below is a (non-exhaustive) list of notation used throughout the appendix.
\begin{itemize}
    \item $q,r$ : coefficients for the cost at time $t$ of $qx_t^2 + ru_t^2$.
    \item $W = \{w_t\}_{t=0}^{T-1}$ : The noise random variables for the $T$-length trajectory.
    \item $\mathcal{D}$ : Distribution of $w_t$
    \begin{itemize}
        \item $B_P$ : Upper bound on the density of $\mathcal{D}$
        \item $F_{\mathcal{D}}$ : Cumulative Density Function (CDF) of $\mathcal{D}$
        \item $\bar{w}$: the bound of $\mathcal{D}$ when the distribution is bounded.
    \end{itemize}
    \item $\Theta = [\underline{a}, \bar{a}] \times [\underline{b}, \bar{b}]$ : The given initial set of dynamics such that $\theta^* \in \Theta$ and $\mathrm{size}(\Theta) = \min(\bar{a} - \underline{a}, \bar{b} - \underline{b})$ 
    \item $\theta^* = (a^*,b^*)$ : The true (unknown) dynamics.
    \item $C^{\mathrm{init}}$ : The initial safe controller satisfying Assumption \ref{assum_init}.
    \item $D = (D_{\mathrm{L}}, D_{\mathrm{U}})$ : the expected-state boundary for the safety constraint.
    \item A set of controls $\{u_t\}$ are safe for dynamics $\{\theta_t\}$ if for all $t$, $D_{\mathrm{L}} \le a_tx_t + b_tu_t \le D_{\mathrm{U}}$.
    \item $H_t = (x_0,u_0,x_1,u_1,...,u_{t-1}, x_t)$ and $\mathcal{F}_t = \sigma(H_t)$.
    \item $J(\theta, C,T,x,W)$ : The random variable cost of using controller $C$ starting at state $x_0 = x$ for $T$ time steps under dynamics $\theta$ with noise random variables $W$.
    \item $J^*(\theta, C,T) = J^*(\theta, C,T,0) = \E[J(\theta,C,T,x,W) \mid \theta, C, T, x]$ and $J^*(\theta, C,T) = J^*(\theta, C,T,0)$.
    \item $J^*(\theta, C) = J^*(\theta, C, 0) = \lim_{T \rightarrow \infty} J^*(\theta, C,T,0)$.
    \item $\mathcal{C}^{\theta} = \{C_K^{\theta}\}_{K \in [K^\theta_{\mathrm{L}}, K^\theta_{\mathrm{U}}]}$ : a class of controllers that are safe for dynamics $\theta$ that are parameterized by $K \in [K^\theta_{\mathrm{L}}, K^\theta_{\mathrm{U}}]$
    \item $K_{\mathrm{opt}}(\theta, T)$ : The $K$ that maximizes $J^*(\theta, C_K^\theta, T, 0)$ for $K \in [K^\theta_{\mathrm{L}}, K^\theta_{\mathrm{U}}]$.
    \item $K_{\mathrm{opt}}(\theta)$ : The $K$ that maximizes $J^*(\theta, C_K^\theta)$ for $K \in [K^\theta_{\mathrm{L}}, K^\theta_{\mathrm{U}}]$.
    \item $C_K^{\mathrm{unc}}$ : The unconstrained linear controller with parameter $K$, i.e. such that $C_K^{\mathrm{unc}}(x) = -Kx$.
    \item $F_{\mathrm{opt}}(\theta)$ : The $K$ that maximizes $J^*(\theta, C_K^{\mathrm{unc}})$.
\end{itemize}

\subsection{Algorithm Notation}
\begin{itemize}
    \item $s_e$ : The number of rounds of the safe exploitation loop.
    \item $T_s = 2^s\sqrt{T}$ : The length \textbf{and} starting time of round $s$ of the safe exploitation phase. Note that $T_0 = \sqrt{T}$.
    \item $\epsilon_s$ : Uncertainty bound for $\theta^*$ in round $s$ of the for loop.
    \item $\hat{\theta}_s$ : An estimate of $\theta^*$ that is with high probability within $\epsilon_s$ distance of $\theta^*$
    \item $C_s^{\mathrm{alg}}(x_t)$ : the controller that the algorithm uses in round $s$ of the safe exploitation phase before safety adjustments
    \item $u_t^{\mathrm{safeL}}, u_t^{\mathrm{safeU}}$: bounds to enforce safety on the chosen control, which \\ is $u_t = \max\left(\min\left(C_s^{\mathrm{alg}}(x_t), u^{\mathrm{safeU}}_t\right),u^{\mathrm{safeL}}_t\right)$.
    \item $C^{\mathrm{alg}}$ : The controller of the algorithm.
\end{itemize}

\subsection{Proof Notation}
\begin{itemize}
    \item $ W_s = \{w_i\}_{i=   T_s}^{   T_{s+1}-1}$ : Noise random variables in the round $s$ of the safe exploitation phase.
    \item $(x'_0,x'_1,...)$ and $(u'_0,u'_1,...)$: Unless otherwise specified, these are the states and controls of the algorithm $C^{\mathrm{alg}}$.
    \item $(\hat{x}_{T_0}, \hat{x}_{T_0+1},...)$ : Unless otherwise defined in the theorem/lemma statement, $\hat{x}_{   T_0},\hat{x}_{   T_0+1},...$ is the sequence of states if the control at each time $t \ge T_0$ is $C^{\hat{\theta}_s}_{K_{\mathrm{opt}}(\hat{\theta}_s, T_s)}(x_t)$ for $s = \lfloor\log_2\left(\sqrt{T}\right)\rfloor$ and starting at $\hat{x}_{   T_0} = x'_{   T_0}$.
    \item $ E_{\mathrm{safe}} = \left\{ \forall t < T: D_{\mathrm{L}} \le a^*x'_t + b^*u'_t \le D_{\mathrm{U}}\right\}$ : The event that all of the controls satisfy the safety constraints.
    \item $E_1 = \left\{\forall t < T : |w_t| \le \log^2(T) \right\}$ : Event that all noise values have magnitude less than $\log^2(T)$ 
    \item $E_0 = \left\{\forall s \le s_e : \norm{\theta^* - \hat{\theta}_s}_{\infty} \le \epsilon_s \right\}$ : The event that all of the estimates of $\theta^*$ are within $\epsilon_s$ of $\theta^*$.
    \item $E_2 = E_0 \bigcap \left\{ \max_{s \in [0:s_e]} \epsilon_s \le \tilde{O}_T(T^{-1/4})\right\}$.
    \item $E_2^s = \left\{\norm{\hat{\theta}_{s} - \theta^*}_\infty \le \epsilon_s \le c_T \cdot T^{-1/4} \right\}$, where $c_T$ is the coefficient in the $\tilde{O}_T(T^{-1/4})$ of the definition of event $E_2$.
    \item $E = E_{\mathrm{safe}} \cap E_1 \cap E_2$
    \item $B_x = \log^3(T)$ : Used throughout the appendix to simplify notation.
    \item $K_{D_{\mathrm{U}}}^\theta$ : the value of $K$ that satisfies the equation $\frac{D_{\mathrm{U}}}{a-bK_{D_{\mathrm{U}}}^\theta} - D_{\mathrm{U}} = \bar{w}$.
\end{itemize}

\section{Proof sketch of Lemmas \ref{parameterization_assum2} and \ref{parameterization_assum3}}\label{app:proof_of_lemmas}

In order to use the results from \citep{schiffer2024stronger}, we must show Lemmas \ref{parameterization_assum2} and \ref{parameterization_assum3}. While both of these properties are relatively easy to show for the class of linear controllers, proving them for the class of truncated linear controllers is significantly more complicated. We first outline the proof of Lemma \ref{parameterization_assum3}. Lemma \ref{parameterization_assum3} compares the cost of two trajectories when using truncated linear controller $C_{K_{\mathrm{opt}}(\theta, t)}^{\theta}$, one trajectory starting at state $x$ and the other trajectory starting at state $x+\delta$. In the proof of Lemma \ref{parameterization_assum3}, we  show that the difference in states of the two trajectories will decrease at most (but not all) time steps. The difference does not decrease at all time steps because the difference between $\hat{\theta}$ and $\theta^*$ leads to low probability events where the difference between the states of the two trajectories increases (Lemma \ref{lemma:trunclincontbound}). We are able to bound the probability of the event that the difference in state increases, and this gives the desired result (Lemma \ref{lemma:bound_on_dt_both}). For Lemma \ref{parameterization_assum2}, we first show that the truncated linear controller $C^{\theta}_{K_{\mathrm{opt}}(\theta, t)}$ under dynamics $\theta$ has only $\tilde{O}_T(\norm{\theta - \theta^*}_{\infty})$ more cost than the truncated linear controller $C^{\theta^*}_{K_{\mathrm{opt}}(\theta^*, t)}$ under dynamics $\theta^*$. We then show that for any $K$, the truncated linear controller $C_K^\theta$ under dynamics $\theta^*$ for $t$ steps has only $\tilde{O}_T(\norm{\theta - \theta^*}_{\infty})$ more cost than $C_K^\theta$ under dynamics $\theta$ for $t$ steps. Combining these two results directly gives the desired result of Lemma \ref{parameterization_assum2}. For more details on these two proofs, see Appendices \ref{app:paramterization_assum2} and \ref{app:parameterization_assum3}.
\section{Proof of Lemma \ref{parameterization_assum3}}\label{app:parameterization_assum3}

\begin{proof}

Let $\delta_{\mathrm{L}\ref{parameterization_assum3}} = \frac{1}{\log^{10}(T)}$ and $\epsilon_{\mathrm{L}\ref{parameterization_assum3}} = \frac{1}{\log^{46}(T)}$

Define $\epsilon = \norm{\theta - \theta^*}_{\infty}$ and $\delta = |x-y|$. In order to bound the cost difference of the two trajectories, we will first bound the differences in states and controls of the two trajectories. We begin with the following lemma bounding the difference in future states when starting at two different initial states.

\begin{lemma}\label{lemma:bound_on_dt_both}
    In the setting of Problem \ref{problem_main}, for any $\theta \in \Theta$ such that $\epsilon := \norm{\theta - \theta^*}_{\infty} \le \frac{1}{\log^{46}(T)}$, $t \le T$, $W' = \{w_i\}_{i=0}^{t-1}$, and any $K \in [\frac{a-1}{b}, \frac{a}{b}]$, there exists $\mathcal{Y}_{\mathrm{L}\ref{lemma:bound_on_dt_both}} \in \mathbb{R}^t$ that only depends on $K$ and $\theta$ such that the event $E_{\mathrm{L}\ref{lemma:bound_on_dt_both}}(K, \theta, W') := \{W' \in \mathcal{Y}_{\mathrm{L}\ref{lemma:bound_on_dt_both}}\}$ satisfies $\P(E_{\mathrm{L}\ref{lemma:bound_on_dt_both}}(K, \theta, W')) = 1- o_T(1/T^{10})$ and the following holds. Suppose that $|x|, |y| \le 4\log^2(T)$ and $d := |x-y| \le \frac{1}{\log^{10}(T)}$. Define $d_i$ as the difference in state at time $i$ when starting at $x_0 = x$ versus starting at $x_0 = y$ and using controller $C_K^\theta \in \mathcal{C}^{\theta}_{\mathrm{tr}}$ with noise variables $W'$.  Then there exists an $L = \tilde{O}_T(1)$ such that for sufficiently large $T$, conditional on $E_{\mathrm{L}\ref{lemma:bound_on_dt_both}}(K, \theta, W')$, 
    \begin{equation}\label{eq:sat_22bb_main}
        d_i  \le 
        \begin{cases}
            2\xi^i \cdot d, & \text{for } L < i \le t\\
            4d + \tilde{O}_T( \epsilon)  & \text{for } 0 \le i \le L,
        \end{cases}
    \end{equation}
    where $\xi :=  \left(1 - \frac{1}{\log^{10}(T)}\right)$.
\end{lemma}
The proof of Lemma \ref{lemma:bound_on_dt_both} can be found in Appendix \ref{sec:proof_of_lemma:bound_on_dt_both}.

We can also bound the difference in control in terms of the difference in state.

\begin{lemma}\label{lemma:trunclincontboundcontroller}
     In the setting of Problem \ref{problem_main}, for any $\theta \in \Theta$ such that $\norm{\theta - \theta^*}_{\infty} \le \frac{1}{\log^{46}(T)}$, any $K \in [\frac{a-1}{b}, \frac{a}{b}]$, and any $x, y$ such that $d := |y-x| \le \frac{1}{\log^{10}(T)}$,
\begin{equation}\label{eq:sat_22a}
    |C^\theta_K(x) - C^\theta_K(y)| = O_T(d).
\end{equation}
\end{lemma}
The proof of Lemma \ref{lemma:trunclincontboundcontroller} can be found in Appendix \ref{sec:proof_of_lemma:trunclincontboundcontroller}.

We also will need the following event, which is a subset of the event $E_1$ applied only to times $i < t$.
\begin{definition}
    Define the event $E_1^t$ as the event that for all $i \le t-1$, $|w_i| \le \log^2(T)$.
\end{definition}

We can proceed by bounding the difference in total costs conditional on the event $E_{\mathrm{L}\ref{lemma:bound_on_dt_both}}(K, \theta, W') \cap E_1^t$. Let $d_0,d_1,...,d_{t}$ and $d^u_0,...,d^u_{t-1}$ respectively be the absolute difference in states and controls when starting at $x_0 = x$ versus starting at $x_0 = x+\delta$ and using controller $C_K^\theta$ with noise $W'$. Let $x_0,...,x_{t}$ and $u_0,...,u_{t-1}$ be the states and controls when using controller $C_K^\theta$ starting at $x_0 = x$ with noise $W'$. Then we have the following result conditional on $E_{\mathrm{L}\ref{lemma:bound_on_dt_both}}(K, \theta, W') \cap E_1^t$ for sufficiently large $T$:
\begin{align*}
    &\left|t \cdot J(\theta^*, C_K^\theta, t, x, W') - t \cdot J(\theta^*, C_K^\theta, t, x+\delta, W')\right| \\
    &\le  2qd_t|x_t| + qd_t^2 + \sum_{i=0}^{t-1} 2qd_i|x_i| + qd_i^2 + 2r|u_i|d^u_i + r\left(d^u_i\right)^2 \\
    &\le 2qd_t|x_t| + qd_t^2 + \sum_{i=0}^{t-1} 2qd_i|x_i| + qd_i^2 + 2r|u_i|O_T( d_i) + rO_T\left( d_i \right)^2 && \text{Lemma \ref{lemma:trunclincontboundcontroller}}\\
    &= O_T\left( \sum_{i=0}^{t} (d_i + d_i^2)\left(|x| + \norm{D}_{\infty} + \max_{w \in W'} |w|\right) \right) && \text{Lemma \ref{bounded_pos_cont}}\\
    &= \tilde{O}_T\left( \sum_{i=0}^{t} (d_i + d_i^2)\right) \quad \quad \quad \quad   \text{[Event $E_1^t$, $\norm{D}_{\infty} \le \log^2(T)$, $|x| \le 4\log^2(T)$]}\\
    &= \tilde{O}_T\left( \sum_{i=0}^{L} \left((4\delta+\tilde{O}_T(\epsilon)) + (4\delta+\tilde{O}_T(\epsilon ))^2\right)  +  \sum_{i=L+1}^{t} \left( 2\xi^i \delta + 4\xi^{2i} \delta^2 \right)\right) && \text{Eq \eqref{eq:sat_22bb_main}} \\
    &=  \tilde{O}_T\left(\delta+\epsilon + \delta\sum_{i=0}^{t} \xi^i + \delta^2\sum_{i=0}^{t} \xi^{2i}\right) \\
    &= \tilde{O}_T(\delta + \epsilon).
\end{align*}
The last line comes from the fact that $\xi = 1 - \frac{1}{\log^{10}(T)}$ and the formula for the sum of a geometric series. The above result holds conditional on event $E_{\mathrm{L}\ref{parameterization_assum3}}(K, \theta, W') := E_{\mathrm{L}\ref{lemma:bound_on_dt_both}}(K, \theta, W') \cap E_1^t$, and by a union bound and Equation \eqref{eq:e1bound},
\[
    \P(E_{\mathrm{L}\ref{parameterization_assum3}}(K, \theta, W')) = \P(E_{\mathrm{L}\ref{lemma:bound_on_dt_both}}(K, \theta, W') \cap E_1^t) =  1-o_T(1/T^{10}).
\]

\end{proof}

\subsection{Proof of Lemma \ref{lemma:bound_on_dt_both}}\label{sec:proof_of_lemma:bound_on_dt_both}

In order to prove Lemma \ref{lemma:bound_on_dt_both}, we will use the following lemma that has a similar result but holds conditional on an event that depends on $x$.

\begin{lemma}\label{lemma:conditional_version}
    There exists an $L = \tilde{O}_T(1)$ such that the following holds. Suppose that $|x|, |y| \le 4\log^2(T)$ and $d := |x-y| \le \frac{1}{\log^{10}(T)}$. In the setting of Problem \ref{problem_main}, for any $\theta \in \Theta$ such that $\epsilon := \norm{\theta - \theta^*}_{\infty} \le \frac{1}{\log^{46}(T)}$, $t \le T$, $W' = \{w_i\}_{i=0}^{t-1}$, and any $K \in [\frac{a-1}{b}, \frac{a}{b}]$, there exists $\mathcal{Y}_{\mathrm{L}\ref{lemma:conditional_version}} \in \mathbb{R}^t$ that only depends on $x$, $K$ and $\theta$ such that the event $E_{\mathrm{L}\ref{lemma:conditional_version}}(x, K, \theta, W') := \{W' \in \mathcal{Y}_{\mathrm{L}\ref{lemma:conditional_version}}\}$ satisfies $\P(E_{\mathrm{L}\ref{lemma:conditional_version}}(x, K, \theta, W')) = 1- o_T(1/T^{20})$ and the following holds. Define $d_i$ as the difference in state at time $i \le t$ when starting at $x_0 = x$ versus starting at $x_0 = y$ and using controller $C_K^\theta$ with noise variables $W'$.  Then for sufficiently large $T$, conditional on $E_{\mathrm{L}\ref{lemma:conditional_version}}(x, K, \theta, W')$, 
    \begin{equation}\label{eq:sat_22bb}
        d_i  \le 
        \begin{cases}
            \left(1 - \frac{1}{\log^{10}(T)}\right)^i \cdot d, & \text{if } i > L\\
           2d + \tilde{O}_T(\epsilon) & \text{if } i \le L.
        \end{cases}
    \end{equation}
\end{lemma}
The proof of Lemma \ref{lemma:conditional_version} can be found in Appendix \ref{sec:proof_of_lemma:conditional_version}.

Now we need to find a single event $E_{\mathrm{L}\ref{lemma:bound_on_dt_both}}(K, \theta, W')$ such that Equation \eqref{eq:sat_22bb_main} holds for all $|x|, |y| \le 4\log^2(T)$ under this event. Define the set
\[
G := \left\{-4\log^2(T) + \frac{i}{\log^{10}(T)}\right\}_{i \in [0:8\log^{12}(T)]},
\]
i.e. $G$ is a grid of points evenly spaced $\frac{1}{\log^{10}(T)}$ apart. Note that $|G| = \tilde{O}_T(1)$. Now, take 
\[
E_{\mathrm{L}\ref{lemma:bound_on_dt_both}}(K, \theta, W') = \bigcap_{g \in G} E_{\mathrm{L}\ref{lemma:conditional_version}}(g, K, \theta, W').
\]
First, we note that because $\P(E_{\mathrm{L}\ref{lemma:conditional_version}}(g, K, \theta, W')) = 1 - o_T(1/T^{20})$ for all $g$ and because $|G| = \tilde{O}_T(1)$, by a union bound $\P(E_{\mathrm{L}\ref{lemma:bound_on_dt_both}}) = 1 - o_T(1/T^{10})$.

Now, consider any $|x|, |y| \le 4\log^2(T)$ such that $|x-y| \le \frac{1}{\log^{10}(T)}$. Then there must exist some $g \in G$ such that $\max\left(|x-g|, |y-g|\right) \le \frac{1}{\log^{10}(T)}$. For this $g$, let $d_0^x, d_1^x,...$ be the sequence of differences of states when starting at state $g$ versus $x$ and using controller $C_K^\theta$ with noise $W'$, and likewise let $d_0^y, d_1^y,...,$ be the sequence of absolute differences of states when starting at state $g$ versus $y$ and using controller $C_K^\theta$ with noise $W'$. Conditional on event $E_{\mathrm{L}\ref{lemma:bound_on_dt_both}}(K, \theta, W')$, we have by Lemma \ref{lemma:conditional_version} that $\{d^x_i\}$ and $\{d^y_i\}$ both satisfy Equation \eqref{eq:sat_22bb}. Since $\{d^x_i\}$ and $\{d^y_i\}$ are both distances comparing to the same set of states starting at state $g$, we have by the triangle inequality that
\[
    d_i \le d_i^x + d_i^y.
\]
Therefore, for $i \le t$ we have the following, where $L$ is from Lemma \ref{lemma:conditional_version}:
    \begin{equation}\label{eq:sat_22bbb}
        d_i  \le 
        \begin{cases}
            2\left(1 - \frac{1}{\log^{10}(T)}\right)^i \cdot d, & \text{if } i > L\\
           4d +\tilde{O}_T(\epsilon) & \text{if } i \le L.
        \end{cases}
    \end{equation}
    This is exactly the desired result, and therefore we are done.

\subsection{Proof of Lemma \ref{lemma:conditional_version}}\label{sec:proof_of_lemma:conditional_version}
\begin{proof}

The main tool we will use for this proof is the following lemma that bounds the difference in future states in three different cases.

\begin{lemma}\label{lemma:trunclincontbound}
    For any $x, y$, define $d = |y-x|$. In the setting of Problem \ref{problem_main} and for sufficiently large $T$, suppose $\theta \in \Theta$, $K \in \left[\frac{a-1}{b}, \frac{a}{b}\right]$, and $\norm{\theta - \theta^*}_{\infty} = \epsilon \le \frac{1}{\log^{46}(T)}$. Then for some $\rho := |a^*-b^*K| + O_T(\epsilon)$,
    \begin{equation}
    |a^*x + b^*C_K^\theta(x) - a^*y - b^*C_K^\theta(y)| \le    \begin{cases}
            \min\left(2\rho d, \rho d + O_T(\epsilon)(|x|+\norm{D}_{\infty})\right) & \text{if $\mathcal{Z}$} \\
           O_T(\epsilon) d & \text{if $\mathcal{W}$} \\
            \rho d & \text{otherwise }
        \end{cases}
    \end{equation}
    \[
    \mathcal{Z} := \left\{\min(x,y) \le \frac{D_L}{a-bK} \le \max(x, y) \le \frac{D_U}{a-bK} \text{ or } \frac{D_L}{a-bK} \le \min(x,y) \le \frac{D_U}{a-bK} \le \max(x, y)\right\}
    \]
    \[
        \mathcal{W} :=  \left\{\max(x,y) \le \frac{D_L}{a-bK} \text{ or } \frac{D_U}{a-bK} \le \min(x, y)\right\}.
    \]
\end{lemma}
The proof of Lemma \ref{lemma:trunclincontbound} can be found in Appendix \ref{sec:proof_of_lemma:trunclincontboundcontroller}.

The rest of this proof will be structured as follows. First, we will introduce some additional definitions that we will use to construct event $E_{\mathrm{L}\ref{lemma:conditional_version}}(x, K, \theta, W')$. Then, we will prove Lemma \ref{lemma:conditional_version} in two cases.

Define $x_0,x_1,...,x_{T}$ as the sequence of states when starting at state $x_0 =x$ and using controller $C_K^\theta$ with noise $W'$. For $i \le t$, define the event 
\[
    X(i, x,K,\theta, W') := \left\{\min\left(\left|x_i - \frac{D_L}{a-bK}\right|, \left|x_i - \frac{D_U}{a-bK}\right|\right) \le \frac{3}{\log^{10}(T)} \right\}.
\]
Note that whether the event $ X(i, x,K,\theta, W')$ occurs depends on $w_0,...,w_{i-1}$. For $0 \le j \le t$ and $x \in \mathbb{R}$, define the event $H(j, x, K, \theta, W')$ as 
\[
    H(j, x, K, \theta, W') := \left\{\left|\{0\le i \le j : X(i,x,K\theta,W') \}\right| \le \log^{23}(T) + \frac{24B_Pj}{\log^{10}(T)}\right\}.
\]
Define
\[
    E^*(x, K, \theta, W') := \bigcap_{0 \le j \le t} H(j, x, K, \theta, W').
\]
Now we will show that $\P(E^*(x,K, \theta, W')) = 1-o_T\left(\frac{1}{T^{20}}\right)$. Fix any $j \le t$. If $j \le \log^{23}(T)$, then $H(j,x, K, \theta, W')$ holds with probability $1$ by definition. Now suppose $j > \log^{23}(T)$. Because $\mathcal{D}$ has a density bounded by $B_P$ and $x_i = a^*x_{i-1} + b^*u_{i-1} + w_{i-1}$, we must have that $\P( X(i, x,K,\theta, W')) \le \frac{12B_P}{\log^{10}(T)}$ for all $i$. Define $M_k = \sum_{i=0}^{k-1} 1_{X(i,x,K,\theta, W')} - \frac{12B_P}{\log^{10}(T)}$. For sufficiently large $T$, $M_k$ is a supermartingale with differences bounded in magnitude by $1$. Therefore, by the Azuma--Hoeffding inequality, with probability $1-o_T(1/T^{21})$,
\[
   \left|\{0\le i \le j :  X(i, x,K,\theta, W') \}\right| \le \frac{12B_P(j+1)}{\log^{10}(T)} + \log(T)\sqrt{j} \le \frac{24B_Pj}{\log^{10}(T)},
\]
where the last inequality holds for sufficiently large $T$ and $j > \log^{23}(T)$. Therefore, $\P(H(j,x,K,\theta,W')) \ge 1- o_T(1/T^{21})$. Taking a union bound over all $\log^{23}(T) < j \le t$ gives that $\P(E^*(x,K, \theta, W')) = 1-o_T(1/T^{20})$.

To prove Lemma \ref{lemma:conditional_version}, will split the range of potential $K$ into two parts, $K \in \left[\frac{a^*-1 + \frac{1}{\log^{9}(T)}}{b^*}, \frac{a}{b}\right]$ and $K \in \left[\frac{a-1}{b}, \frac{a^*-1 + \frac{1}{\log^{9}(T)}}{b^*}\right]$. We will also use the following bounds.
\begin{lemma}\label{relaxed_bounds}
    For any $\theta \in \Theta$ such that $\norm{\theta - \theta^*}_{\infty} \le \frac{1}{\log^{10}(T)}$, 
    \[
        \frac{a-1}{b} = \frac{a^* - 1 -O_T\left(\frac{1}{\log^{10}(T)}\right)}{b^*}
    \]
    and
    \[
        \frac{a}{b} = \frac{a^* + O_T\left(\frac{1}{\log^{10}(T)}\right)}{b^*}.
    \]
\end{lemma}
The proof of Lemma \ref{relaxed_bounds} can be found in Appendix \ref{sec:proof_of_relaxed_bounds}.

Now we are ready to proceed with the two cases for $K$.

\textbf{Case 1: }  $K \in \left[\frac{a^*-1 + \frac{1}{\log^{9}(T)}}{b^*}, \frac{a}{b}\right]$

For $i \le t$, define
\[
\mathcal{Z}_i := \left\{{\min(x_i,y_i) \le \frac{D_L}{a-bK} \le \max(x_i, y_i) \le \frac{D_U}{a-bK} \text{ or } \frac{D_L}{a-bK} \le \min(x_i,y_i) \le \frac{D_U}{a-bK} \le \max(x_i, y_i)} \right\}
\] 
and define 
\[
\kappa(j) = \left|\{0 \le i \le j : \mathcal{Z}_i \}\right|.
\]
Because Lemma \ref{relaxed_bounds} implies that $\frac{a}{b}  = \frac{a^*}{b^*} + O_T\left(\frac{1}{\log^{10}(T)}\right)$, we have for $K \in \left[\frac{a^*-1 + \frac{1}{\log^{9}(T)}}{b^*}, \frac{a}{b}\right]$ that $|a^*-b^*K| \le 1 - \frac{1}{\log^{9}(T)}$. Because $\epsilon \le \frac{1}{\log^{46}(T)}$, this implies that  $|a^*-b^*K|+O_T(\epsilon) \le 1- \frac{1}{2\log^9(T)}$. This will allow us to bound the $\rho$ in Lemma \ref{lemma:trunclincontbound} by $1- \frac{2}{\log^{9}(T)}$. Combining this with Lemma \ref{lemma:trunclincontbound}, we have the following piece-wise upper bound (note that we combined the $\mathcal{W}$ and the ``otherwise'' case using that $O_T(\epsilon)  \le 1- \frac{1}{2\log^9(T)}$ for suff large $T$),
\begin{equation}\label{eq:piecewisetrunc}
d_{j+1} \le    \begin{cases}
        \min\left(2\left(1- \frac{1}{2\log^9(T)}\right)  d_j, d_j + O_T(\epsilon)(|x_j|+\norm{D}_{\infty})\right) &  \text{if $\mathcal{Z}_j$}\\
        \left(1- \frac{1}{2\log^9(T)}\right) d_j &\text{otherwise. }
    \end{cases}
\end{equation}
Conditional on event $E_1^t$, for all $j\le t$, $|x_j| \le O_T(\log^2(T))$ by Lemma \ref{bounded_pos_cont} because $\norm{D}_{\infty} \le \log^2(T)$ and $|x| \le 4\log^2(T)$. Starting with the base case that $d_0 = d$, this with Equation \eqref{eq:piecewisetrunc} implies the following two relationships both hold for $d_{j+1}$ conditional on event $E_1^t$ for sufficiently large $T$. Equation \eqref{eq:dt_eq1a} holds because $\left(1- \frac{1}{2\log^9(T)}\right) \le 1$ for sufficiently large $T$ and using the second term in the $\min$ of Equation \eqref{eq:piecewisetrunc}. Equation \eqref{eq:dt_eq1b} holds using the first term in the $\min$ of Equation \eqref{eq:piecewisetrunc}.
\begin{equation}\label{eq:dt_eq1a}
    d_{j+1} \le d + \kappa(j)\cdot O_T(\epsilon\log^2(T))
\end{equation}
and
\begin{equation}\label{eq:dt_eq1b}
    d_{j+1} \le \left(\left(1- \frac{1}{2\log^9(T)}\right)^{j+1} 2^{\kappa(j)} \right) \cdot d.
\end{equation}
Equations \eqref{eq:dt_eq1a} and \eqref{eq:dt_eq1b} look almost like the desired result, and the remaining step is to show that $\kappa(j)$ is sufficiently ``small''. 

Next, define the event $A_j$ as
\[
    A_j := \Big\{\forall i \le \min(j, \log^{33}(T)) : d_i \le 2d + \epsilon\log^{36}(T)\Big\} \bigcap \left\{\forall \log^{33}(T) < i \le j : d_i \le \left(1 - \frac{1}{\log^{10}(T)}\right)^{i}d \right\}.
\]
By this construction, $A_t$ is exactly what we are trying to show in Lemma \ref{lemma:conditional_version} with $L = \log^{33}(T)$. We will now prove that $A_t$ holds for sufficiently large $T$ conditional on $E^*(x,K, \theta, W') \cap E_1^t$.

For sufficiently large $T$ and any $j \le t$, by construction of $A_j$ and because $d \le\frac{1}{\log^{10}(T)}$ and $\epsilon \le \frac{1}{\log^{46}(T)}$, we have that
\begin{equation}\label{eq:At}
    A_j \subseteq \left\{\forall 0 \le  i \le j : d_i \le \frac{3}{\log^{10}(T)} \right\}.
\end{equation}
Note that for event $\mathcal{Z}_i$ to hold, it must be the case that $x_i$ is within $d_i$ of either $\frac{D_U}{a-bK}$ or $\frac{D_L}{a-bK}$. Therefore, conditional on $E^*(x, K, \theta, W') \cap A_j$, we have for $j \ge \log^{33}(T)$, 
\begin{align*}
    \kappa(j) &= |\{0 \le i \le j : \mathcal{Z}_i\}| \\
    &\le \left|\left\{ 0 \le i \le j : \min \left(\left|x_i - \frac{D_U}{a-bK}\right|, \left| x_i - \frac{D_L}{a-bK}\right|\right) \le d_i \right\}\right| \\
    &\le \left|\left\{ 0 \le i \le j : \min \left(\left|x_i - \frac{D_U}{a-bK}\right|, \left| x_i - \frac{D_L}{a-bK}\right|\right) \le \frac{3}{\log^{10}(T)} \right\}\right| && \text{Equation \eqref{eq:At}} \\
    &= |\{0 \le i \le j : X(i,x,K,\theta,W')\}| \\
    &\le \log^{23}(T) + \frac{24B_Pj}{\log^{10}(T)}. && \text{$E^*(x,K,\theta, W')$} \\
    &= O_T\left(\frac{j+1}{\log^{10}(T)}\right)\numberthis \label{eq:Abound}
\end{align*}

We will now use Equations \eqref{eq:dt_eq1a} and \eqref{eq:dt_eq1b} to show that $A_{j+1}$ holds conditional on $E_1^t \cap E^*(x, K, \theta, W') \cap A_j$. In order to show that $A_{j+1}$ holds conditional on $A_j$, we must show that $d_{j+1}$ satisfies the necessary inequality in the definition of $A_{j+1}$. Consider the following two cases for $j \ge 0$.

If $j+1 \le \log^{33}(T)$, for sufficiently large $T$ conditional on $A_j \cap E_1^t \cap E^*(x,K, \theta, W')$,
\begin{align*}
d_{j+1} &\le d + \kappa(j) \cdot O_T(\epsilon \log^2(T)) && \text{Equation \eqref{eq:dt_eq1a}}\\
&= d + O_T(j\epsilon\log^2(T)) && \text{$\kappa(j) \le j+1$}\\
&\le d + O_T(\epsilon\log^{35}(T)) \\
&\le d + \log^{36}(T)\epsilon \\
&\le 2d + \log^{36}(T)\epsilon . \numberthis \label{eq:smallj}
\end{align*}
This is the necessary inequality that needs to be shown in order for $A_{j+1}$ to hold given that $A_j$ holds if $j+1 \le \log^{33}(T)$.

If $j+1 > \log^{33}(T)$, for sufficiently large $T$ conditional on $A_j \cap E_1^t \cap E^*(x,K, \theta, W')$,
\begin{align*}
    &d_{j+1} \\
    &\le \left(1- \frac{1}{2\log^9(T)}\right)^{j+1}2^{\kappa(j)} \cdot d && \text{Equation \eqref{eq:dt_eq1b}}\\
    &\le \left(1- \frac{1}{2\log^9(T)}\right)^{j+1}2^{O_T(\frac{j+1}{\log^{10}(T)})} \cdot d && \text{Equation \eqref{eq:Abound}} \\
    &= \left(1- \frac{1}{2\log^9(T)}\right)^{j+1}e^{O_T(\frac{j+1}{\log^{10}(T)})} \cdot d \\
    &\le \left(1- \frac{1}{2\log^9(T)}\right)^{j+1}\left(1+O_T\left(\frac{1}{\log^{10}(T)}\right) + O_T\left(\frac{1}{\log^{20}(T)}\right)\right)^{j+1} \cdot d && \text{     [$e^x \le 1 + x + x^2$ for $x \le 1$]}\\
    &= \left(1- \frac{1}{2\log^9(T)}\right)^{j+1}\left(1+O_T\left(\frac{1}{\log^{10}(T)}\right)\right)^{j+1} \cdot d \\
    &= \left(1-\frac{1}{2\log^9(T)}  +O_T\left(\frac{1}{\log^{10}(T)}\right) - \frac{1}{2\log^9(T)}\cdot O_T\left(\frac{1}{\log^{10}(T)}\right)\right)^{j+1} \cdot d \\
    &\le \left(1- \frac{1}{2\log^9(T)} +O_T\left(\frac{1}{\log^{10}(T)}\right) \right)^{j+1} \cdot d \\
    &\le \left(1 - \frac{1}{\log^{10}(T)}\right)^{j+1} \cdot d. && \text{for sufficiently large $T$} \numberthis \label{eq:largej}
\end{align*}
This is the necessary inequality that needs to be shown in order for $A_{j+1}$ to hold given that $A_j$ holds if $j + 1 \ge \log^{33}(T)$. 

Equations \eqref{eq:smallj} and \eqref{eq:largej} together imply that for sufficiently large $T$, $A_{j+1}$ holds conditional on $A_j \cap E_1^t \cap E^*(x,K, \theta, W')$. Note that $A_0$ always holds by definition because $d_0 = d$. Therefore, we can conclude by induction that $A_t$ must hold conditional on  $E_1^t \cap E^*(x,K, \theta, W')$ for sufficiently large $T$. Finally, by definition of $A_t$, this implies that conditional on $E^*(x,K, \theta, W') \cap E_1^t$ for sufficiently large $T$, for all $0 \le j \le t$,
\begin{equation} 
        d_j  \le 
        \begin{cases}
            \left(1 - \frac{1}{\log^{10}(T)}\right)^j \cdot d, & \text{if } j > \log^{33}(T)\\
           d + \tilde{O}_T( \epsilon), & \text{if } j \le \log^{33}(T).
        \end{cases}
    \end{equation}
Taking $E_{\mathrm{L}\ref{lemma:conditional_version}}(x, K, \theta, W') = E^*(x,K, \theta, W') \cap E_1^t$, by a union bound we have that $\P(E_{\mathrm{L}\ref{lemma:conditional_version}}(x, K, \theta, W') \ge 1-o_T(1/T^{20})$. This completes the proof of Lemma \ref{lemma:conditional_version} for Case 1.

\vspace{20mm}

\textbf{Case 2:} $K \in \left[\frac{a-1}{b}, \frac{a^*-1 + \frac{1}{\log^{9}(T)}}{b^*}\right]$

Define 
\[  
    \mathcal{W}_j := \left\{\min(x_j,y_j) \ge \frac{D_U}{a-bK} \text{ or } \max(x_j,y_j) \le \frac{D_L}{a-bK} \right\}
\]
and
\[
    \lambda(j) := |\{0 \le i \le j : \mathcal{W}_i\}|.
\]
For any $K \in \left[\frac{a-1}{b}, \frac{a^*-1 + \frac{1}{\log^{9}(T)}}{b^*}\right]$, we have that $|a^*-b^*K| - 1 \le O_T\left(\frac{1}{\log^{10}(T)}\right)$ by Lemma \ref{relaxed_bounds}. This with the fact that $\epsilon \le \frac{1}{\log^{46}(T)}$ implies that $(a^*-b^*K) + O_T(\epsilon) \le |a^*-b^*K| + O_T(\epsilon) \le 1 + O_T\left(\frac{1}{\log^{10}(T)}\right)$. This allows us to bound the $\rho$ in Lemma \ref{lemma:trunclincontbound} to be $1 + O_T\left(\frac{1}{\log^{10}(T)}\right)$. By Lemma \ref{lemma:trunclincontbound} and plugging this bound in for $\rho$, this gives the following bound.
{\fontsize{10}{10}
\begin{equation}\label{eq:H4_cases}
d_{j+1} \le    \begin{cases}
        O_T(\epsilon)\left(1+ O_T\left(\frac{1}{\log^{10}(T)}\right)\right) d_j & \text{If $\mathcal{W}_j$}\\
        \min\left(2\left(1+ O_T\left(\frac{1}{\log^{10}(T)}\right)\right)  d_j, \left(1+ O_T\left(\frac{1}{\log^{10}(T)}\right)\right)\left(d_j + O_T(\epsilon)(|x_j|+\norm{D}_{\infty})\right)\right) & \text{If $\mathcal{Z}_j$} \\
        \left(1+ O_T\left(\frac{1}{\log^{10}(T)}\right)\right) d_j & \text{Otherwise}\\
    \end{cases}
\end{equation}
}
Similar to in the proof of Case 1 above, by Lemma \ref{bounded_pos_cont} and the assumption that $\norm{D}_{\infty} \le \log^2(T)$, we have that conditional on event $E_1^t$, Equation \eqref{eq:H4_cases} implies the following two relationships. The first relationship comes from using the first term in the $\min$ of Equation \eqref{eq:H4_cases} and recursing.
\begin{equation}\label{eq:seconddt_bound_H2}
    d_{j+1} \le \left(1 + O_T\left(\frac{1}{\log^{10}(T)}\right)\right)^{j+1} \cdot 2^{\kappa(j)} \cdot O_T(\epsilon)^{\lambda(j)} \cdot d.
\end{equation}
 The second relationship comes from using the second term in the $\min$ of Equation \eqref{eq:H4_cases} and bounding $\left(1+ O_T\left(\frac{1}{\log^{10}(T)}\right)\right)(|x_j| + \norm{D}_{\infty}) = O_T(\log^2(T))$ under event $E_1^t$. This gives the recursive relationship of
 \begin{equation}\label{eq:H4_case_2}
     d_{j+1} \le \left(1+ O_T\left(\frac{1}{\log^{10}(T)}\right)\right)\left(O_T(\epsilon)\right)^{1_{\mathcal{W}_j}} \cdot d_{j} + O_T(\epsilon\log^2(T))1_{\mathcal{Z}_j}.
 \end{equation}
 In other words, at every step there is a multiplicative factor of $\left(1+ O_T\left(\frac{1}{\log^{10}(T)}\right)\right)$. When $\mathcal{W}_j$ holds, there is an additional multiplicative factor of $O_T(\epsilon)$. When $\mathcal{Z}_j$ holds, there is an additive factor of $O_T(\epsilon\log^2(T))$. Unwrapping Equation \eqref{eq:H4_case_2} gives that, at time $j+1$, any additive factor contributed at time $i \le j$ will be scaled by $ O_T(\epsilon)^{\lambda(j) - \lambda(i)}\left(1 + O_T\left(\frac{1}{\log^{10}(T)}\right)\right)^{j - i}$. This gives that
 {\fontsize{10}{10}
\begin{equation}\label{eq:seconddt_bound}
    d_{j+1} \le \left(1 + O_T\left(\frac{1}{\log^{10}(T)}\right)\right)^{j+1}O_T(\epsilon)^{\lambda(j)} \cdot d + O_T(\epsilon \log^2(T))  \cdot \sum_{i=0}^j 1_{\mathcal{Z}_i} O_T(\epsilon)^{\lambda(j) - \lambda(i)}\left(1 + O_T\left(\frac{1}{\log^{10}(T)}\right)\right)^{j - i}.
\end{equation}
}
Again this almost looks like the desired result, except we need to show that the additional terms involving $\kappa(j)$ and $\lambda(j)$ are not ``too large".  We will use the following lemma that lower bounds $\lambda(j)$ using the same event $A_j$ as defined above in the first case. Similar to Case 1, we will then use this to show that for sufficiently large $T$, $A_t$ holds conditional on $E_1^t \cap E_{\mathrm{L}\ref{Ht_bound}}(x,K, \theta, W')\cap E^*(x, K, \theta, W')$.

    \begin{lemma}\label{Ht_bound}
        Suppose $|1 - (a^*-b^*K)| = O_T\left( \frac{1}{\log^{9}(T)}\right)$. Then in the setting of Problem \ref{problem_main} and using the notation and assumptions of Lemma \ref{lemma:conditional_version}, there exists a $\mathcal{Y}_{\mathrm{L}\ref{Ht_bound}} \in \mathbb{R}^t$ that only depends on $x,K,\theta$ such that the event $E_{\mathrm{L}\ref{Ht_bound}}(x,K, \theta, W') := \{W' \in \mathcal{Y}_{\mathrm{L}\ref{Ht_bound}}\}$ satisfies $\P(E_{\mathrm{L}\ref{Ht_bound}}(x,K, \theta, W')) = 1 - o_T(1/T^{20})$ and that for all $ t_1< t_2 \le t$ satisfying $t_2 - t_1 \ge \log^8(T)$, the following is true conditional on event $A_{t_2} \cap E_{\mathrm{L}\ref{Ht_bound}}(x,K, \theta, W')$ for sufficiently large $T$:
        \[
            \left(1 + O_T\left(\frac{1}{\log^{10}(T)}\right)\right)^{t_2 + 1 - t_1} O_T(\epsilon)^{\lambda(t_2) - \lambda(t_1)} \le \left(1 - \frac{1}{2\log^9(T)}\right)^{t_2 + 1 - t_1}.
        \]

    \end{lemma} 
    The proof of Lemma \ref{Ht_bound} can be found in Appendix \ref{sec:proof_of_Ht_bound}.

We will now show that the event $A_{j+1}$ holds conditional on $E_{\mathrm{L}\ref{Ht_bound}}(x,K, \theta, W') \cap E^*(x, K, \theta, W') \cap E_1^t \cap A_j$.

For $j < \log^{8}(T)$,  conditional on $E_1^t \cap E_{\mathrm{L}\ref{Ht_bound}}(x,K, \theta, W')\cap E^*(x, K, \theta, W') \cap A_j$ and for sufficiently large $T$,
{\fontsize{10}{10}
    \begin{align*}
        &d_{j+1} \\
        &\le \left(1 + O_T\left(\frac{1}{\log^{10}(T)}\right)\right)^{j+1}O_T(\epsilon)^{\lambda(j)} \cdot d \\
        & \quad \quad + O_T(\epsilon \log^2(T))  \cdot \sum_{i =0}^j 1_{\mathcal{Z}_j} O_T(\epsilon)^{\lambda(j) - \lambda(i)}\left(1 + O_T\left(\frac{1}{\log^{10}(T)}\right)\right)^{j - i}  && \text{ Eq. \eqref{eq:seconddt_bound}}\\
        &\le \left(1 + O_T\left(\frac{1}{\log^{10}(T)}\right)\right)^{\log^{8}(T)+1} \cdot d + O_T(\epsilon \log^2(T))  \cdot \sum_{i =0}^j\left(1 + O_T\left(\frac{1}{\log^{10}(T)}\right)\right)^{j} && \text{$\epsilon \le O_T(1)$}\\
        &\le  \left(1 + O_T\left(\frac{1}{\log^2(T)}\right)\right) \cdot d + O_T(\epsilon \log^2(T))  \cdot (j+1) \cdot \left(1 + O_T\left(\frac{1}{\log^{10}(T)}\right)\right)^{j} && \text{Lemma \ref{lemma:inequality_for_f_and_g}} \\
        &\le  \left(1 + O_T\left(\frac{1}{\log^2(T)}\right)\right) \cdot d + O_T(\epsilon \log^2(T))  \cdot (\log^8(T)+1) \cdot \left(1 + O_T\left(\frac{1}{\log^{10}(T)}\right)\right)^{\log^8(T)} \\
        &\le \left(1 + O_T\left(\frac{1}{\log^2(T)}\right)\right)\left(d + O_T\left(\epsilon \log^{10}(T)\right)\right) && \text{Lemma \ref{lemma:inequality_for_f_and_g}} \\
        & \le 2d + O_T(\epsilon\log^{10}(T)) && \text{Suff. large $T$} \\
        &\le 2d + \epsilon \log^{36}(T).  && \text{Suff. large $T$} \numberthis \label{eq:smallt}
    \end{align*}
}

Above, we used the following result:
\begin{lemma}\label{lemma:inequality_for_f_and_g}
    Suppose $g(T)$ is a non-negative function of $T$ such that $g(T) > 1$ for sufficiently large $T$. Furthermore, suppose $f(T)$ is a non-negative function of $T$ that satisfies 
$f(T)g(T) \le 1/2$ for sufficiently large $T$. Then we have that
    \[
      1 + f(T)g(T) \le  (1 + f(T))^{g(T)} \le 1 + 2f(T)g(T).
    \]
    This implies that
    \[
        (1 + f(T))^{g(T)} = 1 + \Theta_T(f(t)\cdot g(T)).
    \]
\end{lemma}
\begin{proof}
    First, we note that for any $x \ge 0$ and $r > 1$, $(1+x)^r \ge 1 + rx$. This implies that for sufficiently large $T$, we have that
    \[
        (1+f(T))^{g(T)} \ge 1 + f(T)g(T).
    \]
    This proves one direction of the desired equation. For the other direction, note that for $r > 0$ and $x \in [0,1/r)$, we have $(1+x)^r \le \frac{1}{1-rx}$. This  implies that
    \begin{align*}
        (1+f(T))^{g(T)}  &\le \frac{1}{1 - f(T)g(T)} \\
        &= 1 + \frac{f(T)g(T)}{1 - f(T)g(T)} \\
        &\le 1 + 2f(T)g(T).
    \end{align*}
    This proves the other direction of the desired equation. Therefore we have that $(1+f(T))^{g(T)} = 1  + \Theta(f(T)g(T))$.
\end{proof}

For $\log^{8}(T) \le j < \log^{33}(T)$, conditional on event $E_1^t \cap E_{\mathrm{L}\ref{Ht_bound}}(x,K, \theta, W')\cap E^*(x, K, \theta, W') \cap A_j$ and for sufficiently large $T$,
{\fontsize{10}{10}
\begin{align*}
        &d_{j+1} \\
        &\le \left(1 + O_T\left(\frac{1}{\log^{10}(T)}\right)\right)^{j+1}O_T(\epsilon)^{\lambda(j)} \cdot d \\
        &\quad \quad + O_T(\epsilon \log^2(T))  \cdot \sum_{i=0}^j 1_{\mathcal{Z}_j} O_T(\epsilon)^{\lambda(j) - \lambda(i)}\left(1 + O_T\left(\frac{1}{\log^{10}(T)}\right)\right)^{j - i}  && \text{Eq \eqref{eq:seconddt_bound}}\\
        &\le \left(1 + O_T\left(\frac{1}{\log^{10}(T)}\right)\right)^{j+1-0}O_T(\epsilon)^{\lambda(j) - \lambda(0)} \cdot d  \\
        &\quad \quad+ O_T(\epsilon \log^2(T))  \cdot \sum_{i=0}^j 1_{\mathcal{Z}_j} O_T(\epsilon)^{\lambda(j) - \lambda(i)}\left(1 + O_T\left(\frac{1}{\log^{10}(T)}\right)\right)^{j - i} \\
        &\le \left(1 - \frac{1}{2\log^9(T)}\right)^{j+1} d  \\
        &\quad \quad+ O_T(\epsilon \log^2(T))  \cdot \sum_{i=0}^j 1_{\mathcal{Z}_j} O_T(\epsilon)^{\lambda(j) - \lambda(i)}\left(1 + O_T\left(\frac{1}{\log^{10}(T)}\right)\right)^{j - i}  && \text{Lemma \ref{Ht_bound}} \\
        &\le d + O_T(\epsilon \log^2(T))  \sum_{i=0}^{\lceil j - \log^8(T) \rceil-1}  1_{\mathcal{Z}_j}O_T(\epsilon)^{\lambda(j) - \lambda(i)}\left(1 + O_T\left(\frac{1}{\log^{10}(T)}\right)\right)^{j + 1 - i} \\
         & \quad \quad \quad  + O_T(\epsilon \log^2(T)) \sum_{ i = \lceil j - \log^8(T) \rceil}^j\left(1 + O_T\left(\frac{1}{\log^{10}(T)}\right)\right)^{j - i} \\
        &\le d + O_T(\epsilon \log^2(T)) \cdot \sum_{i=0}^{\lceil j - \log^8(T) \rceil-1}   \left( 1_{\mathcal{Z}_j} \cdot \left(1 - \frac{1}{2\log^9(T)}\right)^{j+1-i} \right) \\
        & \quad \quad \quad +  O_T(\epsilon \log^{2}(T)) \cdot  \sum_{ i = \lceil j - \log^8(T) \rceil}^j\left(1 + O_T\left(\frac{1}{\log^{10}(T)}\right)\right)^{j-i}&& \text{Lemma \ref{Ht_bound}}\\
        &\le d + O_T(\epsilon \log^2(T)) \cdot\sum_{i=0}^{\lceil j - \log^8(T) \rceil-1}   \left( 1_{\mathcal{Z}_j} \cdot O_T(1) \right) \\
        & \quad \quad \quad +  O_T(\epsilon \log^{2}(T)) \cdot \sum_{ i = \lceil j - \log^8(T) \rceil}^j\left(1 + O_T\left(\frac{1}{\log^{10}(T)}\right)\right)^{j-i} \\
        &\le d + O_T(\epsilon \log^2(T)) \cdot\sum_{i=0}^{\lceil j - \log^8(T) \rceil-1}  \left( 1_{\mathcal{Z}_j} \cdot O_T(1) \right) \\
        & \quad \quad \quad +  O_T(\epsilon \log^{2}(T)) \cdot  \sum_{ i = \lceil j - \log^8(T) \rceil}^j\left(1 + O_T\left(\frac{1}{\log^{10}(T)}\right)\right)^{\log^8(T)} \\       
        &\le d + O_T(\epsilon \log^2(T)) \cdot \sum_{i=0}^{\lceil j - \log^8(T) \rceil-1}  \left( 1_{\mathcal{Z}_j} \cdot O_T(1) \right) \\
        & \quad \quad \quad +  O_T(\epsilon \log^{2}(T)) \cdot  \sum_{ i = \lceil j - \log^8(T) \rceil}^j\left(1 + O_T\left(\frac{1}{\log^{2}(T)}\right)\right) && \text{Lemma \ref{lemma:inequality_for_f_and_g}} \\ 
        &\le d + O_T(\epsilon \log^{35}(T))  +  O_T\left(\epsilon \log^{10}(T)\right) \\
        &\le d + O_T(\epsilon \log^{35}(T)) && \\
        &\le d + \epsilon \log^{36}(T). && \text{Suff large $T$}
\end{align*}
}
Finally, for $j \ge \log^{33}(T)$,  conditional on event $E_1^t \cap E_{\mathrm{L}\ref{Ht_bound}}(x,K, \theta, W')\cap E^*(x, K, \theta, W') \cap A_j$ and for sufficiently large $T,$
\begin{align*}
    d_{j+1} &\le \left(1 + O_T\left(\frac{1}{\log^{10}(T)}\right)\right)^{j+1} \cdot 2^{\kappa(j)} \cdot O_T(\epsilon)^{\lambda(j)} \cdot d && \text{Equation \eqref{eq:seconddt_bound_H2}}\\
    &\le \left(1 + O_T\left(\frac{1}{\log^{10}(T)}\right)\right)^{j+1-0} \cdot O_T(\epsilon)^{\lambda(j) - \lambda(0)} \cdot 2^{\kappa(j)}  \cdot d\\
    &\le \left(1 - \frac{1}{2\log^{9}(T)}\right)^{j+1}2^{\kappa(j)} \cdot d && \text{Lemma \ref{Ht_bound}} \\
    &\le \left(1 - \frac{1}{\log^{10}(T)}\right)^{j+1} \cdot d. && \text{As in Equation \eqref{eq:largej}}
\end{align*}

Combining all three cases, we have that for all $j \ge 0$, conditional on $E_1^t \cap E_{\mathrm{L}\ref{Ht_bound}}(x,K, \theta, W')\cap E^*(x, K, \theta, W') \cap A_j$, $A_{j+1}$ holds. As in Case 1, we can conclude by induction using $A_0$ as the base case to get that conditional on $E_{\mathrm{L}\ref{Ht_bound}}(x,K, \theta, W') \cap E^*(x,K, \theta, W') \cap E_1^t$, the event $A_t$ holds, which implies that 
\begin{equation} 
        d_j  \le 
        \begin{cases}
            \left(1 - \frac{1}{\log^{10}(T)}\right)^j \cdot d, & \text{if } j > \log^{33}(T)\\
           2d + \tilde{O}_T( \epsilon), & \text{if } j \le \log^{33}(T).
        \end{cases}
    \end{equation}
Taking $E_{\mathrm{L}\ref{lemma:conditional_version}}(x, K, \theta, W') = E_{\mathrm{L}\ref{Ht_bound}}(x,K, \theta, W') \cap E^*(x,K, \theta, W') \cap E_1^t$, we have by a union bound that $\P(E_{\mathrm{L}\ref{lemma:conditional_version}}(x, K, \theta, W')) = 1-o_T(1/T^{20})$. This completes the proof of Lemma \ref{lemma:conditional_version} for Case 2.

\end{proof}

\subsection{Proof of Lemma \ref{lemma:trunclincontboundcontroller} and Lemma \ref{lemma:trunclincontbound}}\label{sec:proof_of_lemma:trunclincontboundcontroller}
\begin{proof}
    We have four cases depending on the values of $x,y$. We will prove the results of Lemma \ref{lemma:trunclincontboundcontroller} and Lemma  \ref{lemma:trunclincontbound} for each of these cases separately. WLOG assume that $x \le y$.

    \textbf{Case 1:} $\frac{D_L}{a-bK} \le x \le y \le \frac{D_U}{a-bK}$.

    In this case, $C_K^\theta(x) = -Kx$ and $C_K^\theta(y) = -Ky$, and therefore the following two equations hold.
    Case 1 Lemma \ref{lemma:trunclincontbound}:
    \[
    |a^*x + b^*C_K^\theta(x) - a^*y - b^*C_K^\theta(y)| = |a^*-b^*K|d = (|a^*-b^*K| + O_T(\epsilon))d.
    \]
Case 1 Lemma \ref{lemma:trunclincontboundcontroller}:
\[
        |C_K^\theta(x) - C_K^\theta(y)| = Kd \le \frac{a}{b} \cdot d \le \frac{\bar{a}}{\underline{b}} \cdot d= O_T(d).
    \]

     \vspace{5mm}

     \textbf{Case 2:} $\frac{D_U}{a-bK} \le x \le y$ or $x \le y \le \frac{D_L}{a-bK}$ (which is $\mathcal{W}$ of Lemma \ref{lemma:trunclincontbound}).

     First, assume the former is true. Then $C_K^\theta(x) = \frac{D_U - ax}{b}$ and likewise $C_K^\theta(y) = \frac{D_U - ay}{b}$. Therefore the following equations hold.
     
    Case 2 Lemma \ref{lemma:trunclincontbound}:
\begin{align*}
    |a^*x + b^*C_K^\theta(x) - a^*y - b^*C_K^\theta(y)| &= d\left|a^* - \frac{a}{b}b^*\right| \\
    &= d\left|\frac{a^*b-ab^*}{b}\right|\\
    &\le d\frac{\max\left((a+\epsilon)b - a(b-\epsilon), |(a-\epsilon)b - a(b+\epsilon)|\right)}{b}\\
    &= d \frac{\epsilon b + \epsilon a}{b} \\
    &\le d \frac{\epsilon \bar{b} + \epsilon \bar{a}}{\underline{b}} \\
    &\le O_T(\epsilon)d.
\end{align*}
Case 2 Lemma \ref{lemma:trunclincontboundcontroller}:
\[
    |C_K^\theta(y) - C_K^\theta(x)| = \frac{a}{b} \cdot d \le \frac{\bar{a}}{\underline{b}} \cdot d= O_T(d).
\]
The same logic holds for when $x \le y \le \frac{D_L}{a-bK}$.

    \vspace{5mm}

    \textbf{Case 3:} $x \le \frac{D_L}{a-bK}$ and $y \ge \frac{D_U}{a-bK}$

In this case, $C_K^\theta(x) = \frac{D_L - ax}{b}$ and $C_K^\theta(y) = \frac{D_U - ay}{b}$. 
We will use the fact that $|a-bK|d = |a-bK||y-x| \ge |D_U - D_L|$ in this case.

Case 3 Lemma \ref{lemma:trunclincontbound}:
\begin{align*}
    &| a^*x + b^*C_K^\theta(x) - a^*y - b^*C_K^\theta(y)| \\
    &= \left|\frac{b^*}{b}\left(D_L - D_U\right) + \left(a^* - \frac{a}{b}b^*\right)(x-y)\right| \\
    &\le \frac{b^*}{b}|D_U - D_L| +  \left|a^* - \frac{a}{b}b^*\right|d \\
    &\le \frac{b^*}{b}|a-bK|d +  \left|a^* - \frac{a}{b}b^*\right|d \\
    &\le \frac{b^*}{b}|a^*-b^*K|d + \frac{b^*}{b}\left|a-a^*+(b^*-b)K\right|d +  \left|a^* - \frac{a}{b}b^*\right|d \\
    &\le |a^*-b^*K|d + \left|\frac{b^*}{b} - 1\right||a^*-b^*K|d +  \frac{b^*}{b}|a-a^*+(b^*-b)K|d +  \left|a^* - \frac{a}{b}b^*\right|d \\
    &\le \left(|a^*-b^*K| + O_T(\epsilon)\right)d. \numberthis \label{eq:case3dudl}
\end{align*}
In the last line we used that $|a^*-b^*K| \le a^* + b^*|K| \le \bar{a} + \bar{b}\frac{\bar{a} + 1}{\underline{b}} = O_T(1)$, that $|\frac{b^*}{b} - 1| \le \frac{\epsilon}{\underline{b}} = O_T(\epsilon)$, that $|a-a^* + (b^*-b)K| \le \epsilon(1+|K|) \le \epsilon(1+\frac{\bar{a}+1}{\underline{b}})=  O_T(\epsilon)$, and that $|a^* - \frac{a}{b}b^*| \le \epsilon + \frac{a}{b}\epsilon \le \epsilon + \frac{\bar{a}}{\underline{b}}\epsilon= O_T(\epsilon)$. 

    Case 3 Lemma \ref{lemma:trunclincontboundcontroller}:

\begin{align*}
    |C_K^\theta(y) - C_K^\theta(x)| &=\left|\frac{1}{b}\left(D_U - D_L\right) + \frac{a}{b}(x-y)\right| \\
    &=\frac{1}{b}\left|D_U - D_L\right| + \frac{a}{b}|x-y| \\
    &\le \frac{1}{b}|a-bK|d + \frac{a}{b}d \\
    &\le \frac{1}{\underline{b}}d + \frac{\bar{a}}{\underline{b}}d \\
    &= O_T(d).
    \end{align*}
    \vspace{5mm}
    
    \textbf{Case 4:}  If $\frac{D_L}{a-bK} \le x \le \frac{D_U}{a-bK}$ and $y \ge \frac{D_U}{a-bK}$. Note that by symmetry, this is equivalent to $\frac{D_L}{a-bK} \le y \le \frac{D_U}{a-bK}$ and $x \le \frac{D_L}{a-bK}$. We will first assume the former. For Lemma \ref{lemma:trunclincontbound}, this case is equivalent to $\mathcal{Z}$.

    \vspace{5mm}
    Case 4 Lemma \ref{lemma:trunclincontbound}:

    In this case, $C_K^\theta(x) = -Kx$ and $C_K^\theta(y) = \frac{D_U - ay}{b}$. Furthermore, in this case $|y-x| \ge \left|y - \frac{D_U}{a-bK}\right|$. Therefore, in this case we have
    \begin{align*}
    &| a^*x + b^*C_K^\theta(x) - a^*y - b^*C_K^\theta(y)| \\
    &= |a^*x + b^*C_K^\theta(x) - a^*y -b^*Ky + b^*Ky - b^*C_K^\theta(y)| \\
    &\le |(a^*-b^*K)x - (a^*-b^*K)y| + b^*\left|-Ky - \frac{D_U-ay}{b}\right| \\
    &\le |(a^*-b^*K)x - (a^*-b^*K)y| + b^*\left|\frac{(a-bK)y-D_U}{b}\right| \\
    &\le |(a^*-b^*K)x - (a^*-b^*K)y| + \frac{b^*|a-bK|}{b}\left|y-\frac{D_U}{a-bK}\right| \\
    &\le |(a^*-b^*K)x - (a^*-b^*K)y| + \frac{b^*|a-bK|}{b}\left|y-x\right| \\
    &= |a^*-b^*K|d + |a-bK|\frac{b^*}{b}d \\
    &\le |a^*-b^*K|d + |a-bK|d + \left| 1 -\frac{b^*}{b} \right||a-bK|d \\
    &\le 2|a^*-b^*K|d + |a-bK-(a^*-b^*K)|d + \left| 1 -\frac{b^*}{b} \right||a-bK|d \\
    &\le 2(|a^*-b^*K| + O_T(\epsilon))d .&& \text{As in Equation \eqref{eq:case3dudl}} 
\end{align*}
Alternatively, note that in this case, 
\begin{equation}\label{eq:astar1}
    (a^*-b^*K)|x| \le (a-bK)|x| + O_T(\epsilon)|x|
\end{equation}
and 
\begin{equation}\label{eq:astar2}
    (a-bK)x \le D_U \le (a-bK)y.
\end{equation}
Therefore,
\begin{align*}
   \left|(a^*-b^*K)x -D_U \right| &\le \left|(a-bK)x - D_U\right| + O_T(\epsilon)|x| && \text{Equation \eqref{eq:astar1}}\\
   &\le \left| (a-bK)x -  (a-bK)y \right| + O_T(\epsilon)|x| && \text{Equation \eqref{eq:astar2}}\\
   &\le  |a-bK|d + O_T(\epsilon)|x| \\
   &\le |a^*-b^*K|d + O_T(\epsilon)(d+|x|). \numberthis \label{eq:DU_equation}
\end{align*}
Therefore we can find an alternative bound on $| a^*x + b^*C_K^\theta(x) - a^*y - b^*C_K^\theta(y)|$, using Equation \eqref{eq:DU_equation} and that $|y| \le |x| + d$.

    \begin{align*}
    &| a^*x + b^*C_K^\theta(x) - a^*y - b^*C_K^\theta(y)| \\
    &= \left|(a^*-b^*K)x - \frac{b^*}{b}D_U - \left(a^* - \frac{ab^*}{b}\right)y\right| \\
    &\le\left|(a^*-b^*K)x -D_U \right| +\left|1 - \frac{b^*}{b} \right|D_U +  \left| a^* - \frac{ab^*}{b}\right| |y|\\
    &\le |a^*-b^*K|d + O_T(\epsilon)(d+|x|) +\left|1 - \frac{b^*}{b} \right|D_U +  \left| a^* - \frac{ab^*}{b}\right| |y| && \text{Equation \eqref{eq:DU_equation}}\\
    &\le |a^*-b^*K|d+  O_T(\epsilon)(d+|x|) + \left|1 - \frac{b^*}{b} \right|D_U + \left| a^* - \frac{ab^*}{b}\right| \left(|x| + d\right)\\
    &\le (|a^*-b^*K| +  O_T(\epsilon))d + O_T(\epsilon)(|x|+D_U)\\
    &\le (|a^*-b^*K| +  O_T(\epsilon))d + O_T(\epsilon)(|x|+\norm{D}_{\infty}).
\end{align*}
where in the last line we once again bounded $|1 - \frac{b^*}{b}| = O_T(\epsilon)$ and $|a^* - \frac{ab^*}{b}| = O_T(\epsilon)$. Therefore, we have shown in this case that 
\begin{align*}
    &|a^*x + b^*C_K^\theta(x) - a^*y - b^*C_K^\theta(y)|\\
    &\le \min \left(2(|a^*-b^*K| + O_T(\epsilon))d ,(|a^*-b^*K| +  O_T(\epsilon))d + O_T(\epsilon)(|x|+\norm{D}_{\infty})\right) 
\end{align*}

    Case 4 Lemma \ref{lemma:trunclincontboundcontroller}:
\begin{align*}
    |C_K^\theta(x) - C_K^\theta(y)| &=\left|-Kx - \frac{D_U-ay}{b} \right| \\
    &\le |K||x-y| + \left| -Ky - \frac{D_U-ay}{b}\right| \\
    &\le |K||x-y| +\left| \frac{(a-bK)y - D_U}{b}\right| \\
    &\le |K||x-y| + \frac{|a-bK|}{b}\left| y - \frac{D_U}{a-bK}\right| \\
    &\le |K||x-y| + \frac{|a-bK|}{b}\left| y - x\right|  && \text{Equation \eqref{eq:astar2}}\\
    &= |K|d + \frac{|a-bK|}{b}d \\
    &\le \frac{\bar{a}+1}{\underline{b}}d + \frac{1}{\underline{b}}d \\
    &= O_T(d).
    \end{align*}

Because these four cases cover all possible situations, we have shown the desired two lemmas.
\end{proof}

\subsection{Proof of Lemma \ref{relaxed_bounds}}\label{sec:proof_of_relaxed_bounds}

\begin{proof}
    For sufficiently large $T$ we have the following two results, using that $\norm{\theta - \theta^*}_{\infty} \le \frac{1}{\log^{10}(T)}$:
    \begin{align*}
        \frac{a-1}{b} &\ge \frac{a^*-\frac{1}{\log^{10}(T)} - 1}{b^* + \frac{1}{\log^{10}(T)}} \\
        &= \frac{a^*-1}{b^*} \cdot \frac{b^*}{b^* + \frac{1}{\log^{10}(T)}} - \frac{1}{\log^{10}(T)(b^* + \frac{1}{\log^{10}(T)})} \\
        &= \frac{a^*-1}{b^*}\cdot \left(1 - \frac{1}{\log^{10}(T)\left(b^* +\frac{1}{\log^{10}(T)}\right)}\right) -  \frac{1}{\log^{10}(T)(b^* + \frac{1}{\log^{10}(T)})} \\
        &= \frac{a^*-1}{b^*} - \frac{a^*-1}{b^*\log^{10}(T)(b^* + \frac{1}{\log^{10}(T)})} -  \frac{1}{\log^{10}(T)(b^* + \frac{1}{\log^{10}(T)})} \\
        &= \frac{a^* - 1 -O_T\left(\frac{1}{\log^{10}(T)}\right)}{b^*}.
    \end{align*}
    \begin{align*}
       \frac{a}{b} & \le \frac{a^*+\frac{1}{\log^{10}(T)} }{b^* - \frac{1}{\log^{10}(T)}} \\
        &= \frac{a^*}{b^*}  \cdot \frac{b^*}{b^* - \frac{1}{\log^{10}(T)}}  + \frac{1}{\log^{10}(T)(b^* - \frac{1}{\log^{10}(T)})} \\
        &= \frac{a^*}{b^*}  \cdot \left( 1  + \frac{1}{\log^{10}(T)\left(b^* - \frac{1}{\log^{10}(T)}\right)}  \right) + \frac{1}{\log^{10}(T)(b^* - \frac{1}{\log^{10}(T)})} \\
        &= \frac{a^*}{b^*}  + \frac{a^*}{b^*(b^* - \frac{1}{\log^{10}(T)})\log^{10}(T)}  + \frac{1}{\log^{10}(T)(b^* - \frac{1}{\log^{10}(T)})} \\
        &= \frac{a^* + O_T(1/\log^{10}(T))}{b^*}.
    \end{align*}

\end{proof}

    \subsection{Proof of Lemma \ref{Ht_bound}}\label{sec:proof_of_Ht_bound}
    \begin{proof}
    The first step to this proof is to construct event $E_{\mathrm{L}\ref{Ht_bound}}(x,K,\theta,W')$. For any $t_2 > t_1$ and $t_2 - t_1 \ge \log^8(T)$, define the event $E^{t_1,t_2}_{\mathrm{L}\ref{Ht_bound}}$ as 
    \[
        E^{t_1,t_2}_{\mathrm{L}\ref{Ht_bound}} =  \left\{\exists j \in [t_1:t_2-\lceil \log^5(T) \rceil-1] : \left| \sum_{i = j}^{j+\lceil \log^5(T) \rceil} w_i \right| \ge 7\log^2(T)\right\}.
    \]
    Define
    \[
        E_{\mathrm{L}\ref{Ht_bound}}(x,K,\theta,W') := E_1^t \cap  \bigcap_{t_1 < t_2 \le t, t_2 - t_1 \ge \log^8(T)} E^{t_1,t_2}_{\mathrm{L}\ref{Ht_bound}}.
    \]
    First we will show that $\P(E_{\mathrm{L}\ref{Ht_bound}}(x,K,\theta,W')) = 1-o_T(1/T^{20})$. Consider any pair $t_2 > t_1$ such that $t_2 - t_1 \ge \log^8(T)$. Divide the interval $[t_1:t_2-1]$ into $\lfloor \frac{t_2 - t_1}{\lceil \log^5(T) \rceil + 1}\rfloor$ consecutive disjoint intervals of length $\lceil \log^5(T) \rceil + 1$. Consider one such interval $[s_1, s_2]$. Then the distribution of $\frac{1}{\sqrt{\lceil \log^5(T) \rceil+1}}\sum_{i=s_1}^{s_2} w_i$ converges in distribution to $N(0,\sigma_{\mathcal{D}}^2)$ as $T$ grows, where we recall $\sigma_{\mathcal{D}}^2$ is the variance of distribution $\mathcal{D}$. The rate of this convergence depends on $\mathcal{D}$. Therefore, for sufficiently large $T$, we have that 
    \begin{equation}
        \left| \P\left( \left| \frac{1}{\sqrt{\lceil \log^5(T) \rceil+1}}\sum_{i=s_1}^{s_2} w_i \right| \ge \sigma_\mathcal{D}/2 \right) - \P\left(|N(0, \sigma_\mathcal{D}^2)| \ge \sigma_{\mathcal{D}}/2\right)  \right| \le 0.1.
    \end{equation}
    This implies that
    \begin{equation}
        \P\left( \left| \frac{1}{\sqrt{\lceil \log^5(T) \rceil+1}}\sum_{i=s_1}^{s_2} w_i \right| \ge \sigma_\mathcal{D}/2 \right) \ge \P\left(|N(0, \sigma_\mathcal{D}^2)| \ge \sigma_{\mathcal{D}}/2\right) - 0.1 \ge 0.5.
    \end{equation}
    For sufficiently large $T$, we have that $\frac{\sqrt{\lceil \log^5(T) \rceil+1}\sigma_{\mathcal{D}}}{2} \ge 7\log^2(T)$, and therefore this implies that for sufficiently large $T$,    
    \begin{equation}\label{eq:interval_sat}
            \P\left( \left| \sum_{i=s_1}^{s_2} w_i \right| \ge 7\log^2(T) \right) \ge 0.5.
    \end{equation}
    Because the random variables in each disjoint interval are independent, we have that each interval independently satisfies Equation \eqref{eq:interval_sat} with probability at least $1/2$. Therefore,  for sufficiently large $T$, the probability that Equation \eqref{eq:interval_sat} fails to hold for all $\lfloor \frac{|t_2 - t_1|}{\lceil \log^5(T) \rceil+1} \rfloor \ge \log^2(T)$ intervals is at most $(1/2)^{\lfloor \frac{|t_2 - t_1|}{\lceil \log^5(T) \rceil + 1} \rfloor} \le 0.5^{\log^2(T)} = o_T(1/T^{22})$. Therefore, we have shown that 
    \[
        \P(E^{t_1,t_2}_{\mathrm{L}\ref{Ht_bound}}) \ge 1-o_T(1/T^{22}).
    \]
    Since there are less than $T^2$ pairs $(t_1,t_2)$ and $\P(E_1^t) \ge \P(E_1) = 1-o_T(1/T^{20})$ by Equation \eqref{eq:e1bound}, we have by a union bound that
    \[
        \P(E_{\mathrm{L}\ref{Ht_bound}}(x,K,\theta,W')) \ge  1 - o_T(T^2/T^{22}) - o_T(1/T^{20}) = 1-o_T(1/T^{20}).
    \]
    \begin{lemma}\label{lemma:all_pairs}
        Using the assumptions and notation of the proof of Lemma \ref{Ht_bound},   for all pairs $t_1,t_2$ such that $t_2 - t_1 \ge \log^8(T)$,  conditional on event $A_{t_2} \cap E_{\mathrm{L}\ref{Ht_bound}}(x,K,\theta,W')$,
    \begin{equation}\label{eq:H2_equation}
        \lambda(t_2) - \lambda(t_1) = \Omega_T\left(\frac{|t_2 - t_1|}{\log^8(T)}\right).
    \end{equation} 
    \end{lemma}
  
    By Lemma \ref{lemma:all_pairs}, conditional on $A_{t_2} \cap E_{\mathrm{L}\ref{Ht_bound}}(x,K,\theta,W')$, we that:
    \begin{align*}
    &\left(1 + O_T\left(\frac{1}{\log^{10}(T)}\right)\right)^{t_2 + 1 - t_1} O_T(\epsilon)^{\lambda(t_2) - \lambda(t_1)} \\
    &=\left(1 + O_T\left(\frac{1}{\log^{10}(T)}\right)\right)^{t_2 + 1 - t_1} O_T\left(1/\log(T)\right)^{\lambda(t_2) - \lambda(t_1)} && \text{ $\epsilon = O_T(1/\log(T))$} \\
    &\le \left(1+O_T\left(\frac{1}{\log^{10}(T)}\right)\right)^{t_2 + 1- t_1}\cdot O_T \left(\frac{1}{\log(T)}\right)^{ \Omega_T\left(\frac{|t_2 - t_1|}{\log^8(T)}\right)}  && \text{Equation \eqref{eq:H2_equation}}\\
    &\le \left(1+O_T\left(\frac{1}{\log^{2}(T)}\right)\right)^{(t_2 + 1 - t_1)/\log^8(T)}\cdot O_T \left(\frac{1}{\log(T)}\right)^{ \Omega_T\left(\frac{|t_2 - t_1|}{\log^8(T)}\right) } && \text{Lemma \ref{lemma:inequality_for_f_and_g}} \\
    &\le O_T\left(\left(\frac{1}{\log(T)}\left(1 + \frac{1}{\log^2(T)}\right)\right)\right)^{ \Omega_T\left(\frac{|t_2 - t_1|}{\log^8(T)}\right)} \\
    &\le \left(O_T\left(\frac{1}{\log(T)}\right)\right)^{\Omega_T\left(\frac{|t_2 - t_1|}{\log^8(T)}\right)} \\
    &\le \left(\left(O_T\left(\frac{1}{\log(T)}\right)\right)^{\Omega_T\left(\frac{1}{\log^8(T)}\right)}\right)^{(t_2 + 1 - t_1)} \\
        &\le \left(1 - \frac{2}{\log^{9}(T)}\right)^{t_2 + 1 - t_1}.  \numberthis \label{eq:zeta} 
\end{align*}

This is the desired result. In the last line we used that for sufficiently large $T$,
\begin{align*}
\left(O_T\left(\frac{1}{\log(T)}\right)\right)^{\Omega_T\left(\frac{1}{\log^8(T)}\right)} &\le \left(\frac{1}{2}\right)^{\Omega_T\left(\frac{1}{\log^8(T)}\right)}\\
&\le \left(\frac{1}{2}\right)^{\frac{4}{\log^9(T)}}\\
&\le \left(1 - \frac{2}{\log^{9}(T))}\right) && \text{Lemma \ref{lemma:inequality_for_f_and_g}} 
\end{align*}
Note that the first inequality above is a very loose bound, however it is what we need to prove the desired lemma.
    \end{proof}

 \subsection{Proof of Lemma \ref{lemma:all_pairs}}
    
    To show Equation \eqref{eq:H2_equation}, we will show that for all $t_2 \ge \lceil \log^8(T) \rceil$, conditional on event $A_{t_2} \cap E_{\mathrm{L}\ref{Ht_bound}}(x,k,\theta,W')$, for every $j \le t_2 - \lceil \log^8(T) \rceil + 1$ there exists some $i \in [j:j+\lceil \log^8(T) \rceil )$ such that $\mathcal{W}_i$ holds, where we recall that
    \[
        \mathcal{W}_i = \left\{\min(x_i,y_i) \ge \frac{D_U}{a-bK} \text{ or } \max(x_i,y_i) \le \frac{D_L}{a-bK} \right\}.
    \]
    This in turn implies Equation \eqref{eq:H2_equation} because we can divide  $[t_1+1: t_2]$ into $\Omega_T(\frac{|t_2 - t_1|}{\log^8(T)})$ disjoint intervals of the form $[j:j+\lceil \log^8(T) \rceil )$ where each interval contains an $i$ such that $\mathcal{W}_i$ holds.

 For the rest of the proof, we will prove by contradiction that  conditional on event $A_{t_2} \cap E_{\mathrm{L}\ref{Ht_bound}}(x,k,\theta,W')$, for every $j \le t_2 - \lceil \log^8(T) \rceil$ there exists some $i \in [j:j+\lceil \log^8(T) \rceil )$ such that $\mathcal{W}_i$ holds. Assume that this is not the case, and there exists $j$ such that there are no $i \in [j:j+\lceil \log^8(T) \rceil )$ such that $\mathcal{W}_i$ holds.

    By definition of $\mathcal{W}_i$, if $y_i \not\in \left[\frac{D_L}{a-bK}-d_i ,\frac{D_U}{a-bK} + d_i\right]$, then $\mathcal{W}_i$ must hold. Recall that conditional on event $A_{t_2}$, $d_i \le \frac{3}{\log^{10}(T)}$ for all $i \le t_2$. Therefore,  conditional on event $A_{t_2}$, if $y_i  \not\in \left[\frac{D_L}{a-bK}- \frac{3}{\log^{10}(T)} ,\frac{D_U}{a-bK} + \frac{3}{\log^{10}(T)} \right]$ then $\mathcal{W}_i$ must hold. Because we assumed that there are no $i \in [j:j+\lceil \log^8(T) \rceil )$ such that $\mathcal{W}_i$ holds, this implies that for all $i \in [j:j+\lceil \log^8(T) \rceil )$,
    \begin{equation}\label{eq:all_yi_by_assum}
        y_i  \in \left[\frac{D_L}{a-bK}- \frac{3}{\log^{10}(T)} ,\frac{D_U}{a-bK} + \frac{3}{\log^{10}(T)} \right].
    \end{equation}
    We also have that for sufficiently large $T$, 
    \begin{align*}
        \frac{\norm{D}_{\infty}}{a-bK} &\le  \frac{\norm{D}_{\infty}}{a^*-b^*K - O_T\left(\frac{1}{\log^{10}{T}}\right)} && \text{$\norm{\theta - \theta^*}_{\infty} \le 1/\log^{10}(T)$}
        \\
        &\le \frac{\norm{D}_{\infty}}{1 - O_T\left(\frac{1}{\log^{9}{T}}\right)} && \text{ $|1-(a^*-b^*K)| \le \frac{1}{\log^9(T)}$}  \\
        &\le 2\norm{D}_{\infty} \\
        &\le 2\log^2(T).  \numberthis \label{eq:random_walk}
    \end{align*}
    Therefore, if $|y_i| \ge \log^2(T) \ge 2\log^2(T) + \frac{3}{\log^{10}(T)}$ for sufficiently large $T$, then $\mathcal{W}_i$ must hold. For the rest of the proof, we will show that if Equation \eqref{eq:all_yi_by_assum} holds for all $i \in [j:j+\lceil \log^8(T) \rceil )$, then at least one such $i$ must satisfy $|y_i| \ge 3\log^2(T)$, which implies that $\mathcal{W}_i$ will hold which is a contradiction.

    \begin{lemma}\label{lemma:notsatisfy_yi}
        Using the notation and assumptions of Lemma \ref{lemma:all_pairs}, conditional on $A_{t_2} \cap E_{\mathrm{L}\ref{Ht_bound}}(x,k,\theta,W')$, if $y_i  \in \left[\frac{D_L}{a-bK}- \frac{3}{\log^{10}(T)} ,\frac{D_U}{a-bK} + \frac{3}{\log^{10}(T)} \right]$, then $y_{i+1} - y_i \in [w_i - O_T(1/\log^7(T)), w_i + O_T(1/\log^7(T))]$.
    \end{lemma}
    \begin{proof}
    The control at time $i$ is either $-Ky_i$, $\frac{D_U - ay_i}{b}$, or $\frac{D_L - ay_i}{b}$. If the control is $-Ky_i$, then under event $E_1^t$,
    \begin{align*}
        |y_{i+1} - y_{i} - w_{i}| &= |(a^*-b^*K)y_{i} - y_{i}| \\
        &= |y_{i}||1-(a^*-b^*K)| \\
        &= O_T\left( \frac{|y_{i}|}{\log^{9}(T)}\right) && \text{Assumed in Lemmas \ref{Ht_bound}, \ref{lemma:all_pairs}, and \ref{lemma:notsatisfy_yi}}\\
        &= O_T\left( \frac{1}{\log^{7}(T)}\right). && \text{Under event $E_1^t$ by Lemma \ref{bounded_pos_cont}}.
    \end{align*}
    The control at state $y_i$  is $\frac{D_U - ay_i}{b}$ only when $y_i \ge \frac{D_U}{a-bK}$. Because $y_i \le \frac{D_U}{a-bK} + \frac{3}{\log^{10}(T)}$, this implies that $\left| y_i - \frac{D_U}{a-bK}\right| \le \frac{3}{\log^{10}(T)}$, and because $(a-bK) \le 1$ this implies that $|D_U - (a-bK)y_i| = O_T(1/\log^{10}(T))$. Therefore, under event $E_1^t$, when the control at state $y_i$  is $\frac{D_U - ay_i}{b}$,
    \begin{align*}
        |y_{i+1} - y_{i} - w_{i}| &= |a^*y_i + b^*\frac{D_U - ay_i}{b} - y_i| \\
        &\le |(a^*-b^*K)y_{i} - y_{i}| + b^*\left|Ky_i + \frac{D_U-ay_i}{b}\right|\\
        &\le |(a^*-b^*K) - 1||y_i| + \frac{b^*}{b}\left|D_U - (a-bK)y_i \right|\\
        &\le O_T\left(\frac{|y_i|}{\log^{9}(T)}\right) + O_T\left(\frac{1}{\log^{10}(T)}\right)\\
        &\le O_T\left(\frac{1}{\log^7(T)}\right). && \text{Under event $E_1^t$ by Lemma \ref{bounded_pos_cont}}
    \end{align*}
    A symmetric result holds if the control at state $y_i$ is $\frac{D_L - ay_i}{b}$ (which happens when $y_i \le \frac{D_L}{a-bK}$). This exactly implies the desired result.
    
        \end{proof}

    Using Lemma \ref{lemma:notsatisfy_yi}, for $j\le  i_1 < i_2 \le j + \lceil \log^8(T) \rceil$ such that $i_2 - i_1  \le \lceil \log^5(T) \rceil $ and sufficiently large $T$, if $y_i  \in \left[\frac{D_L}{a-bK}- \frac{3}{\log^{10}(T)} ,\frac{D_U}{a-bK} + \frac{3}{\log^{10}(T)} \right]$ for all $i \in [j:j+\lceil \log^8(T) \rceil )$, then
    \begin{align*}
        \left| y_{i_2+1} - y_{i_1} \right| &\ge  \left| \sum_{j=i_1}^{i_2} w_j \right| - O_T\left(\frac{|i_2 - i_1|}{\log^7(T)}\right)\\
        &\ge \left|\sum_{j=i_1}^{i_2} w_j \right| - \frac{1}{\log(T)}. && \text{$i_2 - i_1 \le \lceil \log^5(T) \rceil$} \numberthis \label{eq:random_walk_bound}
    \end{align*}
     By construction, event $E_{\mathrm{L}\ref{Ht_bound}}(x,K,\theta,W')$ directly implies that for sufficiently large $T$, there exists some $i \in [j: j+\lceil \log^8(T) \rceil - \lceil \log^5(T) \rceil - 1]$ such that
    \begin{equation}\label{eq:sum_of_ws}
   \left|\sum_{j=i}^{i+\lceil \log^5(T) \rceil} w_j\right| \ge 7\log^2(T) \ge 2\cdot 3\log^2(T) + \frac{1}{\log(T)}.
    \end{equation}
    Combining this with Equation \eqref{eq:random_walk_bound} for $i_1 = i$ and $i_2 = i+\lceil \log^{5}(T) \rceil $, conditional on $A_{t_2} \cap E_{\mathrm{L}\ref{Ht_bound}}(x,k,\theta,W')$, 
    \[
        \left| y_{i_2+1} - y_{i_1} \right| \ge 6\log^2(T).
    \]
    This implies that either $|y_i|$ or $|y_{i+\lceil \log^5(T) \rceil+1}|$ is greater than $3\log^2(T)$. However, as argued above this implies that $\mathcal{W}_i$ or $\mathcal{W}_{i+\lceil \log^5(T) \rceil+1}$ holds, which is a contradiction. This completes the proof by contradiction.

\section{Proof of Lemma \ref{parameterization_assum2}}\label{app:paramterization_assum2}

\begin{proof}

Let $\epsilon_{\mathrm{L}\ref{parameterization_assum2}} = \frac{1}{\log^{46}(T)}$. We will combine the following two results.

\begin{lemma}\label{ptimal_off_theta}
Under Assumptions \ref{assum:constraints}--\ref{assum:initial}, for any $\theta$ such that $\norm{\theta- \theta^*}_{\infty} = \epsilon \le \frac{1}{\log^{46}(T)}$, the following holds for the class of truncated linear controllers for $t \le T$:
\[
   \bar{J}(\theta,C^\theta_{K_{\mathrm{opt}}(\theta, t)}, t) -  \bar{J}(\theta^*,C^{\theta^*}_{K_{\mathrm{opt}}(\theta^*, t)},t) = \tilde{O}_T(\epsilon).
\]
\end{lemma}
The proof of Lemma \ref{ptimal_off_theta} can be found in Appendix \ref{sec:proof_of_ptimal_off_theta}.

\begin{lemma}\label{diff_in_theta}
    Under Assumptions \ref{assum:constraints}--\ref{assum:initial}, for any $\norm{\theta - \theta^*}_{\infty} = \epsilon  \le \frac{1}{\log^{46}(T)}$, $t \le T$, and $K \in [\frac{a-1}{b}, \frac{a}{b}]$,
    \begin{equation}\label{eq:diff_in_theta2}
    |\bar{J}(\theta^*,C^{\theta}_{K},t) - \bar{J}(\theta,C^{\theta}_K, t)| =  \tilde{O}_T\left(\epsilon + \frac{1}{T^2}\right).
    \end{equation}
\end{lemma}
The proof of Lemma \ref{diff_in_theta} can be found in Appendix \ref{sec:proof_of_diff_in_theta}.

Putting together Lemma \ref{ptimal_off_theta} and Lemma \ref{diff_in_theta} with $K = K_{\mathrm{opt}}(\theta, t)$, we have the desired result that
\[
    \bar{J}(\theta^*,C^\theta_{K_{\mathrm{opt}}(\theta, t)},t)- \bar{J}(\theta^*,C^{\theta^*}_{K_{\mathrm{opt}}(\theta^*, t)},t) = \tilde{O}_T\left(\epsilon + \frac{1}{T^2}\right).
\]
\end{proof}

\subsection{Proof of Lemma \ref{ptimal_off_theta}}\label{sec:proof_of_ptimal_off_theta}
\begin{proof}
First, we will prove some results about $a^*,b^*,K_{\mathrm{opt}}(\theta^*, t)$. Because $b, b^* \ge \underline{b}$ and $\norm{\theta - \theta^*}_{\infty} = \epsilon  \le \frac{1}{\log^{46}(T)} < b/2$  for large enough $T$, we have that
\begin{equation}\label{eq:par11a}
\left|\left(\frac{a^*}{b^*}\right)^2 - \left(\frac{a}{b}\right)^2\right| = \left| \frac{(a^*)^2b^2 - (b^*)^2a^2}{b^2(b^*)^2}\right| \le \frac{\epsilon^2 b^2 +2\epsilon ab^2 + 2\epsilon ba^2 + \epsilon^2 a^2}{b^2(b-\epsilon)^2} = O_T(\epsilon).
\end{equation}
\begin{equation}\label{eq:par11b}
    \left|\frac{a}{b} - \frac{a^*}{b^*}\right| = \left| \frac{a^*b - b^*a}{bb^*}\right| \le \left| \frac{\epsilon b + \epsilon a}{b(b-\epsilon)}\right| = O_T(\epsilon).
\end{equation}
Let $K'$ be the solution to $a^* - b^*K_{\mathrm{opt}}(\theta^*, t) = a - bK'$. Then
\[
    K' = \frac{(a - a^*) + b^*K_{\mathrm{opt}}(\theta^*, t)}{b} = K_{\mathrm{opt}}(\theta^*, t) + \frac{(b^* - b)K_{\mathrm{opt}}(\theta^*, t)}{b} + \frac{a-a^*}{b}.
\]
Since $K_{\mathrm{opt}}(\theta^*, t) \le \frac{a^*}{b^*}$ by definition, we have the following two equations:
\begin{equation}\label{eq:par11c}
    |K' - K_{\mathrm{opt}}(\theta^*, t)| = \left| \frac{(b^* - b)K_{\mathrm{opt}}(\theta^*, t)}{b} + \frac{a-a^*}{b} \right|\le  \left(\frac{a^*}{bb^*}+ \frac{1}{b}\right)\epsilon = O_T(\epsilon).
\end{equation}
\begin{equation}\label{eq:par11d}
    |(K')^2 - (K_{\mathrm{opt}}(\theta^*, t))^2| \le |K' - K_{\mathrm{opt}}(\theta^*, t)| \cdot |K' + K_{\mathrm{opt}}(\theta^*, t)| = O_T(\epsilon).
\end{equation}
By the choice of $K'$, using the controller $C_{K_{\mathrm{opt}}(\theta^*, t)}^{\theta^*}$ under dynamics $\theta^*$ results in the exact same sequence of states as using the controller $C_{K'}^{\theta}$ under dynamics $\theta$. This is because $a-bK' = a^*-b^*K_{\mathrm{opt}}(\theta^*, t)$, which by construction of truncated linear controllers implies that  $ax +bC_{K'}^{\theta}(x) = a^* + b^*C_{K_{\mathrm{opt}}(\theta^*, t)}^{\theta^*}$ for all $x$. The controls will however be different, and we will now bound that difference in controls. 

Define $x_0,x_1,...,x_{t}$ as the sequence of states when using controller $C_{K_{\mathrm{opt}}(\theta^*, t)}^{\theta^*}$ under dynamics $\theta^*$ starting at state $x_0 = 0$. Then we have the following result.
{\fontsize{10}{10}
    \begin{equation}
    \left|rC^{\theta^*}_{K_{\mathrm{opt}}(\theta^*, t)}(x_i)^2 - rC^\theta_{K'}(x_i)^2\right| =
        \begin{cases}
            \left|rx_i^2\left((K_{\mathrm{opt}}(\theta^*, t))^2 - (K')^2\right)\right| & \text{if } x_i \in [\frac{D_L}{a^*-b^*K_{\mathrm{opt}}(\theta^*, t)}, \frac{D_U}{a^*-b^*K_{\mathrm{opt}}(\theta^*, t)}]\\
            \left|r\left(\frac{D_U -a^*x_i}{b^*}\right)^2 - r\left(\frac{D_U -ax_i}{b}\right)^2\right|
         & \text{if $x_i > \frac{D_U}{a^*-b^*K_{\mathrm{opt}}(\theta^*, t)} $} \\
            \left|r\left(\frac{D_L -a^*x_i}{b^*}\right)^2 - r\left(\frac{D_L -ax_i}{b}\right)^2\right|
         & \text{if $x_i <\frac{D_L}{a^*-b^*K_{\mathrm{opt}}(\theta^*, t)} $} \\
        \end{cases}
    \end{equation}
    }
By Equation \eqref{eq:par11d}, this implies the following.
{\fontsize{10}{10}
    \begin{align*}
   &\left|rC^{\theta^*}_{K_{\mathrm{opt}}(\theta^*, t)}(x_i)^2 - rC^\theta_{K'}(x_i)^2\right| \\
   &\le
        \begin{cases}
            O_T(x_i^2\epsilon) & \text{if } x_i \in [\frac{D_L}{a^*-b^*K_{\mathrm{opt}}(\theta^*, t)}, \frac{D_U}{a^*-b^*K_{\mathrm{opt}}(\theta^*, t)}]\\
            rD_U^2\left|\left(\frac{1}{b^*}\right)^2 - \left(\frac{1}{b}\right)^2\right| + 2D_Ur|x_i|\left|\frac{a}{b} - \frac{a^*}{b^*}\right| + rx_i^2 \left|\left(\frac{a^*}{b^*}\right)^2 - \left(\frac{a}{b}\right)^2\right|
         & \text{if $x_i > \frac{D_U}{a^*-b^*K_{\mathrm{opt}}(\theta^*, t)} $} \\
            rD_L^2\left|\left(\frac{1}{b^*}\right)^2 - \left(\frac{1}{b}\right)^2\right| + 2|D_L|r|x_i|\left|\frac{a}{b} - \frac{a^*}{b^*}\right| + rx_i^2 \left|\left(\frac{a^*}{b^*}\right)^2 - \left(\frac{a}{b}\right)^2\right|
         & \text{if $x_i <\frac{D_L}{a^*-b^*K_{\mathrm{opt}}(\theta^*, t)} $} \\
        \end{cases}
    \end{align*}
    }
By Equations \eqref{eq:par11a} and \eqref{eq:par11b}, we get the following result.
    \begin{equation}
    \left|rC^{\theta^*}_{K_{\mathrm{opt}}(\theta^*, t)}(x_i)^2 - rC^\theta_{K'}(x_i)^2\right|  \le
        \begin{cases}
            O_T(x_i^2)\epsilon & \text{if } x_i \in [\frac{D_L}{a^*-b^*K_{\mathrm{opt}}(\theta^*, t)}, \frac{D_U}{a^*-b^*K_{\mathrm{opt}}(\theta^*, t)}]\\
            O_T(D_U^2\epsilon + D_U|x_i|\epsilon) + O_T(x_i^2\epsilon)
         & \text{if $x_i > \frac{D_U}{a^*-b^*K_{\mathrm{opt}}(\theta^*, t)} $} \\
            O_T(D_L^2\epsilon + |D_L||x_i|\epsilon) + O_T(x_i^2 \epsilon)
         & \text{if $x_i <\frac{D_L}{a^*-b^*K_{\mathrm{opt}}(\theta^*, t)} $} \\
        \end{cases}
    \end{equation}
    Using that $\norm{D}_{\infty} \le \log^2(T)$, in all three cases we have that
    \begin{equation}\label{eq:par11f}
    \left|rC^{\theta^*}_{K_{\mathrm{opt}}(\theta^*, t)}(x_i)^2 - rC^\theta_{K'}(x_i)^2\right|  = \tilde{O}_T\left(1 + |x_i| + |x_i|^2\right)\epsilon.
    \end{equation}
    The last fact we need is to note that $x_i$ is a sequence of states for the controller $C_{K_{\mathrm{opt}}(\theta^*, t)}^{\theta^*}$ under dynamics $\theta^*$, which by construction will always satisfy that $D_L \le a^*x_i + b^*C_{K_{\mathrm{opt}}(\theta^*, t)}^{\theta^*}(x^*) \le D_U$. Therefore, since $\E[|w_{i-1}|]$ and $\E[w_{i-1}^2]$ are constants relative to $T$ that depend on $\mathcal{D}$, for all $i$,
    \[
        \E[|x_i|] \le \norm{D}_{\infty} + \E[|w_{i-1}|] = O_T(\log^2(T)).
    \]
    \[
        \E[|x_i|^2] \le\norm{D}_{\infty}^2 + \E[w_{i-1}^2] + 2\norm{D}_{\infty}\E[|w_{i-1}|] = O_T(\log^4(T)).
    \]
Therefore, we can upper bound the difference in cost as follows:
\begin{align*}
    \bar{J}(\theta,C^\theta_{K'}, t) -  \bar{J}(\theta^*,C^{\theta^*}_{K_{\mathrm{opt}}(\theta^*, t)},t) &\le \E\left[ \frac{1}{t}\sum_{i=0}^{t-1} \left| rC^{\theta^*}_{K_{\mathrm{opt}}(\theta^*, t)}(x_i)^2 - rC^\theta_{K'}(x_i)^2\right| \right] \\
    &\le \frac{1}{t}\sum_{i=0}^{t-1} \tilde{O}_T\left(1 + \E[|x_i|] + \E[|x_i|^2]\right)\epsilon && \text{Equation \eqref{eq:par11f}} \\
    &\le \frac{1}{t}\sum_{i=0}^{t-1} \tilde{O}_T\left(\log^2(T) + \log^4(T)\right)\epsilon \\
    &=  \tilde{O}_T\left(\epsilon\right). 
\end{align*}
Finally, by definition of $K_{\mathrm{opt}}$ we know that
\[
    \bar{J}(\theta,C^{\theta}_{K_{\mathrm{opt}}(\theta, t)},t) \le  \bar{J}(\theta,C^\theta_{K'}, t),
\]
therefore we can conclude that
\[
    \bar{J}(\theta,C^\theta_{K_{\mathrm{opt}}(\theta, t)}, t) -  \bar{J}(\theta^*,C^{\theta^*}_{K_{\mathrm{opt}}(\theta^*, t)},t) = \tilde{O}_T(\epsilon).
\]
\end{proof}

\subsection{Proof of Lemma \ref{diff_in_theta}}\label{sec:proof_of_diff_in_theta}

\begin{proof}
For a set of time varying dynamics $\{\theta_j\}_{j=0}^{t-1}$ where $\theta_j \in \Theta$ for all $j$, we define the expected total cost for varying dynamics as 
\[
\bar{J}(\{\theta_j\}_{j=0}^{t-1},C^{\theta}_{K},t) := qx_t^2 + \sum_{j=0}^{t-1} qx_j^2 + rC_K^\theta(x_{j-1})^2,
\]
where $x_0 = 0$ and $x_j = a_{j-1}x_{j-1} + b_{j-1}C_K^\theta(x_{j-1}) + w_{j-1}$. In other words, this is the total cost if the dynamics at time $j < t$ are $\theta_j$.

For $i \in [0:t]$, let $\{\theta^{i}_j\}_{j=0}^{t-1}$ be a time varying dynamics with $\theta_j^{i} = \theta$ for all $j < i$ and $\theta_j^{i} = \theta^*$ for $j \ge i$. We will now compare the costs under dynamics $\{\theta_j^{i}\}_{j=0}^{t-1}$ versus under $\{\theta_j^{{i+1}}\}_{j=0}^{t-1}$.  Let $x_0,x_1,...x_{t}$ be the states when using controller $C_{K}^\theta$ under time-varying dynamics $\{\theta_j^{i}\}_{j=0}^{t-1}$ and $x_0^*,...x^*_t$ be the states when using controller $C_K^\theta$ under time-varying dynamics $\{\theta_j^{{i+1}}\}_{j=0}^{t-1}$ (both starting at $x_0 = x_0^* = 0$). Up until time $i$, the dynamics of these two trajectories are the same (both equal to $\theta$), and therefore the states and controls of the two trajectories are equivalent up until time $i$. Because $C_K^\theta$ is safe with respect to dynamics $\theta$, $|x_i^*| = |x_i| \le \norm{D}_{\infty} + |w_{i-1}|$. Because $\norm{D}_{\infty} \le \log^2(T)$, this implies that
\begin{equation}\label{eq:xi_ot1}
    \E[|x_i^*|] = \E[|x_i|] = \tilde{O}_T(1).
\end{equation}
Also note that by construction of the truncated linear controller, $|C_K^\theta(x_i)| \le K|x_i| + \frac{\norm{D}_{\infty} + a|x_i|}{b}$. Therefore, we have that

\begin{equation}\label{eq:xbound}
    |x_{i+1} - x_{i+1}^*| = |ax_i + bC_{K}^\theta(x_i) - a^*x_i - b^*C_{K}^\theta(x_i)| \le \epsilon |x_i| + \epsilon |C_K^\theta(x_i)| \le \epsilon\left(|x_i| + K|x_i| + \frac{\norm{D}_{\infty} + a|x_i|}{b}\right).
\end{equation}
Combining Equations \eqref{eq:xi_ot1} and \eqref{eq:xbound} gives that
\begin{equation}\label{eq:diff_in_xs}
    \E[|x_{i+1} - x_{i+1}^*|] = \tilde{O}_T(\epsilon).
\end{equation}
Consider $x_{i+1}$. Define the event $F = \{|x_{i+1}| < \log^3(T)\}$. As argued above, $|x_i| \le \norm{D}_{\infty} + |w_{i-1}| \le 2\log^2(T)$ under event $E_1$. Furthermore, the control $C_K^\theta(x_i)$ is safe with respect to dynamics $\theta^i_i$ and $\norm{\theta_i^i - \theta^*}_{\infty} = \norm{\theta - \theta^*}_{\infty} \le 1/\log^{46}(T) \le  1/\log(T)$ for sufficiently large $T$. Therefore, we can apply Lemma \ref{bounded_approx} for one step to get that for sufficiently large $T$, $|x_{i+1}| \le 4\log^2(T)$ under event $E_1$. Therefore, for sufficiently large $T$, $\P(F) \ge \P(E_1) = 1-o_T(1/T^{11})$. By Lemma \ref{lemma:subgaussian_tail} (using the same logic as in Equation \eqrefgen{eq:lemma13app}), this implies that
\begin{align*}
     \P(|x_{i+1}| \ge \log^3(T))\E[|x_{i+1}|^2 \mid |x_{i+1}| \ge \log^3(T)]  = o_T(1/T^{10}).
\end{align*}
The same logic holds for $x^*_{i+1}$. We showed above that $\P(|x_{i+1}| \le 4\log^2(T)) = 1-o_T(1/T^{11})$ (and the same equation holds for $x^*_{i+1}$). Therefore, we can apply Lemma \ref{offbyepsilon_exp} to get that
\begin{align*}
    &|t \cdot J ^*(\{\theta_j^{{i}}\}_{j=0}^{t-1},C^{\theta}_K,t, 0) - t \cdot J ^*(\{\theta_j^{{i+1}}\}_{j=0}^{t-1},C^{\theta}_{K}, t, 0)| \\
    &= \E \left[ |(t-i)\bar{J}(\theta^*,C^{\theta}_K,t-i,x_{i+1} ) - (t-i)\bar{J}(\theta^*,C^{\theta}_K, t-i, x_{i+1}^*)|\right] \\
    &=  \tilde{O}_T \left( \E\left[ \left|x_{i+1}- x_{i+1}^*\right|\right] + \epsilon + \frac{1}{T^2}\right) && \text{ Lemma \ref{offbyepsilon_exp}} \\
    &=  \tilde{O}_T\left(\epsilon + \frac{1}{T^2}\right). && \text{Equation \eqref{eq:diff_in_xs}}  \numberthis \label{eq:ind_step}
\end{align*}
Now, we conclude by noting that
\begin{align*}
    |t \cdot J ^*(\theta^*,C^{\theta}_{K},t) - t \cdot J ^*(\theta,C^{\theta}_{K}, t)| &= \left|\sum_{i=0}^{t} t \cdot J ^*(\{\theta_j^{i+1}\}_{j=0}^{t-1},C^{\theta}_K,t, 0) - t \cdot J ^*(\{\theta_j^{i}\}_{j=0}^{t-1},C^{\theta}_{K}, t, 0)\right| \\
    &= \tilde{O}_T\left(t\left(\epsilon + \frac{1}{T^2}\right)\right),
\end{align*}
and dividing both sides of the equation by $t$ gives the desired result.
\end{proof}

\newpage

\section{Proof of Theorem \ref{trunc_thm}}\label{sec:trunc_lin_sqrtt}

For the proof of Theorem \ref{trunc_thm}, recall the following notation (which was also defined in the proof sketch of Theorem \ref{trunc_thm}). Define $\mathcal{C}^{\mathrm{unc}} = \{C_K^{\mathrm{unc}}\}_{K \in \mathbb{R}}$ as the class of untruncated linear controllers, where $C_K^{\mathrm{unc}}(x) = -Kx$. For any controller $C$ and dynamics $\theta$, define $\bar{J}(\theta, C) =  \lim_{T \longrightarrow \infty} \bar{J}(\theta, C, T)$. Define  $K_{\mathrm{opt}}(\theta) = \arg\sup_{K} \bar{J}(\theta, C_K^\theta)$ and $F_{\mathrm{opt}}(\theta) = \arg\sup_{K} \bar{J}(\theta, C_K^{\mathrm{unc}})$. 

By Lemmas \ref{parameterization_assum2} and \ref{parameterization_assum3}, the class of truncated linear controllers satisfies the assumptions of Theorem \refgen{sufficiently_large_error}. If $\mathcal{D}$ has infinite support and $\norm{D}_{\infty} = O_T(1)$, then Assumption \refgen{assum_sufficiently_large_error} is satisfied. Furthermore, for noise distribution with infinite support, Algorithm \ref{alg:cap3} will choose the exact same controls as Algorithm \refgen{alg:cap_large}. Therefore, under Assumptions \ref{assum:constraints}--\ref{assum:initial}, if $\mathcal{D}$ has infinite support, then Algorithm \ref{alg:cap3} with the baseline class of truncated linear controllers has regret of $\tilde{O}_T(\sqrt{T})$ by Theorem \refgen{sufficiently_large_error}. Therefore, Theorem \refgen{sufficiently_large_error} directly proves Theorem \ref{trunc_thm} in the case when $\mathcal{D}$ has infinite support. For the rest of this proof, we will focus on proving Theorem \ref{trunc_thm} when $\mathcal{D}$ has bounded support, therefore making the following assumption.

\begin{assum}\label{assum_thm4}
    The distribution $\mathcal{D}$ has bounded support, i.e. there exists $\bar{w} > 0$ such that $\P_{w \sim \mathcal{D}}(|w| \le \bar{w}) = 1$.
\end{assum}
For the rest of the proof of Theorem \ref{trunc_thm}, we will also assume WLOG that $D_U \le |D_L|$.

\begin{definition}\label{def:Kdu_theta}
    Define $K_{D_U}^\theta$ as the value that satisfies the equation
        \[
            \frac{D_U}{a-bK_{D_U}^\theta} - D_U = \bar{w}.
        \]
\end{definition}

For the rest of Appendix \ref{sec:trunc_lin_sqrtt}, let $C^{\mathrm{alg}}$ be the controller of Algorithm \ref{alg:cap3} and $\mathcal{C}^\theta_{\mathrm{tr}}$ be the class of truncated linear controllers for dynamics $\theta$ as in Equation \eqref{eq:def_of_trunc_linear}. 

Let $s_e = \log_2(\sqrt{T}) - 1$, and let
\begin{equation}\label{eq:E0_trunc}
    E_0 := \left\{\forall s \in [0:s_e] : \norm{\theta^* - \hat{\theta}_s^{\mathrm{pre}}}_{\infty} \le \epsilon_s \right\}.
\end{equation} 
The following lemma (Lemma \refgen{v_to_use}) bounds the uncertainty in $\theta^*$ from regularized least squares estimation.
\begin{lemma}[Lemma \refgen{v_to_use}, Theorem 1 in \cite{abbasi2011regret}]\label{v_to_use}
    Suppose $x_t$ and $u_t$ are respectively the state and control at time $t$ when using an arbitrary controller $C$ starting at state $x_0 = 0$. Define $z_t = (x_t,u_t)$ and let $\lambda > 0$. Let $Z_t \in \mathbb{R}^{t \times 2}$ where the $i$th row is $z_{i-1}$, let $X_t \in \mathbb{R}^{t \times 1}$ where the $i$th element is $x_{i}$, and let $I \in \mathbb{R}^{2 \times 2}$ be the identity matrix.  Then under Assumptions \ref{assum:constraints}--\ref{assum:initial}, with probability $1-o_T\left(\frac{1}{T^2}\right)$ the following holds for all $1 \le t \le T-1$ and for any $S \subseteq [0:t-1]$:
    \begin{equation}\label{eq:v_to_use}
        \norm{\theta^* - (Z_{t}^{\top}Z_{t}+\lambda I)^{-1}Z_{t}^{\top}X_{t}}_{\infty} \le \sqrt{\frac{\max((V_t^S)_{11}, (V_t^S)_{22})}{\det(V_t^S)}}B_t,
    \end{equation}
    where $V_t^S = \lambda I + \sum_{s=0}^{t-1}z_sz_s^\top 1_{s \in S}$, $B_t = \alpha\sqrt{\log\left(\det\left(
    V_t^{[0:t-1]}
    \right)\right) + \log(\lambda^2) + 2\log(T^2)} + \sqrt{\lambda}(\bar{a}^2 + \bar{b}^2)$, and $\alpha$ is from the subgaussian assumption on the noise distribution $\mathcal{D}$, which implies that there exists an $\alpha$ such that $\E_{w \sim \mathcal{D}}[\exp(\gamma w)] \le \exp(\gamma^2\alpha^2/2)$ for any $\gamma \in \mathbb{R}$.
\end{lemma}
By Lemma  \ref{v_to_use} we have that with probability $1-o_T(1/T^2)$, for all $s$, $\norm{\theta^* - \hat{\theta}_s^{\mathrm{pre}}}_{\infty} \le \epsilon_s$.Therefore, 
\[
    \P(E_0) = 1-o_T(1/T^2).
\]
 By construction we also have that $\norm{\hat{\theta}_s - \hat{\theta}_s^{\mathrm{pre}}}_{\infty} \le \epsilon_s$. This implies by the triangle inequality that under event $E_0$, $\norm{\hat{\theta}_s - \theta^*}_{\infty} \le 2\epsilon_s$. 
 
We also have the following uncertainty result that is equivalent to Lemma \refgen{initial_uncertainty}:
\begin{lemma}\label{initial_uncertainty_trunc}
    Under Assumptions \ref{assum:constraints}--\ref{assum:initial}, there exists a $c_{\mathrm{L}\ref{initial_uncertainty_trunc}} = \tilde{O}_T(1)$ such that with probability $1-o_T(1/T^2)$
    \[
     \max_{s \in [0:s_e]} \epsilon_s \le c_{\mathrm{L}\ref{initial_uncertainty_trunc}} T^{-1/4} = \tilde{O}_T(T^{-1/4}).
    \]
\end{lemma}
The proof of Lemma \refgen{initial_uncertainty} relies only on the first $1/\nu_T^2$ steps and is written agnostic to the choice of $\nu_T$, and therefore the result of Lemma \ref{initial_uncertainty_trunc} follows directly from that proof. Note that we explicitly named the constant in Lemma \ref{initial_uncertainty_trunc} as we will use this constant later in the proof. For the rest of this section, define
\begin{equation}\label{eq:E2_trunc}
    E_2 := E_0 \bigcap \left\{ \max_{s \in [0:s_e]} \epsilon_s \le c_{\mathrm{L}\ref{initial_uncertainty_trunc}} T^{-1/4} = \tilde{O}_T(T^{-1/4}) \right\}.
\end{equation}
Lemma \ref{initial_uncertainty_trunc} implies that we have
\[
    \P(E_2) = 1-o_T(1/T^2).
\]
Define
\[
    E_2^0 := \{\epsilon_0 \le c_{\mathrm{L}\ref{initial_uncertainty_trunc}} T^{-1/4} \} \cap \{ \norm{\theta^* - \hat{\theta}_0^{\mathrm{pre}}}_{\infty} \le \epsilon_0\} \subseteq E_2.
\]
Recall $\hat{\theta}_{\mathrm{wu}}$, which is defined in Line \ref{line:thetawu} of Algorithm \ref{alg:cap3}. Because $\hat{\theta}_{\mathrm{wu}} = \hat{\theta}_0^{\mathrm{pre}}$, by the same logic as above, under $E_2^0$ we have that $\norm{\theta^* - \hat{\theta}_{\mathrm{wu}}}_{\infty} \le 2\epsilon_0 \le 2c_{\mathrm{L}\ref{initial_uncertainty_trunc}} T^{-1/4}$.

Define $E_1$ as
\begin{equation}\label{eq:E1}
    E_1 = \left\{\forall t < T : |w_t| \le \log^2(T) \right\}.
\end{equation}
and $E_{\mathrm{safe}}$ as the following, where $x'_t$ and $u'_t$ are the states and controls respectively of the algorithm:
\begin{equation}\label{eq:safeE}
  E_{\mathrm{safe}} = \left\{ \forall t < T: D_{\mathrm{L}} \le a^*x'_t + b^*u'_t \le D_{\mathrm{U}}\right\},
\end{equation}
Finally, we define the event 
\[
E = E_1 \cap E_2 \cap E_{\mathrm{safe}}.
\] 
By a union bound we have that $\P(E) = 1-o_T(1/T^2)$. Using this new notation and Lemma \ref{initial_uncertainty_trunc}, we can proceed to the main proof. 

The desired safety of $C^{\mathrm{alg}}$ follows from the following lemma:
\begin{lemma}\label{safety_append}
    Under Assumptions \ref{assum:constraints}--\ref{assum:initial} , Algorithm \ref{alg:cap3} is safe for $T$ steps for dynamics $\theta^*$ with probability $1-o_T(1/T^2)$.
\end{lemma}
The proof of Lemma \ref{safety_append} follows exactly as in the proof of Lemma \refgen{safety_append} except using Lemma \ref{initial_uncertainty_trunc} and the above definitions of $E_0,E_1$ and $E_2$ with respect to Algorithm \ref{alg:cap3}. The following result is equivalent to Lemma \refgen{lemma:L_less_than_U} and is proven in the exact same way using that $T^{-1/4} = o_T(1/\log(T))$.
\begin{lemma}\label{lemma:L_less_than_U}
    Under Assumptions \ref{assum:constraints}--\ref{assum:initial}, conditional on $E_1 \cap E_2$ and for sufficiently large $T$, if $u_{T_0-1}$ is safe for dynamics $\theta^*$, then for all $t \in [T_0, T]$,
    \[
        u_t^{\mathrm{safeL}} \le u_t^{\mathrm{safeU}}.
    \]
\end{lemma}The rest of the proof of Lemma \ref{safety_append} follows directly using Lemma \ref{lemma:L_less_than_U}.

The rest of this section will focus on proving that the regret of Algorithm \ref{alg:cap3} is $\tilde{O}_T(\sqrt{T})$ with probability $1-o_T(1/T)$. 

Let $C_{\mathrm{switch}} = \frac{c_{\mathrm{E}\ref{eq:Fhat_approx}}D_{\mathrm{U}}}{c_{\mathrm{L}\ref{j_bounded_from_0}}^2} = \tilde{O}_T(1)$ where $c_{\mathrm{E}\ref{eq:Fhat_approx}} = \tilde{O}_T(1)$ and is defined in Equation \eqref{eq:Fhat_approx} and $c_{\mathrm{L}\ref{j_bounded_from_0}} = \Omega(1)$ defined in Lemma \ref{j_bounded_from_0}; Equation~\eqref{eq:Fhat_approx} and Lemma~\ref{j_bounded_from_0} will both appear in Appendix~\ref{sec:proof_of_cswitch_prop}. Note that $C_{\mathrm{switch}}$ is used in Line \ref{line:split} of Algorithm \ref{alg:cap3}. Define the event $E_{\mathrm{E}\ref{eq:first_case}}$ as
\begin{equation}\label{eq:first_case}
    E_{\mathrm{E}\ref{eq:first_case}} := \left\{
        \bar{w} +D_U - \frac{D_U}{\hat{a}_{\mathrm{wu}}-\hat{b}_{\mathrm{wu}}F_{\mathrm{opt}}(\hat{\theta}_{\mathrm{wu}})} \le C_{\mathrm{switch}}T^{-1/4}\right\}.
\end{equation}

We will study the regret of Algorithm \ref{alg:cap3} separately under event $E_{\mathrm{E}\ref{eq:first_case}}$ and under event $\neg E_{\mathrm{E}\ref{eq:first_case}}$. Informally, if $E_{\mathrm{E}\ref{eq:first_case}}$ holds then the optimal linear controller is close to being safe for dynamics $\theta^*$. If $\neg E_{\mathrm{E}\ref{eq:first_case}}$, then the magnitude of the noise is large relative to the constraints, and therefore an argument similar to that of Theorem \refgen{sufficiently_large_error} will bound the regret.

\begin{proposition}\label{prop_Kcase}
    Under Assumptions \ref{assum:constraints}--\ref{assum:initial} and \ref{assum_thm4}, there exists an event $E_{\mathrm{P}\ref{prop_Kcase}}$ such that $E_{\mathrm{P}\ref{prop_Kcase}} \subseteq \neg E_{\mathrm{E}\ref{eq:first_case}}$, such that $\P(E_{\mathrm{P}\ref{prop_Kcase}}) \ge \P(\neg E_{\mathrm{E}\ref{eq:first_case}}) - o_T(1/T)$, and such that conditional on event $E_{\mathrm{P}\ref{prop_Kcase}}$, Algorithm \ref{alg:cap3} has $\tilde{O}_T(\sqrt{T})$ regret.
\end{proposition}
    The proof of Proposition \ref{prop_Kcase} can be found in Appendix \ref{sec:proof_of_prop_Kcase}.

\begin{proposition}\label{prop_Fcase}
    Under Assumptions \ref{assum:constraints}--\ref{assum:initial} and \ref{assum_thm4}, there exists an event $E_{\mathrm{P}\ref{prop_Fcase}}$ such that $E_{\mathrm{P}\ref{prop_Fcase}} \subseteq E_{\mathrm{E}\ref{eq:first_case}}$, such that $\P(E_{\mathrm{P}\ref{prop_Fcase}}) \ge \P(E_{\mathrm{E}\ref{eq:first_case}}) - o_T(1/T)$, and such that conditional on event $E_{\mathrm{P}\ref{prop_Fcase}}$, Algorithm \ref{alg:cap3} has $\tilde{O}_T(\sqrt{T})$ regret.
\end{proposition}
    The proof of Proposition \ref{prop_Fcase} can be found in Appendix \ref{sec:proof_of_prop_Fcase}.

Combining these two propositions gives that the regret of Algorithm \ref{alg:cap3} is $\tilde{O}_T(\sqrt{T})$ conditional on $E_{\mathrm{P}\ref{prop_Kcase}} \cup E_{\mathrm{P}\ref{prop_Fcase}}$. Because $E_{\mathrm{P}\ref{prop_Kcase}} \cap E_{\mathrm{P}\ref{prop_Fcase}} = \emptyset$ by construction, we have that 
\[
\P(E_{\mathrm{P}\ref{prop_Kcase}} \cup E_{\mathrm{P}\ref{prop_Fcase}}) = \P(E_{\mathrm{P}\ref{prop_Kcase}}) + \P(E_{\mathrm{P}\ref{prop_Fcase}}) \ge \P(E_{\neg \ref{eq:first_case}}) - o_T(1/T) + \P(E_{\mathrm{E}\ref{eq:first_case}}) - o_T(1/T) = 1 - o_T(1/T).
\]
Therefore the desired result holds with unconditional probability $1-o_T(1/T)$, completing the proof of Theorem \ref{trunc_thm}.

\subsection{Proof of Proposition \ref{prop_Kcase}}\label{sec:proof_of_prop_Kcase}

\begin{proof}
We can decompose the regret in the following manner. As in \cite{schiffer2024stronger}, for any $\left(K, \{K_s\}_{0 \le s \le s_e}\right)$ where $K, K_s \in (\frac{a-1}{b}, \frac{a}{b})$, define $x_0^{\left(K, \{K_s\}_{0 \le s \le s_e}\right)}, x_1^{\left(K, \{K_s\}_{0 \le s \le s_e}\right)},...$ as the states that result from starting at $x_0 = 0$ and at each time $t < T_0$ using controller $C^{\theta^*}_{K}$ and at $t \ge T_0$ uses controller $C_{K_s}^{\theta^*}$, where $s = \lfloor \log_2\left(tT^{-1/2}\right)\rfloor$.  Define $\left(K^*, \{K^*_s\}_{0 \le s \le s_e}\right)$ as:
\begin{align*}
&\left(K^*, \{K_s^*\}_{0 \le s \le s_e}\right) \\
&:= \argmin_{\left(K, \{K_s\}_{0 \le s \le s_e}\right)}  \E\left[\sqrt{T}J\left(\theta^*,C^{\theta^*}_{K}, \sqrt{T}, 0, \{w_t\}_{t=0}^{T_0-1}\right) + \sum_{s=0}^{s_e} T_sJ(\theta^*,C^{\theta^*}_{K_s}, T_s,  x_{   T_s}^{\left(K, \{K_s\}_{0 \le s \le s_e}\right)}, W_s)\right].
\end{align*}
Define $x'_t$ as the state of the controller of Algorithm \ref{alg:cap3} at time $t$. Define $\hat{x}_{   T_0},\hat{x}_{   T_0+1},...$ as the sequence of random variables representing the sequence of states if the control at each time $t \ge    T_0$ is $C^{\hat{\theta}_s}_{K_{\mathrm{opt}}(\hat{\theta}_s)}(\hat{x}_t)$ for $s = \lfloor\log_2\left(tT^{-1/2}\right)\rfloor$ and starting at $\hat{x}_{   T_0} = x'_{   T_0}$. 

\begin{align*}
   & T\cdot J(\theta^*, C^{\mathrm{alg}},T,0,W) - T \cdot \bar{J}(\theta^*, C^{\theta^*}_{K_{\mathrm{opt}}(\theta^*, T)}, T) \\
   &\le T\cdot J(\theta^*, C^{\mathrm{alg}},T,0,W) - \E\left[\sqrt{T}J\left(\theta^*,C^{\theta^*}_{K^*}, \sqrt{T}, 0, \{w_t\}_{t=0}^{\sqrt{T}-1}\right) + \sum_{s=0}^{s_e} T_sJ(\theta^*,C^{\theta^*}_{K^*_s}, T_s,  x^*_{   T_s}, W_s)\right]  \\
   &\le T\cdot J(\theta^*, C^{\mathrm{alg}},T,0,W) - \E\left[\sum_{s=0}^{s_e} T_s\bar{J}(\theta^*,C^{\theta^*}_{K^*_s}, T_s,  x^*_{   T_s}, W_s)\right]  \\
    &= \underbrace{\sum_{s=0}^{s_e}  \E\left[T_sJ(\theta^*,C^{\hat{\theta}_s}_{K_{\mathrm{opt}}(\hat{\theta}_s, T_s)}, T_s,  0, W_s) \cond \hat{\theta}_s \right] - \E\left[ \sum_{s=0}^{s_e} T_s\bar{J}(\theta^*,C^{\theta^*}_{K^*_s}, T_s,  x^*_{   T_s}, W_s) \right]}_{R_1} \\
    &\quad \quad \quad \quad + \underbrace{\sum_{s=0}^{s_e}  \E\left[T_sJ(\theta^*,C^{\hat{\theta}_s}_{K_{\mathrm{opt}}(\hat{\theta}_s)}, T_s,  0, W_s)\cond \hat{\theta}_s\right] - \sum_{s=0}^{s_e}  \E\left[ T_sJ(\theta^*,C^{\hat{\theta}_s}_{K_{\mathrm{opt}}(\hat{\theta}_s, T_s)}, T_s,  0, W_s) \cond \hat{\theta}_s \right]}_{R_{1b}} \\
    &\quad \quad \quad \quad + \underbrace{\sum_{s=0}^{s_e}  T_sJ(\theta^*,C^{\hat{\theta}_s}_{K_{\mathrm{opt}}(\hat{\theta}_s)}, T_s,  \hat{x}_{   T_s}, W_s) -  \sum_{s=0}^{s_e}  \E\left[T_sJ(\theta^*,C^{\hat{\theta}_s}_{K_{\mathrm{opt}}(\hat{\theta}_s)}, T_s,  0, W_s) \cond \hat{\theta}_s\right]}_{R_{2}} \\
    &\quad \quad \quad \quad + \underbrace{\sum_{s=0}^{s_e}  T_sJ(\theta^*,C_s^{\mathrm{alg}}, T_s,  x'_{   T_s}, W_s) - \sum_{s=0}^{s_e} T_sJ(\theta^*,C^{\hat{\theta}_s}_{K_{\mathrm{opt}}(\hat{\theta}_s)}, T_s,  \hat{x}_{   T_s}, W_s)}_{R_{3}} \\
    &\quad \quad \quad \quad + \underbrace{T \cdot J(\theta^*, C^{\mathrm{alg}}, T,0,W) - \sum_{s=0}^{s_e}  T_sJ(\theta^*,C^{\mathrm{alg}}_s, T_s,  x'_{   T_s}, W_s)}_{R_0}.  \numberthis \label{eq:regret_3}
\end{align*}

Informally, we will show that with high probability $\epsilon_s = \tilde{O}_T(1/\sqrt{T_s})$ for all $s$.
    \begin{lemma}\label{bounded_st_b}
        Under Assumptions \ref{assum:constraints}--\ref{assum:initial} and \ref{assum_thm4}, there exists event $E_{\mathrm{L}\ref{bounded_st_b}}$  such that $\P(E_{\mathrm{L}\ref{bounded_st_b}}) = 1-o_T(1/T)$ and such that conditional on $\neg E_{\mathrm{E}\ref{eq:first_case}} \cap  E \cap E_{\mathrm{L}\ref{bounded_st_b}}$,
        \[
            \max_{s \in [0: s_e]} \epsilon_s\sqrt{T_s}  = \tilde{O}_T(1).
        \]
    \end{lemma}
    The proof of Lemma \ref{bounded_st_b} can be found in Appendix \ref{proof:bounded_st_b}.
Define event $E_3$ as 
    \[
        E_3 = \left\{\max_{s \in [0: s_e]} \epsilon_s\sqrt{T_s}  = \tilde{O}_T(1)\right\}.
    \]
Lemma \ref{bounded_st_b} implies that $\neg E_{\mathrm{E}\ref{eq:first_case}} \cap  E \cap E_{\mathrm{L}\ref{bounded_st_b}} \subseteq E_3$. Note that compared to the regret decomposition in \cite{schiffer2024stronger}, there is an extra regret term $R_{1b}$. This extra regret term can be thought of as the extra regret caused by choosing the best infinite horizon controller instead of the best finite horizon controller. The following lemma bounds the regret of this term by $\tilde{O}_T(\sqrt{T})$.
\begin{proposition}\label{fin_to_inf}
Define $R_{1b}$ as 
    \begin{equation}
        R_{1b} = \sum_{s=0}^{s_e}  \E\left[T_sJ(\theta^*,C^{\hat{\theta}_s}_{K_{\mathrm{opt}}(\hat{\theta}_s)}, T_s,  0, W_s)\cond \hat{\theta}_s\right] - \sum_{s=0}^{s_e}  \E\left[ T_sJ(\theta^*,C^{\hat{\theta}_s}_{K_{\mathrm{opt}}(\hat{\theta}_s, T_s)}, T_s,  0, W_s) \cond \hat{\theta}_s \right].
    \end{equation}
Under Assumptions \ref{assum:constraints}--\ref{assum:initial} and \ref{assum_thm4}, conditional on event $E \cap E_3$,
\[
    R_{1b} = \tilde{O}_T\left(\sqrt{T}\right).
\]
\end{proposition}
    The proof of Proposition \ref{fin_to_inf} can be found in Appendix \ref{sec:proof_of_fin_to_inf}. The following propositions bound the remaining regret terms.

\begin{proposition}[Regret from Randomness]\label{r1b_bound_trunc}
     Define $R_2$ as 
    \[
        R_2 := \sum_{s=0}^{s_e}  T_sJ(\theta^*,C^{\hat{\theta}_s}_{K_{\mathrm{opt}}(\hat{\theta}_s)}, T_s,  \hat{x}_{   T_s}, W_s) -  \sum_{s=0}^{s_e}  \E\left[T_sJ(\theta^*,C^{\hat{\theta}_s}_{K_{\mathrm{opt}}(\hat{\theta}_s)}, T_s,  0, W_s) \; \cond \; \hat{\theta}_s \right].
    \]
    Then under Assumptions \ref{assum:constraints}--\ref{assum:initial} and \ref{assum_thm4} there exists an event $E_{\mathrm{P}\ref{r1b_bound_trunc}}$ such that $\P(E_{\mathrm{P}\ref{r1b_bound_trunc}}) = 1 - o_T(1/T)$ and  conditional on  $E_{\mathrm{P}\ref{r1b_bound_trunc}} \cap \neg E_{\mathrm{E}\ref{eq:first_case}} \cap E$,
    \begin{equation}\label{eq:r1b_bound_trunc}
         R_2 = \tilde{O}_T(\sqrt{T}).
    \end{equation}
\end{proposition}
The proof of Proposition \ref{r1b_bound_trunc} can be found in Appendix \ref{proof:r1b_bound_trunc}.

\begin{proposition}\label{non_optimal_controller_trunc}
 Define $R_1$ as 
 \[
    R_1 :=  \sum_{s=0}^{s_e}  \E\left[T_sJ(\theta^*,C^{\hat{\theta}_s}_{K_{\mathrm{opt}}(\hat{\theta}_s, T_s)}, T_s, 0, W_s) \cond \hat{\theta}_s \right] -  \E\left[\sum_{s=0}^{s_e} T_sJ(\theta^*,C^{\theta^*}_{K^*_s}, T_s,  x^*_{   T_s}, W_s)  \right].
 \]
Under Assumptions \ref{assum:constraints}--\ref{assum:initial} and \ref{assum_thm4}, conditional on event $E_3 \cap  E$, \begin{equation}\label{eq:non_optimal_controller_trunc}
       R_1 = \tilde{O}_T\left(\sqrt{T}\right).
    \end{equation}
\end{proposition}
The proof of Proposition \ref{non_optimal_controller_trunc} can be found in Appendix \ref{proof:non_optimal_controller_trunc}.

\begin{proposition}\label{enforcing_safety_trunc}
    Define $R_3$ as (the random variable)
    \[
        R_3 := \sum_{s=0}^{s_e}  T_sJ(\theta^*,C^{\mathrm{alg}}_s, T_s,  x'_{   T_s}, W_s) - \sum_{s=0}^{s_e} T_sJ(\theta^*,C^{\hat{\theta}_s}_{K_{\mathrm{opt}}(\hat{\theta}_s)}, T_s,  \hat{x}_{   T_s}, W_s).
    \]
    Then under Assumptions \ref{assum:constraints}--\ref{assum:initial} and \ref{assum_thm4},there exists an event $E_{\mathrm{P}\ref{enforcing_safety_trunc}}$ such that $\P(E_{\mathrm{P}\ref{enforcing_safety_trunc}}) = 1 - o_T(1/T)$ and  conditional on  $E_{\mathrm{P}\ref{enforcing_safety_trunc}} \cap \neg E_{\mathrm{E}\ref{eq:first_case}} \cap E \cap E_3$,
    \begin{equation}\label{eq:enforcing_safety_trunc}
        R_3 = \tilde{O}_T(\sqrt{T}).
    \end{equation}
\end{proposition}
The proof of Proposition \ref{enforcing_safety_trunc} can be found in Appendix \ref{proof:enforcing_safety_trunc}.

\begin{proposition}\label{warmup_firstcase}
Under Assumptions \ref{assum:constraints}--\ref{assum:initial} and \ref{assum_thm4}, conditional on event $E$,
    \begin{equation}\label{eq:warmup_F}
        T \cdot J(\theta^*, C^{\mathrm{alg}}, T,0,W) - \sum_{s=0}^{s_e}  T_sJ(\theta^*,C^{\mathrm{alg}}_s, T_s,  x'_{   T_s}, W_s) = \tilde{O}_T(\sqrt{T}).
    \end{equation}
\end{proposition}
    The proof of Proposition \ref{warmup_firstcase} can be found in Appendix \ref{proof:warmup_firstcase}.

Using Equation \eqref{eq:regret_3} combined with Propositions \ref{warmup_firstcase}, \ref{fin_to_inf}, \ref{r1b_bound_trunc}, \ref{non_optimal_controller_trunc} and \ref{enforcing_safety_trunc}, conditional on event $\neg E_{\mathrm{E}\ref{eq:first_case}} \cap E_3 \cap E \cap E_{\mathrm{P}\ref{enforcing_safety_trunc}} \cap E_{\mathrm{P}\ref{r1b_bound_trunc}}$ the total regret is upper bounded by 
\[
       T\cdot J(\theta^*, C^{\mathrm{alg}},T) - T \cdot \bar{J}(\theta^*, C^{\theta^*}_{K_{\mathrm{opt}}(\theta^*, T)}, T) \le R_0 + R_1 + R_{1b} +  R_{2} + R_3 = \tilde{O}_T\left(\sqrt{T}\right).
\]
Combining Propositions \ref{r1b_bound_trunc} and \ref{enforcing_safety_trunc}, $\P(E_{\mathrm{P}\ref{enforcing_safety_trunc}} \cap E_{\mathrm{P}\ref{r1b_bound_trunc}}) = 1-o_T(1/T)$. Therefore, we have that
\begin{align*}
    &\P(E_3 \cap E \cap \neg E_{\mathrm{E}\ref{eq:first_case}} \cap E_{\mathrm{P}\ref{enforcing_safety_trunc}} \cap E_{\mathrm{P}\ref{r1b_bound_trunc}}) \\
    &= \P( E_3 \cap E \cap \neg E_{\mathrm{E}\ref{eq:first_case}})  - o_T(1/T) && \text{Remark \ref{remark_conditioning}}\\
    &\ge \P(E_{\mathrm{L}\ref{bounded_st_b}}\cap E  \cap \neg E_{\mathrm{E}\ref{eq:first_case}}) - o_T(1/T)  && \text{Lemma \ref{bounded_st_b}}\\
    &\ge \P(\neg E_{\mathrm{E}\ref{eq:first_case}}) - o_T(1/T). && \text{Remark \ref{remark_conditioning}}
\end{align*}

Above, we twice used the following remark:

\begin{remark}\label{remark_conditioning}
    If two events $\mathcal{E}_1$ and $\mathcal{E}_2$ satisfy that $\P(\mathcal{E}_1) = 1-o_T(1/T)$, then
    \[
        \P(\mathcal{E}_1 \cap \mathcal{E}_2) = \P(\mathcal{E}_1) + \P(\mathcal{E}_2) - \P(\mathcal{E}_1 \cup \mathcal{E}_2) \ge \P(\mathcal{E}_2) - o_T(1/T)
    \]
\end{remark}

Taking $E_{\mathrm{P}\ref{prop_Kcase}} = E_3 \cap E \cap \neg E_{\mathrm{E}\ref{eq:first_case}} \cap E_{\mathrm{P}\ref{enforcing_safety_trunc}} \cap E_{\mathrm{P}\ref{r1b_bound_trunc}}$ gives the desired result.

\end{proof}

\subsection{Proof of Proposition \ref{prop_Fcase} }\label{sec:proof_of_prop_Fcase}

Informally, $E_{\mathrm{E}\ref{eq:first_case}}$ implies that the optimal linear controller for $\theta^*$ is close to satisfying the constraints. Therefore, we will bound the regret by approximating both the best constrained controller and the controller of Algorithm \ref{alg:cap3} by the optimal unconstrained linear controller.

We will decompose the regret as follows. Define $C^{\mathrm{alg}'}$ to be the controller of Algorithm \ref{alg:cap3} after the warm-up period, i.e. starting at time $t = T_0$. Therefore, $C^{\mathrm{alg'}}_t = C^{\mathrm{alg}}_{t + T_0}$. Define $x'_0,x'_1,...$ as the series of states when using algorithm $C^{\mathrm{alg}}$. Define $W' = \{w_i\}_{i=T_0}^{T-1}$. Recall that $C^{\mathrm{unc}}_{K}$ is the linear controller such that $C^{\mathrm{unc}}_K(x) = -Kx$. We can decompose the regret as follows:
\begin{align*}
    &T \cdot J(\theta^*, C^{\mathrm{alg}}, T,0,W)  - T \cdot \bar{J}(\theta^*, C_{K_{\mathrm{opt}}(\theta^*, T)}^{\theta^*}, T) \\
    &\le T \cdot J(\theta^*, C^{\mathrm{alg}}, T,0,W)  - (T-T_0) \cdot \bar{J}(\theta^*, C_{K_{\mathrm{opt}}(\theta^*, T)}^{\theta^*}, T-T_0) \\
    &= \underbrace{(T - T_0) \cdot  \bar{J}(\theta^*, C_{F_{\mathrm{opt}}(\hat{\theta}_{\mathrm{wu}})}^{\mathrm{unc}}, T -T_0) - (T - T_0) \cdot  \bar{J}(\theta^*, C_{K_{\mathrm{opt}}(\theta^*, T)}^{\theta^*}, T - T_0)}_{R'_1} \\
    & \quad \quad + \underbrace{(T - T_0) \cdot  J(\theta^*, C_{F_{\mathrm{opt}}(\hat{\theta}_{\mathrm{wu}})}^{\mathrm{unc}}, T - T_0,0,W') -  (T - T_0) \cdot  \bar{J}(\theta^*, C_{F_{\mathrm{opt}}(\hat{\theta}_{\mathrm{wu}})}^{\mathrm{unc}}, T - T_0) }_{R'_2} \\
    & \quad \quad + \underbrace{ (T - T_0) \cdot  J(\theta^*, C_{F_{\mathrm{opt}}(\hat{\theta}_{\mathrm{wu}})}^{\mathrm{unc}}, T - T_0, x'_{T_0},W') - (T - T_0) \cdot  J(\theta^*, C_{F_{\mathrm{opt}}(\hat{\theta}_{\mathrm{wu}})}^{\mathrm{unc}}, T - T_0,0,W') }_{R'_3} \\
    & \quad \quad +  \underbrace{(T - T_0) \cdot  J(\theta^*, C^{\mathrm{alg}'}, T - T_0, x'_{T_0},W') - (T - T_0) \cdot  J(\theta^*, C_{F_{\mathrm{opt}}(\hat{\theta}_{\mathrm{wu}})}^{\mathrm{unc}}, T - T_0, x_{T_0}',W')  }_{R'_4} \\
    & \quad \quad + \underbrace{ T \cdot J(\theta^*, C^{\mathrm{alg}}, T,0,W)  - (T - T_0) \cdot J(\theta^*, C^{\mathrm{alg}'}, T - T_0, x'_{T_0},W')}_{R'_5}. \numberthis \label{eq:first_case_regret_decomposition}
\end{align*}

We will now individually analyze each of these components of regret. The first component of regret ($R'_1$) is the extra expected cost of using $C_{F_{\mathrm{opt}}(\hat{\theta}_{\mathrm{wu}})}^{\mathrm{unc}}$ versus $C^{\theta^*}_{K_{\mathrm{opt}}(\theta^*, T)}$. We will bound that regret with the following proposition. 

\begin{proposition}\label{close_J2}
Under Assumptions \ref{assum:constraints}--\ref{assum:initial} and \ref{assum_thm4}, conditional on event $E_2^0$,
    \begin{equation}\label{eq:close_J2}
        (T-T_0) \cdot \bar{J}(\theta^*, C_{F_{\mathrm{opt}}(\hat{\theta}_{\mathrm{wu}})}^{\mathrm{unc}}, T-T_0) - (T-T_0) \cdot \bar{J}(\theta^*, C_{K_{\mathrm{opt}}(\theta^*, T)}^{\theta^*}, T-T_0) = \tilde{O}_T(\sqrt{T}).
    \end{equation}
\end{proposition}
    The proof of Proposition \ref{close_J2} can be found in Appendix \ref{sec:proof_of_close_J2}.

The next source of regret ($R'_2$) is the variation in the realization of the $T-T_0$ time step cost versus the expected cost. We will bound this regret with Proposition \ref{close_safe_J}.

\begin{proposition}\label{close_safe_J}
Under Assumptions \ref{assum:constraints}--\ref{assum:initial} and \ref{assum_thm4}, there exists an event $E_{\mathrm{P}\ref{close_safe_J}}$ such that $\P(E_{\mathrm{P}\ref{close_safe_J}}) = 1-o_T(1/T)$ and such that conditional on event $E_{\mathrm{P}\ref{close_safe_J}}$,
    \begin{equation}\label{eq:close_safe_J}
        \left|(T-T_0) \cdot J(\theta^*, C_{F_{\mathrm{opt}}(\hat{\theta}_{\mathrm{wu}})}^{\mathrm{unc}}, T-T_0, 0, W') - (T-T_0) \cdot \bar{J}(\theta^*, C_{F_{\mathrm{opt}}(\hat{\theta}_{\mathrm{wu}})}^{\mathrm{unc}}, T-T_0)\right| = \tilde{O}_T(\sqrt{T}).
    \end{equation}
\end{proposition}
    The proof of Proposition \ref{close_safe_J} can be found in Appendix \ref{sec:proof_of_close_safe_J}.

The next source of regret ($R'_3$) comes from the starting state of the controller $C_{F_{\mathrm{opt}}(\hat{\theta}_{\mathrm{wu}})}^{\mathrm{unc}}$. We will bound this regret with Proposition \ref{F_start}.

\begin{proposition}\label{F_start}
Under Assumptions \ref{assum:constraints}--\ref{assum:initial} and \ref{assum_thm4}, conditional on event $E$, 
\begin{equation}\label{eq:close_safe_F}
        \left|(T-T_0) \cdot J(\theta^*, C_{F_{\mathrm{opt}}(\hat{\theta}_{\mathrm{wu}})}^{\mathrm{unc}}, T-T_0, x'_{T_0},W') - (T-T_0) \cdot J(\theta^*, C_{F_{\mathrm{opt}}(\hat{\theta}_{\mathrm{wu}})}^{\mathrm{unc}}, T-T_0, 0,W')\right| = \tilde{O}_T(1).
    \end{equation}
\end{proposition}
    The proof of Proposition \ref{F_start} can be found in Appendix \ref{sec:proof_of_F_start}.

The next component of regret ($R'_4$) is the additional cost of enforcing safety on top of the controller $C_{F_{\mathrm{opt}}(\hat{\theta}_{\mathrm{wu}})}^{\mathrm{unc}}$. Define event $E_{\mathrm{safe}}^{\mathrm{wu}}$ as the event that the first $\sqrt{T}$ controls used by controller $C^{\mathrm{alg}}$ are safe for dynamics $\theta^*$.

\begin{proposition}\label{safety_is_cheap}
Under Assumptions \ref{assum:constraints}--\ref{assum:initial} and \ref{assum_thm4}, there exists an event $E_{\mathrm{P}\ref{safety_is_cheap}}$ such that $\P(E_{\mathrm{P}\ref{safety_is_cheap}} \mid E_{\mathrm{E}\ref{eq:first_case}} \cap E_2^0  \cap E_{\mathrm{safe}}^{\mathrm{wu}} )  = 1-o_T(1/T)$ and such that conditional on $E_{\mathrm{E}\ref{eq:first_case}} \cap E_2^0  \cap E_{\mathrm{safe}}^{\mathrm{wu}} \cap E_{\mathrm{P}\ref{safety_is_cheap}}$,
    \begin{equation}\label{eq:safety_is_cheap}
        \left| (T-T_0) \cdot J(\theta^*, C^{\mathrm{alg}'}, T-T_0,x'_{T_0},W') -  (T-T_0) \cdot \bar{J}(\theta^*, C_{F_{\mathrm{opt}}(\hat{\theta}_{\mathrm{wu}})}^{\mathrm{unc}}, T-T_0, x'_{T_0}, W')\right| = \tilde{O}_T(\sqrt{T}).
    \end{equation}
\end{proposition}
    The proof of Proposition \ref{safety_is_cheap} can be found in Appendix \ref{sec:proof_of_safety_is_cheap}.

The last source of regret is the regret from the warm-up period. By Proposition \ref{warmup_firstcase}, this source of regret is $\tilde{O}(\sqrt{T})$ conditional on event $E$, because by definition $T \cdot J(\theta^*, C^{\mathrm{alg}}, T,0,W)  - (T - T_0) \cdot J(\theta^*, C^{\mathrm{alg}'}, T - T_0, x'_{T_0},W') = T \cdot J(\theta^*, C^{\mathrm{alg}}, T,0,W) - \sum_{s=0}^{s_e}  T_sJ(\theta^*,C^{\mathrm{alg}}_s, T_s,  x'_{   T_s}, W_s)$.

Recall that  $E \subseteq E_2^0 \cap E_{\mathrm{safe}}^{\mathrm{wu}}$. Therefore, conditional on $ E_{\mathrm{P}\ref{safety_is_cheap}} \cap E_{\mathrm{P}\ref{close_safe_J}} \cap E \cap E_{\mathrm{E}\ref{eq:first_case}} $, by Equation \eqref{eq:first_case_regret_decomposition} and Propositions \ref{close_J2}, \ref{close_safe_J}, \ref{F_start}, \ref{safety_is_cheap}, and \ref{warmup_firstcase}, we have that
\[
    T \cdot J(\theta^*, C^{\mathrm{alg}}, T,0, W)  - T \cdot \bar{J}(\theta^*, C_{K_{\mathrm{opt}}(\theta^*, T)}^{\theta^*}, T) = \tilde{O}_T(\sqrt{T}).
\]
Furthermore, because $\P(E_2^0  \cap E_{\mathrm{safe}}^{\mathrm{wu}}) \ge \P(E) \ge 1 -o_T(1/T)$,  we have that
\begin{align*}
    &\P( E_{\mathrm{P}\ref{safety_is_cheap}} \cap E_{\mathrm{P}\ref{close_safe_J}} \cap E \cap E_{\mathrm{E}\ref{eq:first_case}} ) \\
    &= \P( E_{\mathrm{P}\ref{safety_is_cheap}}\cap E_2^0  \cap E_{\mathrm{safe}}^{\mathrm{wu}}  \cap E \cap E_{\mathrm{E}\ref{eq:first_case}} )  - o_T(1/T) && \text{Remark \ref{remark_conditioning}} \\
     &= \P( E_{\mathrm{P}\ref{safety_is_cheap}} \cap E_2^0  \cap E_{\mathrm{safe}}^{\mathrm{wu}} 
 \cap E_{\mathrm{E}\ref{eq:first_case}} )  - o_T(1/T) && \text{Remark \ref{remark_conditioning}} \\
    &= \P(E_{\mathrm{P}\ref{safety_is_cheap}}  \mid E_2^0  \cap E_{\mathrm{safe}}^{\mathrm{wu}} 
 \cap E_{\mathrm{E}\ref{eq:first_case}} )\P( E_2^0  \cap E_{\mathrm{safe}}^{\mathrm{wu}} 
 \cap E_{\mathrm{E}\ref{eq:first_case}} ) - o_T(1/T)\\
    &\ge (1-o_T(1/T)) \P( E_2^0  \cap E_{\mathrm{safe}}^{\mathrm{wu}} 
 \cap E_{\mathrm{E}\ref{eq:first_case}} )  - o_T(1/T)\\
    &= \P( E_2^0  \cap E_{\mathrm{safe}}^{\mathrm{wu}} 
 \cap E_{\mathrm{E}\ref{eq:first_case}} )  - o_T(1/T) \\
    &= \P(E_{\mathrm{E}\ref{eq:first_case}}) - o_T(1/T). && \text{Remark \ref{remark_conditioning}}
\end{align*}
Taking $E_{\mathrm{P}\ref{prop_Fcase}} =  E_{\mathrm{P}\ref{safety_is_cheap}} \cap E_{\mathrm{P}\ref{close_safe_J}}  \cap E \cap E_{\mathrm{E}\ref{eq:first_case}} $ gives the desired result.

\section{Proofs from Appendix \ref{sec:proof_of_prop_Kcase}}\label{sec:trunc_lin_sqrtt_proofs}

\subsection{Proof of Lemma \ref{bounded_st_b}}\label{proof:bounded_st_b}   
\begin{proof}
We will use the following equivalent version of Lemma \refgen{boundary_uncertainty_cont} for Algorithm \ref{alg:cap3}.
\begin{lemma}\label{boundary_uncertainty_cont_b}
    Let $x_t, u_t$ respectively be the state and control of $C^{\mathrm{alg}}$ (the controller of Algorithm \ref{alg:cap3}) at time $t$ starting at $x_0 = 0$. Define $G_i = (x_0,u_0,...,x_{i-1}, u_{i-1})$. For constant $\gamma > 0$, define $S_t$ as 
    \begin{equation}
        S_t = \Big\{i< t : u_i = u_i^{\mathrm{safeU}} \text{ $\mathrm{and}$ } \P(u_i = u_i^{\mathrm{safeU}} \mid G_i, E) \ge \gamma \Big\}.
    \end{equation}
    Then under Assumptions \ref{assum:constraints}--\ref{assum:initial} and for sufficiently large $T$, with probability $1-o_T(1/T)$, 
    \begin{equation}\label{eq:boundary_uncertainty_cont_b}
        \max_{s \in [0:s_e]} \epsilon_s \sqrt{|S_{T_s}|} = \tilde{O}_T\left(1\right).
    \end{equation}
\end{lemma}
The proof of Lemma \ref{boundary_uncertainty_cont_b} can be found in Appendix \ref{sec:proof_of_boundary_uncertainty_cont_b}.

While we have not yet explained the significance of Lemma \ref{cswitch_prop}, we state it here because the definition of $\epsilon^*$ is needed for other definitions below.  
\begin{lemma}\label{cswitch_prop}
Define
\begin{equation}\label{eq:epsilonstar}
    \epsilon^* := \bar{w} - \left(\frac{D_U}{a^*-b^*K_{\mathrm{opt}}(\theta^*)} - D_U\right).
\end{equation}
Then event $\neg E_{\mathrm{E}\ref{eq:first_case}} \cap E$ can only hold if $\epsilon^* > 0$.

\end{lemma}
    The proof of Lemma \ref{cswitch_prop} can be found in Appendix \ref{sec:proof_of_cswitch_prop}.

Define $\gamma_\epsilon = \frac{\P_{w \sim \mathcal{D}}(w \ge \bar{w} - 3\epsilon^*/8)}{2}$ (which is a constant) and define $S'_t$ as
\begin{equation}\label{eq:St_trunc}
    S'_t := \Big\{i< t : u_i = u_i^{\mathrm{safeU}} \text{ and } \P(u_i = u_i^{\mathrm{safeU}} \mid G_i, E) \ge \gamma_\epsilon  \Big\}.
\end{equation}
Note that this is the same as the definition of $S_t$ in Lemma \ref{boundary_uncertainty_cont_b} except with $\gamma = \gamma_\epsilon$.

\begin{lemma}\label{sufficiently_many_boundaries}
    Under Assumptions \ref{assum:constraints}--\ref{assum:initial} and \ref{assum_thm4}, there exists an event $E_{\mathrm{L}\ref{sufficiently_many_boundaries}}$ such that $\P(E_{\mathrm{L}\ref{sufficiently_many_boundaries}}) \ge 1 - o_T(1/T)$ and such that conditional on event $E_{\mathrm{L}\ref{sufficiently_many_boundaries}} \cap\neg E_{\mathrm{E}\ref{eq:first_case}}$, 
   \[
    \max_{s \in [1:s_e]} \frac{T_s}{\left|S'_{T_s}\right|} = \tilde{O}_T(1).
   \]
\end{lemma}
    The proof of Lemma \ref{sufficiently_many_boundaries} can be found in Appendix \ref{sec:proof_of_sufficiently_many_boundaries}.

Define $E_{\mathrm{L}\ref{boundary_uncertainty_cont_b}}$ as the event that Equation \eqref{eq:boundary_uncertainty_cont_b} holds for $S_{T_s} = S'_{T_s}$. Then $\P(E_{\mathrm{L}\ref{boundary_uncertainty_cont_b}}) = 1-o_T(1/T)$ by Lemma \ref{boundary_uncertainty_cont_b}.
By Lemma \ref{sufficiently_many_boundaries}, conditional on event $E_{\mathrm{L}\ref{boundary_uncertainty_cont_b}} \cap E_{\mathrm{L}\ref{sufficiently_many_boundaries}} \cap\neg E_{\mathrm{E}\ref{eq:first_case}}$, 
\[
    \max_{s \in [1:s_e]} \epsilon_s\sqrt{T_s} \le \sqrt{\max_{s \in [1:s_e]} \frac{T_s}{\left|S'_{T_s}\right|} }\left(\max_{s \in [1:s_e]} \epsilon_s \sqrt{|S'_{T_s}|}\right)= \tilde{O}_T(1).
\]
Under event $E_2$, we also have that $\epsilon_0\sqrt{T_0} = \tilde{O}_T(T^{-1/4})T^{1/4} = \tilde{O}_T(1)$. Because $E \subseteq E_2$ this implies that conditional on $E$, we have $\epsilon_0\sqrt{T_0} = \tilde{O}_T(1)$.

Therefore, conditional on $E_{\mathrm{L}\ref{boundary_uncertainty_cont_b}} \cap E_{\mathrm{L}\ref{sufficiently_many_boundaries}} \cap\neg E_{\mathrm{E}\ref{eq:first_case}} \cap E$, 
\[
    \max_{s \in [0:s_e]} \epsilon_s\sqrt{T_s} = \tilde{O}_T(1).
\]
Taking $E_{\mathrm{L}\ref{bounded_st_b}} = E_{\mathrm{L}\ref{boundary_uncertainty_cont_b}} \cap E_{\mathrm{L}\ref{sufficiently_many_boundaries}}$ gives the desired result because $\P(E_{\mathrm{L}\ref{bounded_st_b}}) = 1-o_T(1/T)$ by a union bound.
\end{proof}

\subsection{Proof of Proposition \ref{fin_to_inf}}\label{sec:proof_of_fin_to_inf}
    \begin{proof}
        The goal of this proposition is to show that using the infinite horizon controller is not significantly worse than using the finite horizon controller. This proof will use the following lemma.
 
    \begin{lemma}\label{split}
        Under Assumptions \ref{assum:constraints}--\ref{assum:initial} and \ref{assum_thm4}, for any $\theta \in \Theta$ and $K \in [\frac{a-1}{b}, \frac{a}{b}]$,
        \[
            |\bar{J}(\theta, C_K^\theta, T) -  \bar{J}(\theta, C_K^\theta)| = \tilde{O}_T\left(\frac{1}{T}\right).
        \]
    \end{lemma}

    The proof of Lemma \ref{split} can be found in Appendix \ref{sec:proof_of_split}.

        We can apply Lemma \ref{split} to get the following two equations:
        \begin{equation}\label{eq:lessthanlog1}
            \left| \bar{J}(\hat{\theta}_s,C^{\hat{\theta}_s}_{K_{\mathrm{opt}}(\hat{\theta}_s)}, T_s) - \bar{J}(\hat{\theta}_s,C^{\hat{\theta}_s}_{K_{\mathrm{opt}}(\hat{\theta}_s)})\right| = \tilde{O}_T\left(\frac{1}{T_s}\right)
        \end{equation}
        \begin{equation}\label{eq:lessthanlog2}
            \left| \bar{J}(\hat{\theta}_s,C^{\hat{\theta}_s}_{K_{\mathrm{opt}}(\hat{\theta}_s, T_s)}, T_s) - \bar{J}(\hat{\theta}_s,C^{\hat{\theta}_s}_{K_{\mathrm{opt}}(\hat{\theta}_s, T_s)})\right| = \tilde{O}_T\left(\frac{1}{T_s}\right).
        \end{equation}
        By definition, we also also have the following two inequalities.
        \begin{equation}\label{eq:bydefofJ1}
             \bar{J}(\hat{\theta}_s,C^{\hat{\theta}_s}_{K_{\mathrm{opt}}(\hat{\theta}_s, T_s)}, T_s) \le  \bar{J}(\hat{\theta}_s,C^{\hat{\theta}_s}_{K_{\mathrm{opt}}(\hat{\theta}_s)}, T_s)
        \end{equation}
        \begin{equation}\label{eq:bydefofJ2}
             \bar{J}(\hat{\theta}_s,C^{\hat{\theta}_s}_{K_{\mathrm{opt}}(\hat{\theta}_s)}) \le  \bar{J}(\hat{\theta}_s,C^{\hat{\theta}_s}_{K_{\mathrm{opt}}(\hat{\theta}_s, T_s)}).
        \end{equation}
        Combining Equations \eqref{eq:lessthanlog1}--\eqref{eq:bydefofJ2}, we have that
        \begin{align*}
            \bar{J}(\hat{\theta}_s,C^{\hat{\theta}_s}_{K_{\mathrm{opt}}(\hat{\theta}_s, T_s)}, T_s)  &\ge \bar{J}(\hat{\theta}_s,C^{\hat{\theta}_s}_{K_{\mathrm{opt}}(\hat{\theta}_s, T_s)})- \tilde{O}_T\left(\frac{1}{T_s}\right)  && \text{Equation \eqref{eq:lessthanlog2}}\\
           &\ge  \bar{J}(\hat{\theta}_s,C^{\hat{\theta}_s}_{K_{\mathrm{opt}}(\hat{\theta}_s)}) - \tilde{O}_T\left(\frac{1}{T_s}\right)  && \text{Equation \eqref{eq:bydefofJ2}}\\
           &\ge  \bar{J}(\hat{\theta}_s,C^{\hat{\theta}_s}_{K_{\mathrm{opt}}(\hat{\theta}_s)}, T_s)  - \tilde{O}_T\left(\frac{1}{T_s}\right).  && \text{Equation \eqref{eq:lessthanlog1}.}
        \end{align*}
        Combining this with Equation \eqref{eq:bydefofJ1} gives that
        \begin{equation}\label{eq:converttheta0}
            \left| \bar{J}(\hat{\theta}_s,C^{\hat{\theta}_s}_{K_{\mathrm{opt}}(\hat{\theta}_s, T_s)}, T_s) - \bar{J}(\hat{\theta}_s,C^{\hat{\theta}_s}_{K_{\mathrm{opt}}(\hat{\theta}_s)}, T_s)\right| = \tilde{O}_T\left(\frac{1}{T_s}\right).
        \end{equation}
        This is almost the desired result, but to bound the regret term $R_{1b}$ we need to bound the difference under dynamics $\theta^*$, not under $\hat{\theta}_s$. Conditional on event $E$, $\norm{\hat{\theta}_s - \theta^*}_{\infty} = \tilde{O}_T(T^{-1/4}) \le \frac{1}{\log^{46}(T)}$ for sufficiently large $T$, and therefore Lemma \ref{diff_in_theta} implies the following inequalities for sufficiently large $T$:
        \begin{equation}\label{eq:converttheta1}
            \left| \bar{J}(\hat{\theta}_s,C^{\hat{\theta}_s}_{K_{\mathrm{opt}}(\hat{\theta}_s, T_s)}, T_s) - \bar{J}(\theta^*,C^{\hat{\theta}_s}_{K_{\mathrm{opt}}(\hat{\theta}_s, T_s)}, T_s)\right| = \tilde{O}_T\left(\norm{\hat{\theta}_s - \theta^*}_{\infty} + \frac{1}{T^2}\right)
        \end{equation}
        \begin{equation}\label{eq:converttheta2}
            \left| \bar{J}(\hat{\theta}_s,C^{\hat{\theta}_s}_{K_{\mathrm{opt}}(\hat{\theta}_s)}, T_s) - \bar{J}(\theta^*,C^{\hat{\theta}_s}_{K_{\mathrm{opt}}(\hat{\theta}_s)}, T_s)\right| = \tilde{O}_T\left(\norm{\hat{\theta}_s - \theta^*}_{\infty} + \frac{1}{T^2}\right).
        \end{equation}
        Putting together Equations \eqref{eq:converttheta0}, \eqref{eq:converttheta1}, \eqref{eq:converttheta2}, and the fact that $T_s \le T^2$, we have

        \begin{equation}\label{inf_to_notinf}
            \left| \bar{J}(\theta^*,C^{\hat{\theta}_s}_{K_{\mathrm{opt}}(\hat{\theta}_s, T_s)}, T_s) - \bar{J}(\theta^*,C^{\hat{\theta}_s}_{K_{\mathrm{opt}}(\hat{\theta}_s)}, T_s)\right| \le \tilde{O}_T\left(\norm{\hat{\theta}_s - \theta^*}_{\infty} + \frac{1}{T_s}\right).
        \end{equation}
        Now we are ready to use Equation \eqref{inf_to_notinf} to bound $R_{1b}$ conditional on event $E \cap E_3$:
        \begin{align*}
                &R_{1b}\\
                &= \sum_{s=0}^{s_e}  \E\left[T_sJ(\theta^*,C^{\hat{\theta}_s}_{K_{\mathrm{opt}}(\hat{\theta}_s)}, T_s,  0, W_s)\cond \hat{\theta}_s\right] - \sum_{s=0}^{s_e}  \E\left[ T_sJ(\theta^*,C^{\hat{\theta}_s}_{K_{\mathrm{opt}}(\hat{\theta}_s, T_s)}, T_s, W_s) \cond \hat{\theta}_s \right] \\
                 &= \sum_{s=0}^{s_e}  T_s\bar{J}(\theta^*,C^{\hat{\theta}_s}_{K_{\mathrm{opt}}(\hat{\theta}_s)}, T_s) - \sum_{s=0}^{s_e}   T_s\bar{J}(\theta^*,C^{\hat{\theta}_s}_{K_{\mathrm{opt}}(\hat{\theta}_s, T_s)}, T_s)  \\
                 &\le \sum_{s=0}^{s_e}  T_s\left|\bar{J}(\theta^*,C^{\hat{\theta}_s}_{K_{\mathrm{opt}}(\hat{\theta}_s)}, T_s) - \bar{J}(\theta^*,C^{\hat{\theta}_s}_{K_{\mathrm{opt}}(\hat{\theta}_s, T_s)}, T_s)  \right|\\
                &= \tilde{O}_T\left(\sum_{s=0}^{s_e} T_s \left(\norm{\hat{\theta}_s- \theta^* }_{\infty} + \frac{1}{T_s}\right)\right) .  &&\text{Eq \eqref{inf_to_notinf}} \\
                &= \tilde{O}_T\left(s_e + \sum_{s=0}^{s_e} T_s\epsilon_s\right) && \text{Event $E$} \\
                &= \tilde{O}_T(\sqrt{T}) && \text{Event $E_3$}
        \end{align*}
        The last line follows from the fact that $s_e = \tilde{O}_T(1)$ and that under event $E_3$, $T_s\epsilon_s = \sqrt{T_s} \left(\epsilon_s\sqrt{T_s}\right) = \tilde{O}_T(\sqrt{T_s}) = \tilde{O}_T(\sqrt{T})$,
    \end{proof}

\subsection{Proof of Proposition \ref{r1b_bound_trunc}}\label{proof:r1b_bound_trunc}

Because the events $E$ and $E_3$ are defined equivalently to the events in Appendix \refgen{app:suff_large_noise_case}, this proof is very similar to the proof of Proposition \refgen{r1b_bound_large} with the events and variables with respect to Algorithm \ref{alg:cap3} in this paper instead of Algorithm \refgen{alg:cap_large}. There are two differences between this proof and that of Proposition \refgen{r1b_bound_large}. The first difference is that the subscript on the controller is $K_{\mathrm{opt}}(\hat{\theta}_s)$ rather than $K_{\mathrm{opt}}(\hat{\theta}_s, T_s)$. The proof of Proposition \refgen{r1b_bound_large} follows the proof of Proposition \refgen{r1b_bound}, and primarily relies on analogous versions of Lemma \refgen{concentration_of_cond_exp} and Lemma \refgen{uncond_vs_cond_regret}. Examining the proofs of these lemmas, the proofs (and analogous results) hold for any controller $C_K^{\hat{\theta}_s}$ where $K \in [K_\mathrm{L}^{\hat{\theta}_s}, K_{\mathrm{U}}^{\hat{\theta}_s}]$. This is because the value of $K$ is not used anywhere in the proof. Therefore, analogous versions of these lemmas hold for Algorithm \ref{alg:cap3} with  $K_{\mathrm{opt}}(\hat{\theta}_s)$ instead of $K_{\mathrm{opt}}(\hat{\theta}_s, T_s)$.

The second major difference is that Proposition  \refgen{r1b_bound_large} state that the result holds conditional on $E$ with high probability, while Proposition \ref{r1b_bound_trunc} holds conditional on $E \cap E_{\mathrm{P}\ref{r1b_bound_trunc}}$. In the proof of Proposition \refgen{r1b_bound} (specifically Equation \eqrefgen{eq:final_concent}), we can define the event 
\[
E_{\mathrm{E}\ref*{general-eq:final_concent}} := \left\{ \sum_{s=0}^{s_e}  T_sJ(\theta^*,C^{\hat{\theta}_s}_{K_{\mathrm{opt}}(\hat{\theta}_s)}, T_s, 0, W_s) - \sum_{s=0}^{s_e}  \E\left[T_sJ(\theta^*,C^{\hat{\theta}_s}_{K_{\mathrm{opt}}(\hat{\theta}_s)}, T_s, 0, W_s) \cond \hat{\theta}_s \right] \ge  \tilde{O}_T(\sqrt{T}) \right\}.
\]
 Note that we replaced $K_{\mathrm{opt}}(\hat{\theta}_s, T_s)$ with $K_{\mathrm{opt}}(\hat{\theta}_s)$ for reasons discussed in the previous paragraph. Equation \eqrefgen{eq:final_concent} implies that $\P(E_{\mathrm{E}\ref*{general-eq:final_concent}}) = 1-o_T(1/T)$. Looking at the last sentence of the proof of Proposition \refgen{r1b_bound}, we have that  conditional on $E \cap E_{\mathrm{E}\ref*{general-eq:final_concent}} \cap \bigcap_{s=0}^{s_e} E_{\mathrm{L}\ref{parameterization_assum3}}(C_{K_{\mathrm{opt}}(\hat{\theta}_s)}^{\hat{\theta}_s}, W_s)$,
     \begin{equation}
    \sum_{s=0}^{s_e}  T_sJ(\theta^*,C^{\hat{\theta}_s}_{K_{\mathrm{opt}}(\hat{\theta}_s)}, T_s, \hat{x}_{   T_s}, W_s) - \sum_{s=0}^{s_e} \E\left[ T_sJ(\theta^*,C^{\hat{\theta}_s}_{K_{\mathrm{opt}}(\hat{\theta}_s)}, T_s,0, W_s)\cond \hat{\theta}_s  \right] \le  \tilde{O}_T(\sqrt{T}).
    \end{equation}
  Furthermore, because by construction $\P(E_{\mathrm{L}\ref{parameterization_assum3}}(C_{K_{\mathrm{opt}}(\hat{\theta}_s)}^{\hat{\theta}_s}, W_s)) = 1-o_T(1/T^{10})$, we have by a union bound that $\P(E_{\mathrm{E}\ref*{general-eq:final_concent}} \cap \bigcap_{s=0}^{s_e} E_{\mathrm{L}\ref{parameterization_assum3}}(C_{K_{\mathrm{opt}}(\hat{\theta}_s)}^{\hat{\theta}_s}, W_s)) = 1- o_T(1/T)$. Therefore, we can take $E_{\mathrm{P}\ref{r1b_bound_trunc}} = E_{\mathrm{E}\ref*{general-eq:final_concent}} \cap \bigcap_{s=0}^{s_e} E_{\mathrm{L}\ref{parameterization_assum3}}(C_{K_{\mathrm{opt}}(\hat{\theta}_s)}^{\hat{\theta}_s}, W_s)$ to get the desired result of Proposition \ref{r1b_bound_trunc}.

\subsection{Proof of Proposition \ref{non_optimal_controller_trunc}}\label{proof:non_optimal_controller_trunc}

Because the event $E$ and $E_3$ are defined equivalently to the events in Appendix \refgen{app:suff_large_noise_case}, this proof is exactly identical to the proof of Proposition \refgen{non_optimal_controller_suff} with the events and variables with respect to Algorithm \ref{alg:cap3} in this paper instead of Algorithm \refgen{alg:cap_large}.

\subsection{Proof of Proposition \ref{enforcing_safety_trunc}}\label{proof:enforcing_safety_trunc}
Because the events $E$ and $E_3$ are defined analogously to the events in Appendix \refgen{app:suff_large_noise_case}, this proof is very similar to the proof of Proposition \refgen{enforcing_safety_suff} with the events and variables with respect to Algorithm \ref{alg:cap3} of this paper instead of Algorithm \refgen{alg:cap_large}. Other than this redefining of events and variables, there are just two differences.

The first difference between Proposition \ref{enforcing_safety_trunc} of this paper and Proposition \refgen{enforcing_safety_suff} is that the subscript on the controller is $K_{\mathrm{opt}}(\hat{\theta}_s)$ rather than $K_{\mathrm{opt}}(\hat{\theta}_s, T_s)$. The proof of Proposition \refgen{enforcing_safety_suff} follows the proof of Proposition \refgen{enforcing_safety} and analogous versions of Lemma \refgen{offbyepsiloncontrol_propproof}, Lemma \refgen{bound_on_cont_diff_propproof}, and Lemma \refgen{offbyepsiloncontrol}. These lemmas all hold when the controller $C_{K_{\mathrm{opt}(\hat{\theta}_s, T_s)}}^{\hat{\theta}_s}$ is replaced with $C_K^{\hat{\theta}_s}$ for any $K \in [K_\mathrm{L}^{\hat{\theta}_s}, K_{\mathrm{U}}^{\hat{\theta}_s}]$ (because the proofs do not depend on the value of $K$). Therefore, analogous versions of these three lemmas hold for Algorithm \ref{alg:cap3} with $K_{\mathrm{opt}}(\hat{\theta}_s, T_s)$ replaced with $K_{\mathrm{opt}}(\hat{\theta}_s)$.

The second difference is that Proposition \refgen{enforcing_safety_suff} shows a bound that holds with high probability conditional on $E \cap E_3$, while Proposition \ref{enforcing_safety_trunc}'s bound holds conditional on $E \cap E_3 \cap E_{\mathrm{P}\ref{enforcing_safety_trunc}}$. Examining the proof of Proposition \refgen{enforcing_safety} (which is the same as the proof of Proposition \refgen{enforcing_safety_suff}), the high probability event comes from Lemma \refgen{offbyepsiloncontrol_propproof}, and that high probability event comes from Lemma \refgen{offbyepsiloncontrol}. Looking at the proof of Lemma \refgen{offbyepsiloncontrol}, the final result is proven conditional on event $E$ with conditional probability $1-o_T(1/T^{9})$. However, this ``with conditional probability" is coming from the event $\bigcap_{s=0}^{s_e} E_{\mathrm{L}\ref{parameterization_assum3}}(C_{K_{\mathrm{opt}}(\hat{\theta}_s, T_s)}^{\hat{\theta}_s}, W_s)$. Therefore, by Equation \eqrefgen{eq:lemma_16_eq} and the last sentence in the proof of Lemma \refgen{offbyepsiloncontrol}, for Algorithm \ref{alg:cap3},conditional on $E \cap \bigcap_{s=0}^{s_e} E_{\mathrm{L}\ref{parameterization_assum3}}(C_{K_{\mathrm{opt}}(\hat{\theta}_s)}^{\hat{\theta}_s}, W_s)$, for all $s$,
    \begin{align*}
         &|T_s \cdot J(\theta^*, C_{K_{\mathrm{opt}}(\hat{\theta}_s)}^{\hat{\theta}_s}, T_s, x_{   T_s}', W_s)  - T_s \cdot J(\theta^*, C^{\mathrm{alg}}_s, T_s, x_{   T_s}', W_s)| \\
         &= \tilde{O}_T\left(\sum_{i=0}^{T_s-1} |C_s^{\mathrm{alg}}(x_{   T_s + i}') - C_{K_{\mathrm{opt}}(\hat{\theta}_s)}^{\hat{\theta}_s}(x_{   T_s + i}')| + T_s\epsilon_s\right). 
    \end{align*}
Note that we replaced $K_{\mathrm{opt}}(\hat{\theta}_s, T_s)$ with $K_{\mathrm{opt}}(\hat{\theta}_s)$ for reasons discussed in the previous paragraph. Taking $E_{\mathrm{P}\ref{enforcing_safety_trunc}} = \bigcap_{s=0}^{s_e} E_{\mathrm{L}\ref{parameterization_assum3}}(C_{K_{\mathrm{opt}}(\hat{\theta}_s)}^{\hat{\theta}_s}, W_s)$ gives the desired result because by a union round and Assumption \ref{parameterization_assum3}, we have $\P\left(\bigcap_{s=0}^{s_e} E_{\mathrm{L}\ref{parameterization_assum3}}(C_{K_{\mathrm{opt}}(\hat{\theta}_s)}^{\hat{\theta}_s}, W_s)\right) = 1-o_T(1/T)$.

\subsection{Proof of Proposition \ref{warmup_firstcase}}\label{proof:warmup_firstcase}

The proof of Proposition \ref{warmup_firstcase} follows exactly the same as the proof of Proposition \refgen{warm_up_regret_large}. This is because the controller of Algorithm \ref{alg:cap3} is safe for dynamics $\theta^*$ under event $E$, and the result therefore follows directly.

\newpage

\section{Proofs for Appendix \ref{sec:trunc_lin_sqrtt_proofs}}

\subsection{Proof of Lemma \ref{boundary_uncertainty_cont_b}}\label{sec:proof_of_boundary_uncertainty_cont_b}

By Lemmas \ref{parameterization_assum2} and \ref{parameterization_assum3}, the class of truncated linear controllers satisfy all of the assumptions of Lemma \refgen{boundary_uncertainty_cont}. Therefore, the proof of Lemma \ref{boundary_uncertainty_cont_b} follows exactly as the proof of Lemma \refgen{boundary_uncertainty_cont}, except for Algorithm \ref{alg:cap3} from this paper instead of Algorithm \refgen{alg:cap} and with the analogous definition of event $E$.

\subsection{Proof of Lemma \ref{cswitch_prop}}\label{sec:proof_of_cswitch_prop}

\begin{proof}
    The following lemma shows that $F_{\mathrm{opt}}(\theta^*)$ and $F_{\mathrm{opt}}(\hat{\theta}_{\mathrm{wu}})$ are similar under event $E$.
    
    \begin{lemma}\label{close_J}
        Under Assumptions \ref{assum:constraints}--\ref{assum:initial} and \ref{assum_thm4}, conditional on event $E_2^0$, there exists $c_{\mathrm{L}\ref{close_J}} = \tilde{O}_T(1)$ such that for sufficiently large $T$,
        \[
            |F_{\mathrm{opt}}(\theta^*) - F_{\mathrm{opt}}(\hat{\theta}_{\mathrm{wu}})| \le c_{\mathrm{L}\ref{close_J}}T^{-1/4}.
        \]
    \end{lemma}
    The proof of Lemma \ref{close_J} can be found in Appendix \ref{sec:proof_of_close_J}.

    Conditional on $E$ (because $E \subseteq E_2^0$), we have that $\norm{\hat{\theta}_{\mathrm{wu}}-\theta^*}_{\infty} \le 2\epsilon_0 \le 2c_{\mathrm{L}\ref{initial_uncertainty_trunc}}T^{-1/4}$.  This combined with Lemma \ref{close_J} implies that there exists $c_{\mathrm{E}\ref{eq:Fhat_approx}} = \tilde{O}_T(1)$ such that under event $E$ for sufficiently large $T$, 
    \begin{align*}
        \hat{a} - \hat{b}F_{\mathrm{opt}}(\hat{\theta}_{\mathrm{wu}}) &\le a^* - b^*F_{\mathrm{opt}}(\theta^*) + c_{\mathrm{E}\ref{eq:Fhat_approx}}T^{-1/4}. \numberthis \label{eq:Fhat_approx}
    \end{align*}
   Now we will proceed with a proof by contradiction of Lemma \ref{cswitch_prop}. Assume  event $\neg E_{\mathrm{E}\ref{eq:first_case}} \cap E$ holds and $\epsilon^* \le 0$, the latter of which implies
   \begin{equation}\label{eq:contrapos}
    \frac{D_U}{a^*-b^*K_{\mathrm{opt}}(\theta^*)} - D_U \ge \bar{w},
    \end{equation}
    which in turn implies that $K_{\mathrm{opt}}(\theta^*) \ge K_{D_U}^{\theta^*}$ (recall $K_{D_U}^{\theta^*}$ was defined in Definition \ref{def:Kdu_theta}). A key result is the following relationship between $K_{\mathrm{opt}}(\theta^*)$ and $F_{\mathrm{opt}}(\theta^*)$.
    \begin{lemma}\label{K_equals_J}
        Under Assumptions \ref{assum:constraints}--\ref{assum:initial} and \ref{assum_thm4}, for any $\theta \in \Theta$,  if $K_{\mathrm{opt}}(\theta) \ge K_{D_U}^\theta$, then $F_{\mathrm{opt}}(\theta) \ge K_{D_U}^\theta$.
    \end{lemma}
    The proof of Lemma \ref{K_equals_J} can be found in Appendix \ref{sec:proof_of_K_equals_J}.

    We also will need the following result.
    \begin{lemma}\label{j_bounded_from_0}
    Under Assumptions \ref{assum:constraints}--\ref{assum:initial} and \ref{assum_thm4}, there exists $c_{\mathrm{L}\ref{j_bounded_from_0}}= O_T(1)$ such that $c_{\mathrm{L}\ref{j_bounded_from_0}} > 0$ and for all $\theta \in \Theta$,
    \[
        1 - c_{\mathrm{L}\ref{j_bounded_from_0}} > a - bF_{\mathrm{opt}}(\theta) \ge c_{\mathrm{L}\ref{j_bounded_from_0}},
    \]
    \[
        a-bK_{\mathrm{opt}}(\theta) \ge c_{\mathrm{L}\ref{j_bounded_from_0}}.
    \]
\end{lemma}
    The proof of Lemma \ref{j_bounded_from_0} can be found in Appendix \ref{sec:proof_of_j_bounded_from_0}.

    Lemma \ref{K_equals_J} combined with Equation \eqref{eq:contrapos} give that $F_{\mathrm{opt}}(\theta^*) \ge K_{D_U}^{\theta^*}$, or equivalently that $\bar{w} + D_U - \frac{D_U}{a^*-b^*F_{\mathrm{opt}}(\theta^*)} \le 0$. Therefore, we have that for sufficiently large $T$ under event $\neg E_{\mathrm{E}\ref{eq:first_case}} \cap E$,
    \begin{align*}
       &\bar{w} +D_U - \frac{D_U}{\hat{a}-\hat{b}F_{\mathrm{opt}}(\hat{\theta}_{\mathrm{wu}})} \\
       &\le \bar{w} +D_U - \frac{D_U}{a^* - b^*F_{\mathrm{opt}}(\theta^*) + c_{\mathrm{E}\ref{eq:Fhat_approx}}T^{-1/4}}  && \text{Equation \eqref{eq:Fhat_approx}} \\
       &= \frac{c_{\mathrm{E}\ref{eq:Fhat_approx}}T^{-1/4}D_U}{(a^* - b^*F_{\mathrm{opt}}(\theta^*))(a^* - b^*F_{\mathrm{opt}}(\theta^*) + c_{\mathrm{E}\ref{eq:Fhat_approx}}T^{-1/4})} +\bar{w} +D_U - \frac{D_U}{a^*-b^*F_{\mathrm{opt}}(\theta^*)}  \\
       &\le \frac{c_{\mathrm{E}\ref{eq:Fhat_approx}}T^{-1/4} D_U}{(a^* - b^*F_{\mathrm{opt}}(\theta^*))(a^* - b^*F_{\mathrm{opt}}(\theta^*) + c_{\mathrm{E}\ref{eq:Fhat_approx}}T^{-1/4})}
     && \text{Lemma \ref{K_equals_J}, Eq \eqref{eq:contrapos}}\\
       &\le \left(\frac{c_{\mathrm{E}\ref{eq:Fhat_approx}}D_U}{c_{\mathrm{L}\ref{j_bounded_from_0}}^2}\right)T^{-1/4}  && \text{ Lemma \ref{j_bounded_from_0}}\\
       &= C_{\mathrm{switch}}T^{-1/4}.
    \end{align*}
    However, this contradicts event $\neg E_{\mathrm{E}\ref{eq:first_case}}$ and therefore we have a contradiction. This implies the desired result that if $\neg E_{\mathrm{E}\ref{eq:first_case}} \cap E$ holds, then $\epsilon^* > 0$.

\end{proof}

\subsection{Proof of Lemma \ref{sufficiently_many_boundaries}}\label{sec:proof_of_sufficiently_many_boundaries}
    \begin{proof}
        Define the event $E_2^s := \left\{\norm{\hat{\theta}_s^{\mathrm{pre}} - \theta^*}_{\infty} \le \epsilon_s = \tilde{O}_T(T^{-1/4})\right\}$. Define $G_i = (x_0,u_0,...,x_{i-1}, u_{i-1})$ and define
        \[
            S''_t = \Big\{i< t : u_i = u_i^{\mathrm{safeU}} \text{ and } \P(u_i = u_i^{\mathrm{safeU}} \mid G_i) \ge  \P_{w \sim \mathcal{D}}(w \ge \bar{w} - 3\epsilon^*/8) \Big\}.
        \]
        
        \begin{lemma}\label{eventually_go_to_boundary}
         Under Assumptions \ref{assum:constraints}--\ref{assum:initial} and \ref{assum_thm4} there exists a constant $p_\epsilon$ such that the following holds. For sufficiently large $T$ and any $s \in [0:s_e-1]$ and any $T_s \le j < T_{s+1} - \lceil \log(T)\rceil $, there exists an event $X_j$ that depends on $\{w_t\}_{t=j}^{j+\lceil \log(T) \rceil -1}$ such that $\P(X_j) \ge p_\epsilon$ and such that conditional on event $X_j \cap E_2^s \cap \neg E_{\mathrm{E}\ref{eq:first_case}}$, there exists an $\ell \in [j: j+\lceil \log(T) \rceil )$ such that $\ell \in S''_{T_{s+1}}$.

    \end{lemma}
        The proof of Lemma \ref{eventually_go_to_boundary} can be found in Appendix \ref{sec:proof_of_eventually_go_to_boundary}.

    Define 
    \[
        \mathcal{E}^s := \left\{\sum_{\ell=0}^{\lfloor T_s/\lfloor\log(T)\rfloor \rfloor - 1 } 1_{X_{T_s + \ell \lfloor\log(T)\rfloor}} \ge p_\epsilon \left\lfloor\frac{T_s}{\lfloor \log(T) \rfloor} \right \rfloor - \sqrt{ \left\lfloor\frac{T_s}{\lfloor \log(T) \rfloor} \right \rfloor}\log(T)\right\}.
    \]
    Note that $\sum_{\ell=0}^{k } \left(1_{X_{T_s + \ell \lfloor\log(T)\rfloor}} - p_\epsilon \right)$ is a submartingale. Therefore, by the Azuma--Hoeffding inequality, we have that $\P(\mathcal{E}^s) = 1- o_T(1/T^2)$. Define $\mathcal{E} = \cap_{s =0}^{s_e - 1} \mathcal{E}^s$. Then by a union bound $\P(\mathcal{E})  = 1-o_T(1/T)$.

    Conditional on $\mathcal{E} \cap E \cap \neg E_{\mathrm{E}\ref{eq:first_case}}$, we have that 
    \begin{align*}
        |S''_{T_{s+1}}| &\ge \sum_{\ell= 0}^{ \lfloor T_s/\lfloor\log(T)\rfloor \rfloor - 1 } 1_{X_{T_s + \ell \lfloor\log(T)\rfloor}} \\
        &\ge p_\epsilon \left\lfloor\frac{T_s}{\lfloor \log(T) \rfloor} \right \rfloor - \sqrt{ \left\lfloor\frac{T_s}{\lfloor \log(T) \rfloor} \right \rfloor}\log(T) && \text{Event $\mathcal{E}$}\\
        &\ge \frac{p_\epsilon T_s}{2\log(T)} - \sqrt{ \left\lfloor\frac{T_s}{\lfloor \log(T) \rfloor} \right \rfloor}\log(T) \\
        &\ge \frac{p_{\epsilon}}{4\log(T)} \cdot T_s && \text{Suff. large $T$}\\
        &= \frac{p_{\epsilon}}{8\log(T)} \cdot T_{s+1}. && \text{$T_{s+1} = 2T_s$} \numberthis \label{eq:STs}
    \end{align*}
    The following lemma is the same as Lemma \refgen{lemma:converting_to_s_prime}. The proof is the same as the proof of that lemma, as the proof of Lemma \refgen{lemma:converting_to_s_prime} does not depend on the algorithm and only uses that $\P(E) = 1 - o_T(1/T^2)$.
\begin{lemma}\label{lemma:converting_to_s_prime}
                Using the same notation and assumptions as in the proof of Lemma \ref{sufficiently_many_boundaries}, for any constant $c < 1$,
                \[
                    \P \Big(\forall i \in [0:t-1], \P(E \mid G_i) \ge c \Big) = 1- o_T(1/T).
                \]
            \end{lemma}
    Define $E_{\mathrm{L}\ref{lemma:converting_to_s_prime}} = \{\forall i \in [T_0:T-1], \P(E \mid G_i) \ge  1 - \frac{\P_{w \sim \mathcal{D}}(w \ge \bar{w} - 3\epsilon^*/8)}{2} \}$. 
    
    By Lemma \ref{lemma:converting_to_s_prime}, $\P (E_{\mathrm{L}\ref{lemma:converting_to_s_prime}}) = 1- o_T(1/T)$. For any $i\in [T_0:T-1]$, conditional on $E_{\mathrm{L}\ref{lemma:converting_to_s_prime}} \cap \{\P(u_i = u_i^{\mathrm{safeU}} \mid G_i) \ge \P_{w \sim \mathcal{D}}(w \ge \bar{w} - 3\epsilon^*/8)\}$, by the law of total probability 
    \begin{align*}
       \P(u_i = u_i^{\mathrm{safeU}} \mid G_i) &=  \P(u_i = u_i^{\mathrm{safeU}} \mid G_i, E)\P(E \mid G_i ) +  \P(u_i = u_i^{\mathrm{safeU}} \mid G_i, \neg E)\P(\neg E \mid G_i)  \\
       &\le  \P(u_i = u_i^{\mathrm{safeU}} \mid G_i, E)\P(E \mid G_i) + \frac{\P_{w \sim \mathcal{D}}(w \ge \bar{w} - 3\epsilon^*/8)}{2}.
    \end{align*}
    Rearranging terms gives
    \begin{align*}
        \P(u_i = u_i^{\mathrm{safeU}} \mid G_i, E) &\ge \frac{ \P(u_i = u_i^{\mathrm{safeU}} \mid G_i) - \frac{\P_{w \sim \mathcal{D}}(w \ge \bar{w} - 3\epsilon^*/8)}{2}}{\P(E \mid G_i) } \\
        &\ge  \P(u_i = u_i^{\mathrm{safeU}} \mid G_i) - \frac{\P_{w \sim \mathcal{D}}(w \ge \bar{w} - 3\epsilon^*/8)}{2} \\
        &\ge \frac{\P_{w \sim \mathcal{D}}(w \ge \bar{w} - 3\epsilon^*/8)}{2}.
    \end{align*}
    Therefore we have shown that conditional on $E_{\mathrm{L}\ref{lemma:converting_to_s_prime}} \cap \{\P(u_i = u_i^{\mathrm{safeU}} \mid G_i) \ge \P_{w \sim \mathcal{D}}(w \ge \bar{w} - 3\epsilon^*/8)\}$, we also have $\P(u_i = u_i^{\mathrm{safeU}} \mid G_i, E) \ge \frac{\P_{w \sim \mathcal{D}}(w \ge \bar{w} - 3\epsilon^*/8)}{2}$. This implies that conditional on $E_{\mathrm{L}\ref{lemma:converting_to_s_prime}}$, for all $t \in [0:T]$,
    \[
        S''_{t} \subseteq S'_{t}.
    \]
    Combining this with Equation  \eqref{eq:STs}, conditional on $E_{\mathrm{L}\ref{lemma:converting_to_s_prime}} \cap \mathcal{E} \cap E \cap \neg E_{\mathrm{E}\ref{eq:first_case}}$,
    \[
        \max_{s \in [1:s_e]} \frac{T_s}{|S'_{T_s}|} \le \frac{8\log(T)}{p_\epsilon} = \tilde{O}_T(1).
    \]
    We therefore take $E_{\mathrm{L}\ref{sufficiently_many_boundaries}} =E_{\mathrm{L}\ref{lemma:converting_to_s_prime}} \cap \mathcal{E} \cap E $ to get the desired result because $\P(E_{\mathrm{L}\ref{sufficiently_many_boundaries}}) = 1-o_T(1/T)$ by a union bound.
    \end{proof}

\subsection{Proof of Lemma \ref{split}}\label{sec:proof_of_split}
\begin{proof}
Let $x_T$ be the state after starting at $x_0 = 0$ and using the controller $C_K^\theta$ for $T$ steps under dynamics $\theta$.  Therefore, because $C_K^\theta$ is safe for dynamics $\theta$, we must have that $|x_T| \le \max(D_U, |D_L|)  + \bar{w} \le 2\log^2(T)$ for sufficiently large $T$. Therefore, there must exist an $L \le 2\log^2(T)$ such that $\P(|x| \ge L)\E[x^2 \mid |x| \ge L] = o_T(1/T^{11})$. Define $W'= \{w_i\}_{i=0}^T$. We can apply Lemma \ref{offbyepsilon_exp} in the sixth line below to get that 
    \begin{align*}
        &\left| \bar{J}(\theta, C_K^\theta, 2T) - \bar{J}(\theta, C_K^\theta,T)\right|\\
        &= \left| \frac{T\cdot \bar{J}(\theta, C_K^\theta, T) + T\cdot \E\left[\bar{J}(\theta,C_K^\theta,T, x_T)\right]}{2T} - \bar{J}(\theta, C_K^\theta,T)\right| \\
        &=\left| \frac{ \E\left[\bar{J}(\theta,C_K^\theta,T, x_T)\right]}{2} - \frac{1}{2}\bar{J}(\theta, C_K^\theta,T)\right| \\
        &=\frac{1}{2T}\left| \E\left[T\bar{J}(\theta,C_K^\theta,T, x_T)\right] - T\bar{J}(\theta, C_K^\theta,T)\right|  \\
        &=\frac{1}{2T}\Big| \E\big[TJ(\theta,C_K^\theta,T, x_T, W') - TJ(\theta, C_K^\theta,T,0,  W')\big] \Big|  \\
        &\le \frac{1}{T}\tilde{O}_T\left(\E[|x_T|] + 0 + \frac{1}{T^2}\right) && \text{Lemma \ref{offbyepsilon_exp}} \\
        &\le \tilde{O}_T\left(\frac{1}{T}\right). && \text{$|x_T| \le \norm{D}_{\infty} + \bar{w} = \tilde{O}_T(1)$}
    \end{align*}
    The last line follows from the fact that $C_K^\theta$ is safe for dynamics $\theta$. Finally, we have that
    \begin{align*}
         |\bar{J}(\theta, C_K^{\theta}, T) - \bar{J}(\theta, C_K^{\theta})| &=  \left| \sum_{i=0}^{\infty}  \bar{J}(\theta, C_K^{\theta}, 2^{i}T) -  \bar{J}(\theta, C_K^{\theta}, 2^{i+1}T) \right| \\
         &\le \sum_{i=0}^{\infty}  \left| \bar{J}(\theta, C_K^{\theta}, 2^{i}T) -  \bar{J}(\theta, C_K^{\theta}, 2^{i+1}T) \right| \\
         &= \sum_{i=0}^{\infty} \tilde{O}_T \left(\frac{1}{T2^i}\right) \\
         &= \tilde{O}_T \left(\frac{1}{T}\right).
    \end{align*}
\end{proof}

\subsection{Proof of Lemma \ref{close_J}}\label{sec:proof_of_close_J}  
\begin{proof}
     By Lemma \ref{lemma:convexity}, the optimal unconstrained controller for dynamics $\theta$ is $C_{F_{\mathrm{opt}}(\theta)}^{\mathrm{unc}}$, where 
    \begin{equation}\label{eq:unconst_cont_cost}
        F_{\mathrm{opt}}(\theta) = \arg\min_F T \cdot \bar{J}(\theta, C_F^{\mathrm{unc}}) = \arg\min_F  \cdot \frac{q+rF^2}{1-(a-bF)^2}.
    \end{equation}
    We show in the proof of Lemma \ref{j_bounded_from_0} that
     \[
        F_{\mathrm{opt}}(\theta) = \frac{a^2r - b^2q - r + \sqrt{(b^2q + r - a^2r)^2 + 4a^2b^2qr}}{2abr}.
    \]
    Note that this is a differentiable function in both $a$ and $b$ for $\theta \in \Theta$. Under event $E_2^0$, $\norm{\theta^* - \hat{\theta}_{\mathrm{wu}}}_{\infty} =  \tilde{O}_T(T^{-1/4})$ where $\hat{\theta}_{\mathrm{wu}}$ is the estimate from Line \ref{line:thetawu} of Algorithm \ref{alg:cap3}. Therefore, a first order Taylor expansion of $F_{\mathrm{opt}}(\theta)$ around $\theta = \theta^*$ gives that for sufficiently large $T$, $|F_{\mathrm{opt}}(\theta^*) - F_{\mathrm{opt}}(\hat{\theta}_{\mathrm{wu}})| = O_T(\norm{\theta^* - \hat{\theta}_{\mathrm{wu}}}_{\infty}) = \tilde{O}_T(T^{-1/4}) = c_{\mathrm{L}\ref{close_J}}T^{-1/4}$ for some $c_{\mathrm{L}\ref{close_J}} = \tilde{O}_T(1)$.

\end{proof}

\subsection{Proof of Lemma \ref{K_equals_J}}\label{sec:proof_of_K_equals_J} 
    \begin{proof}
      We will prove the contrapositive, which is that if $F_{\mathrm{opt}}(\theta) < K_{D_U}^\theta$, then $K_{\mathrm{opt}}(\theta) < K_{D_U}^\theta$.

      The first tool we need is the following result about $F_{\mathrm{opt}}(\theta)$.
    \begin{lemma}\label{lemma:convexity}
        For any $\theta \in \Theta$ and $K \in (\frac{a-1}{b}, \frac{a}{b}]$,
        \[
        \bar{J}(\theta, C_K^{\mathrm{unc}}) = \lim_{T \rightarrow \infty}   \bar{J}(\theta, C_K^{\mathrm{unc}},T)  =  \frac{\sigma_{\mathcal{D}}^2(q + rK^2)}{1-(a-bK)^2}.
        \]
        This function is convex and twice differentiable for $K \in (\frac{a-1}{b}, \frac{a}{b}]$. Furthermore, if $1-(a-bK) > 0$, then $\left| \frac{d}{dK} \bar{J}(\theta, C_K^{\mathrm{unc}})\right|$ and $\left| \frac{d^2}{dK^2} \bar{J}(\theta, C_K^{\mathrm{unc}}) \right|$ are finite and $\frac{d^2}{dK^2} \bar{J}(\theta, C_K^{\mathrm{unc}}) > 0$.

        Finally, if $K = \frac{a-1}{b}$, then $\bar{J}(\theta, C_K^{\mathrm{unc}}) = \infty$.
    \end{lemma}
    The proof of Lemma \ref{lemma:convexity} can be found in Appendix \ref{sec:proof_of_lemma:convexity}.

        Lemma \ref{lemma:convexity} implies that the function $\bar{J}(\theta, C_K^{\mathrm{unc}})$ has a unique local minimum ($F_{\mathrm{opt}}(\theta)$) and is convex. Therefore, if $F_{\mathrm{opt}}(\theta) < K_{D_U}^\theta$, then for any $K' > K_{D_U}^\theta$,
        \begin{equation}\label{eq:unconstKD}
            \bar{J}(\theta, C^{\mathrm{unc}}_{K_{D_U}^\theta}) \le \bar{J}(\theta, C^{\mathrm{unc}}_{K'}).
        \end{equation}
        For any $K' \ge K_{D_U}^\theta$, the unconstrained and constrained controllers are the same, i.e. $C^{\mathrm{unc}}_{K'} = C^{\theta}_{K'}$. This is because for $K' \ge K_{D_U}^\theta$ the unconstrained controller will always satisfy the state constraints because we assumed WLOG that $D_U \le |D_L|$. This implies by Equation \eqref{eq:unconstKD} that for any $K' > K_{D_U}^\theta$,
        \[
            \bar{J}(\theta, C^{\theta}_{K_{D_U}^\theta}) \le \bar{J}(\theta, C^{\theta}_{K'}).
        \]
        Therefore, to prove that $K_{\mathrm{opt}}(\theta) < K_{D_U}^\theta$ it is sufficient to find some $K' < K_{D_U}^\theta$ such that
        \begin{equation}\label{eq:compKprime1}
            \bar{J}(\theta, C^{\theta}_{K_{D_U}^\theta}) > \bar{J}(\theta, C^{\theta}_{K'}).
        \end{equation}
         Let $K' = K_{D_U}^\theta - \epsilon$, where
        \begin{equation}\label{eq:eps_def}
       0 <  \epsilon \le   \min \left(\frac{4B_P}{(\bar{w} + D_U)^2}, \frac{\min(\bar{w}, D_U)/2}{  (\bar{w} + D_U)}\right) .
        \end{equation}
        We will show that $\bar{J}(\theta, C^{\theta}_{K'}) < \bar{J}(\theta, C^{\theta}_{K_{D_U}^\theta})$ which proves the desired contrapositive result.
        
       Because $a-bK_{D_U}^\theta = \frac{D_U}{D_U + \bar{w}} = 1 - \frac{\bar{w}}{D_U + \bar{w}}$, by Lemma \ref{lemma:convexity} the function $\bar{J}(\theta, C_K^{\mathrm{unc}})$ has a finite derivative at $K = K_{D_U}^\theta$. Furthermore, if $F_{\mathrm{opt}}(\theta) < K_{D_U}^\theta$, then Lemma \ref{lemma:convexity} implies that the derivative of  $\bar{J}(\theta, C_F^{\mathrm{unc}})$ is positive at $K = K_{D_U}^\theta$. Therefore, we can take a first order Taylor expansion around the point $K= K_{D_U}^\theta$ to get that for sufficiently small $\epsilon$,
        \begin{equation}\label{eq:compKprime2}
            \bar{J}(\theta, C^{\mathrm{unc}}_{K'}) - \bar{J}(\theta, C^{\mathrm{unc}}_{K_{D_U}^\theta}) \le -\Omega_T(\epsilon).
        \end{equation}
        Because $C^{\mathrm{unc}}_{K_{D_U}^\theta} = C^{\theta}_{K_{D_U}^\theta}$, Equation \eqref{eq:compKprime2} implies that
        \begin{equation}\label{eq:compKprime3}
            \bar{J}(\theta, C^{\mathrm{unc}}_{K'}) - \bar{J}(\theta, C^{\theta}_{K_{D_U}^\theta}) \le -\Omega_T(\epsilon).
        \end{equation}
        Note that in Equations \eqref{eq:compKprime2} and \eqref{eq:compKprime3}, the LHS is not a function of $T$. We use the notation $-\Omega_T(\epsilon)$ to indicate that the LHS is upper bounded by $-c\epsilon$ for some constant $c$.
        
        Now we will compare the cost of $C^{\mathrm{unc}}_{K'}$ and $C^{\theta}_{K'}$ using the following lemma. Note that this lemma is stated very generally so that it can also be used in future results.
        \begin{lemma}\label{close_to_KDU}
             For $\theta,\hat{\theta}_{\mathrm{L}\ref{close_to_KDU}} \in \Theta$, suppose $\beta \le \frac{1}{\log^2(T)}$ satisfies that $\theta \in \hat{\theta}_{\mathrm{L}\ref{close_to_KDU}}  \pm \beta$. Also, suppose $K'$ satisfies $K_{D_U}^\theta - K'\le \epsilon$ for some $\epsilon > 0$. Furthermore, suppose 
        \begin{equation}\label{eq:def_of_upsilon}
       \upsilon := (b\epsilon+\beta+|K'|\beta) \le  \min \left(\frac{4B_P}{(\bar{w} + D_U)^2}, \frac{\min(\bar{w}, D_U)/2}{  (\bar{w} + D_U)}\right) 
        \end{equation}
             Define the controller $C$ as follows. For any $t$, define $v_t^{\mathrm{safeU}}$ as the largest $u$ such that for all $\theta' \in \hat{\theta}_{\mathrm{L}\ref{close_to_KDU}}  \pm \beta$,
            \[
                a'x_t + b'u \le D_U,
            \]
            and define $v_t^{\mathrm{safeL}}$ as the smallest $u$ such that for all $\theta' \in \hat{\theta}_{\mathrm{L}\ref{close_to_KDU}}  \pm \beta $,
            \[
                D_L \le a'x_t + b'u.
            \]
            Define the controller $C$ as 
            \[
                C(x_t) = \max\left(\min\left(C_{K'}^{\mathrm{unc}}(x_t), v_t^{\mathrm{safeU}}\right), v_t^{\mathrm{safeL}}\right).
            \]
            Let  $|x_0| \le \norm{D}_{\infty} + \bar{w}$. Then under Assumptions \ref{assum:constraints}--\ref{assum:initial} and \ref{assum_thm4},
            \begin{equation}\label{eq:lemma47first}
                |\bar{J}(\theta, C, x_0) - \bar{J}(\theta, C_{K'}^{\mathrm{unc}}, x_0)| \le O_T(\upsilon^2).
            \end{equation}
            Furthermore, with probability $1-o_T(1/T^2)$, for any $\tau \le T$,
            \begin{equation}\label{eq:close_to_KDU_second_eq}
                |J(\theta, C, \tau, x_0, W') - J(\theta, C_{K'}^{\mathrm{unc}}, \tau, x_0, W')| \le O_T\left(\upsilon\log(1/\upsilon)\left(\upsilon + \frac{\log(T)}{\sqrt{\tau}}\right)\right).
            \end{equation}
        \end{lemma}
        The proof of Lemma \ref{close_to_KDU} can be found in Appendix \ref{sec:proof_of_close_to_KDU}.

        We will use Lemma \ref{close_to_KDU} with  the $\epsilon$ defined in Equation \eqref{eq:eps_def}, $K' = K_{D_U}^\theta - \epsilon$, $\theta = \theta$, $\hat{\theta}_{\mathrm{L}\ref{close_to_KDU}} = \theta$, $x_0 = 0$, and $\beta = 0$. Choosing $\hat{\theta}_{\mathrm{L}\ref{close_to_KDU}} = \theta$ and $\beta = 0$ makes the $C$ in Lemma \ref{close_to_KDU} equivalent to a truncated linear controller. Then, Equation \eqref{eq:lemma47first} of Lemma \ref{close_to_KDU} gives that
        \begin{equation}\label{eq:compKprime4}
            |\bar{J}(\theta, C^{\theta}_{K'}) - \bar{J}(\theta, C_{K'}^{\mathrm{unc}})| \le O_T(\epsilon^2).
        \end{equation}
        Putting together Equations \eqref{eq:compKprime3} and \eqref{eq:compKprime4}, for small enough $\epsilon$ we have that
        \begin{align*}
            &\bar{J}(\theta, C^{\theta}_{K'}) - \bar{J}(\theta, C^{\theta}_{K_{D_U}^\theta})
            \\
            &=  \bar{J}(\theta, C^{\theta}_{K'})  - \bar{J}(\theta, C_{K'}^{\mathrm{unc}}) + \bar{J}(\theta, C_{K'}^{\mathrm{unc}}) - \bar{J}(\theta, C^{\theta}_{K_{D_U}^\theta}) \\
            &\le O_T\left(\epsilon^2 \right)-\Omega_T\left(\epsilon\right).  && \text{Equations \eqref{eq:compKprime3}, \eqref{eq:compKprime4}} \\
            &< 0. && \text{For small enough $\epsilon$}
        \end{align*}
    We have shown that $C_{K'}^{\theta}$ has lower cost than $C_{K_{D_U}^\theta}^{\theta}$, and therefore we can conclude that $K_{\mathrm{opt}}(\theta) < K_{D_U}^\theta$, proving the contrapositive and our desired result.
    \end{proof}

\subsection{Proof of Lemma \ref{j_bounded_from_0}}\label{sec:proof_of_j_bounded_from_0}
\begin{proof}
    By Lemma \ref{lemma:convexity}, $F_{\mathrm{opt}}(\theta)$ is the value of $K \in \left(\frac{a-1}{b}, \frac{a}{b} \right]$ that minimizes the function $\frac{q+rK^2}{1-(a-bK)^2}$ (note that we ignore the constant $\sigma_{\mathcal{D}}^2$ as this is a positive constant and does not change the minimization problem). Taking the derivative of this function and equating to $0$, we have that $F_{\mathrm{opt}}(\theta)$ is the solution to
    \[
        \frac{2Kr(1-(a-bK)^2) - 2b(a-bK)(q+rK^2)}{(1-(a-bK)^2)^2} = 0.
    \]
    Simplifying, we have
    \[
        abrK^2 + (b^2q + r - a^2r)K - abq = 0
    \]
    Applying the quadratic formula, we get that the positive root is
    \[
        F_{\mathrm{opt}}(\theta) = \frac{a^2r - b^2q - r + \sqrt{(b^2q + r - a^2r)^2 + 4a^2b^2qr}}{2abr}.
    \]
    We also observe that
    \[
        \left(a^2r + b^2q + r \right)^2 - \left(\sqrt{(b^2q + r - a^2r)^2 + 4a^2b^2qr}\right)^2 =  4a^2r^2,
    \]
    which implies that
    \begin{align*}
     &\left(a^2r + b^2q + r \right) - \left(\sqrt{(b^2q + r - a^2r)^2 + 4a^2b^2qr}\right)\\
     &= \frac{4a^2r^2}{ \left(a^2r + b^2q + r \right)+  \left(\sqrt{(b^2q + r - a^2r)^2 + 4a^2b^2qr}\right)}.
    \end{align*}
    Because $\underline{a} \ge a \ge \bar{a}$, $\underline{b} \ge b \ge \bar{b}$,  and $r > 0$, this implies that there exists a constant $c^{F1}_{\mathrm{L}\ref{j_bounded_from_0}} > 0$ such that 
    \begin{align*}
        \frac{a}{b} - F_{\mathrm{opt}}(\theta)  &= \frac{a^2r + b^2q + r - \sqrt{(b^2q + r - a^2r)^2 + 4a^2b^2qr}}{2abr} \\
        &= \frac{4a^2r^2}{2abr \left(a^2r + b^2q + r + \sqrt{(b^2q + r - a^2r)^2 + 4a^2b^2qr}\right) } \\
        &\ge \frac{4\underline{a}^2r^2}{2\bar{a}\bar{b}r \left(\bar{a}^2r + \bar{b}^2q + r + \sqrt{(\bar{b}^2q + r - \underline{a}^2r)^2 + 4\bar{a}^2\bar{b}^2qr}\right) } \\
        &:=  c^{F1}_{\mathrm{L}\ref{j_bounded_from_0}} \\
        &> 0. 
    \end{align*}
    Similarly, we have that
        \[
        \left(r(a-1)^2 + b^2q \right)^2 - \left(\sqrt{(b^2q + r - a^2r)^2 + 4a^2b^2qr}\right)^2 =  -4ar\left((a-1)^2r + b^2q\right).
    \]
    which implies that
    \begin{align*}
         &\left(r(a-1)^2 + b^2q \right) - \left(\sqrt{(b^2q + r - a^2r)^2 + 4a^2b^2qr}\right)\\
         &=  \frac{-4ar\left((a-1)^2r + b^2q\right)}{ \left(r(a-1)^2 + b^2q \right) + \left(\sqrt{(b^2q + r - a^2r)^2 + 4a^2b^2qr}\right)}.
    \end{align*}
    Because $a \ge \bar{a}$ and $r > 0$, this implies that there exists a constant $c^{F2}_{\mathrm{L}\ref{j_bounded_from_0}} > 0$ such that 
    \begin{align*}
        \frac{a-1}{b} - F_{\mathrm{opt}}(\theta)  &= \frac{r(a-1)^2 +  b^2q - \sqrt{(b^2q + r - a^2r)^2 + 4a^2b^2qr}}{2abr} \\
        &= \frac{-4ar\left((a-1)^2r + b^2q\right)}{2abr \left(r(a-1)^2 + b^2q  + \sqrt{(b^2q + r - a^2r)^2 + 4a^2b^2qr}\right) } \\
        &\le -c^{F2}_{\mathrm{L}\ref{j_bounded_from_0}} \\
        &< 0, 
    \end{align*}
    where the constant $C^{F2}_{\mathrm{L}\ref{j_bounded_from_0}}$ depends on $\bar{a}, \underline{a}, \bar{b}, \underline{b}$.
    Taking $c^F_{\mathrm{L}\ref{j_bounded_from_0}} = \min(c^{F1}_{\mathrm{L}\ref{j_bounded_from_0}}, c^{F2}_{\mathrm{L}\ref{j_bounded_from_0}})$, we have that
    \begin{equation}\label{eq:Fbound_from_0}
        c^F_{\mathrm{L}\ref{j_bounded_from_0}} < a - bF_{\mathrm{opt}}(\theta) < 1 -  c^F_{\mathrm{L}\ref{j_bounded_from_0}}.
    \end{equation}
    To bound $K_{\mathrm{opt}}(\theta, T)$ away from $0$ we need the following lemma:
    \begin{lemma}\label{J_implies_K} 
        Under Assumptions \ref{assum:constraints}--\ref{assum:initial}, for any $\theta \in \Theta$, if $F_{\mathrm{opt}}(\theta) \ge K_{D_U}^\theta$, then $K_{\mathrm{opt}}(\theta) = F_{\mathrm{opt}}(\theta)$.
    \end{lemma}
    \begin{proof}
        If $F_{\mathrm{opt}}(\theta) \ge K_{D_U}^\theta$, then $C_{F_{\mathrm{opt}}(\theta)}^{\mathrm{unc}} = C_{F_{\mathrm{opt}}(\theta)}^{\theta}$, i.e. the unconstrained linear controller for $F_{\mathrm{opt}}(\theta)$ is the same as the constrained linear controller for $F_{\mathrm{opt}}(\theta)$. Therefore, $C^{\mathrm{unc}}_{F_{\mathrm{opt}}(\theta)}$ is in the set of constrained controllers. Because the optimal unconstrained controller is linear \cite{anderson2007optimal}, $C^{\mathrm{unc}}_{F_{\mathrm{opt}}(\theta)}$ is the lowest cost unconstrained controller, and therefore it is also the lowest cost constrained controller.
    \end{proof}

   By Lemma \ref{J_implies_K} and the contrapositive of Lemma \ref{K_equals_J}, either $K_{\mathrm{opt}}(\theta) = F_{\mathrm{opt}}(\theta)$ or $K_{\mathrm{opt}}(\theta) < K_{D_U}^\theta$. By Equation \eqref{eq:Fbound_from_0} and the fact that $a-bK_{D_U}^\theta = \frac{D_U}{D_U + \bar{w}}$, we can conclude that
    \[
        a-bK_{\mathrm{opt}}(\theta) \ge \min\left(\frac{D_U}{D_U + \bar{w}}, c^F_{\mathrm{L}\ref{j_bounded_from_0}}\right) > 0.
    \]
    Therefore, taking $c_{\mathrm{L}\ref{j_bounded_from_0}} = \min\left(\frac{D_U}{D_U + \bar{w}}, c^F_{\mathrm{L}\ref{j_bounded_from_0}}\right)$ we have the desired result.
\end{proof}

\subsection{Proof of Lemma \ref{eventually_go_to_boundary}}\label{sec:proof_of_eventually_go_to_boundary}    
    \begin{proof}
    The structure of this proof is as follows. The bulk of the proof is split into two key lemmas. We then combine these two lemmas to show the desired result.
    Define
    \[
        \tau := \left \lceil 8\left(\frac{2+c_{\mathrm{L}\ref{j_bounded_from_0}}}{c_{\mathrm{L}\ref{j_bounded_from_0}}}\norm{D}_{\infty} + 2\bar{w}\right)/\epsilon^* \right \rceil,
    \]
    where $\epsilon^*$ is from Lemma \ref{cswitch_prop}.
    Now, we will define
    \[
        X_j := \left\{ \forall t \in [j: j+\tau], w_t \ge \bar{w} - \epsilon^*/4\right\}.
    \]
    Note that $\P(X_j) = \left(\P_{w \sim \mathcal{D}}(w \ge \bar{w}-\epsilon^*/4)\right)^{\tau+1} := p_\epsilon$, and for sufficiently large $T$, $\tau \le \lceil \log(T) \rceil$, therefore this $X_j$ has the desired properties.
    
    \begin{lemma}\label{lemma:first_keyequation}
            Using the assumptions and notation of Lemma \ref{eventually_go_to_boundary}, conditional on  $E_2^s \cap \neg E_{\mathrm{E}\ref{eq:first_case}} \cap X_j$, there exists an $\ell \in [j:j+\tau]$ such that $u_\ell = u_\ell^{\mathrm{safeU}}$.
        \end{lemma}
    \begin{proof}
        We will first show that conditional on $E_2^s \cap \neg E_{\mathrm{E}\ref{eq:first_case}}$, for any value of $x$ satisfying $D_L - \bar{w} \le x \le \frac{D_U}{a^*-b^*K_{\mathrm{opt}}(\hat{\theta}_s)}$, and for sufficiently large $T$, if $w \ge \bar{w} - \epsilon^*/4$, then
        \begin{equation}\label{eq:eventually2}
            (a^*-b^*K_{\mathrm{opt}}(\hat{\theta}_s))x + w \ge x+\frac{\epsilon^*}{8}.
        \end{equation}        
        Under event $E^s_2$, $\norm{\theta^* - \hat{\theta}_s}_{\infty} \le \tilde{O}_T(T^{-1/4})$, therefore under event $E^s_2$ we have the following results:
        \begin{align*}
            a^*-b^*K_{\mathrm{opt}}(\hat{\theta}_s) &\ge \hat{a}_s-\hat{b}_sK_{\mathrm{opt}}(\hat{\theta}_s) -  \tilde{O}_T(T^{-1/4}) && \text{$\norm{\theta^* - \hat{\theta}_s}_{\infty} \le \tilde{O}_T(T^{-1/4})$}\\
            &\ge  c_{\mathrm{L}\ref{j_bounded_from_0}} - \tilde{O}_T(T^{-1/4})&& \text{Lemma \ref{j_bounded_from_0}} \\
            &\ge   \frac{c_{\mathrm{L}\ref{j_bounded_from_0}}}{2} && \text{suff large $T$} \numberthis \label{eq:45_lower}
        \end{align*}
        and
        \begin{align*}
            a^*-b^*K_{\mathrm{opt}}(\hat{\theta}_s) &\le \hat{a}_s-\hat{b}_sK_{\mathrm{opt}}(\hat{\theta}_s) +  \tilde{O}_T(T^{-1/4}) && \text{$\norm{\theta^* - \hat{\theta}_s}_{\infty} \le \tilde{O}_T(T^{-1/4})$}\\
            &\le 1 + \tilde{O}_T(T^{-1/4}).&& \text{Lemma \ref{j_bounded_from_0}} \numberthis \label{eq:45_upper}
        \end{align*}        
        Equation \eqref{eq:45_lower} implies that for sufficiently large $T$, 
        \begin{equation}\label{eq:probspec}
            \frac{D_U}{a^*-b^*K_{\mathrm{opt}}(\hat{\theta}_s)} \le \frac{2D_U}{c_{\mathrm{L}\ref{j_bounded_from_0}}} = O_T(1).
        \end{equation}
    To prove Equation \eqref{eq:eventually2}, we will need the following result.
    \begin{lemma}\label{good_chosen_K}
          Under Assumptions \ref{assum:constraints}--\ref{assum:initial} and \ref{assum_thm4}, conditional on event $E_2^s \cap \neg  E_{\mathrm{E}\ref{eq:first_case}}$ and for sufficiently large $T$,
         \[
            \frac{D_U}{a^*-b^*K_{\mathrm{opt}}(\hat{\theta}_s)}\le D_U + \bar{w} - \epsilon^*/2.
         \]
    \end{lemma}
        The proof of Lemma \ref{good_chosen_K} can be found in Appendix \ref{sec:proof_of_good_chosen_K}.

        Conditional on event $E_2^s \cap \neg  E_{\mathrm{E}\ref{eq:first_case}}$, for sufficiently large $T$, and for any $D_L - \bar{w} \le x \le \frac{D_U}{a^*-b^*K_{\mathrm{opt}}(\hat{\theta}_s)}$,
        {\fontsize{10}{10}
        \begin{align*}
            &(a^*-b^*K_{\mathrm{opt}}(\hat{\theta}_s))x + \bar{w}-\epsilon^*/4 \\
            &= D_U+ \left(x - \frac{D_U}{a^*-b^*K_{\mathrm{opt}}(\hat{\theta}_s)}\right)(a^*-b^*K_{\mathrm{opt}}(\hat{\theta}_s)) + \bar{w} - \epsilon^*/2 + \epsilon^*/4  \\
            &\ge  \frac{D_U}{a^*-b^*K_{\mathrm{opt}}(\hat{\theta}_s)} +  \left(x - \frac{D_U}{a^*-b^*K_{\mathrm{opt}}(\hat{\theta}_s)}\right)(a^*-b^*K_{\mathrm{opt}}(\hat{\theta}_s)) + \epsilon^*/4 && \text{Lemma \ref{good_chosen_K}}\\
            &\ge  \frac{D_U}{a^*-b^*K_{\mathrm{opt}}(\hat{\theta}_s)} + \left(x - \frac{D_U}{a^*-b^*K_{\mathrm{opt}}(\hat{\theta}_s)}\right)(1 + \tilde{O}_T(T^{-1/4}))+ \epsilon^*/4 && \text{Eq \eqref{eq:45_upper}, $x \le \frac{D_U}{a^*-b^*K_{\mathrm{opt}}(\hat{\theta}_s)}$} \\
            &= x + \tilde{O}_T\left(T^{-1/4}\left(x - \frac{D_U}{a^*-b^*K_{\mathrm{opt}}(\hat{\theta}_s)}\right)\right)+ \epsilon^*/4  \\
            &\ge x - \tilde{O}_T\left(T^{-1/4}\left(|D_L| + \bar{w} + \frac{D_U}{a^*-b^*K_{\mathrm{opt}}(\hat{\theta}_s)}\right)\right)+ \epsilon^*/4 
 && \text{ $D_L - \bar{w} \le x \le \frac{D_U}{a^*-b^*K_{\mathrm{opt}}(\hat{\theta}_s)}$}\\
            &\ge x - \tilde{O}_T\left(T^{-1/4}\right)+ \epsilon^*/4 && \text{Eq \eqref{eq:probspec}, Assumption \ref{assum_thm4}}  \\
            &\ge x + \epsilon^*/8. && \text{For sufficiently large $T$}.
        \end{align*}
        }
        This in turn implies the statement containing Equation \eqref{eq:eventually2}.
        
        Recall that $u_i$ is the control at time $i$ of Algorithm \ref{alg:cap3} and $x'_{i}$ is the state of Algorithm \ref{alg:cap3} at time $i$. Under event $\neg E_{\mathrm{E}\ref{eq:first_case}}$, for any $i \in [T_s+1: T_{s+1}]$, if $u_{i-1} \ne u_{i-1}^{\mathrm{safeU}}$, then the control at time $i-1$ is either $u_{i-1} = -K_{\mathrm{opt}}(\hat{\theta}_s)x'_{i-1}$ or $u_{i-1} = u_{i-1}^{\mathrm{safeL}} \ge -K_{\mathrm{opt}}(\hat{\theta}_s)x'_{i-1}$. Therefore, under event $\neg E_{\mathrm{E}\ref{eq:first_case}}$, if $u_{i-1} \ne u_{i-1}^{\mathrm{safeU}}$ then
        \begin{equation}\label{eq:ui1}
                u_{i-1} \ge -K_{\mathrm{opt}}(\hat{\theta}_s)x'_{i-1}.
        \end{equation}
        Combining Equations \eqref{eq:eventually2} and \eqref{eq:ui1} gives that for any $i \in [T_s+1: 2T_s]$, conditional on the event  $\{u_{i-1} \ne u_{i-1}^{\mathrm{safeU}}\} \cap \left\{D_L - \bar{w} \le x'_{i-1} \le \frac{D_U}{a^*-b^*K_{\mathrm{opt}}(\hat{\theta}_s)}\right\} \cap  E_2^s \cap \neg E_{\mathrm{E}\ref{eq:first_case}} \cap X_j$, 
        \begin{align*}
            x'_i &= a^*x'_{i-1} + b^*u_{i-1} + w_{i-1} \\
            &\ge a^*x'_{i-1} - b^*K_{\mathrm{opt}}(\hat{\theta}_s)x'_{i-1} + w_{i-1} && \text{Equation \eqref{eq:ui1}}  \\
            &= (a^* - b^*K_{\mathrm{opt}}(\hat{\theta}_s))x'_{i-1} + w_{i-1} \\
            &\ge x'_{i-1} + \frac{\epsilon^*}{8}. && \text{Equation \eqref{eq:eventually2}} \numberthis \label{eq:eventually1}
        \end{align*}       
        If the control at time $j-1$ is safe (which is guaranteed by construction of the algorithm under event $E_2^s$), then $x'_j \ge D_L - \bar{w}$. Therefore by Equation \eqref{eq:probspec}, 
        \begin{equation}\label{eq:distance_upperbound}
            \frac{D_U}{a^*-b^*{K_{\mathrm{opt}}(\hat{\theta}_s)}} - x'_j \le \frac{D_U}{a^*-b^*{K_{\mathrm{opt}}(\hat{\theta}_s)}} + |D_L| + \bar{w} \le  \frac{2D_U}{c_{\mathrm{L}\ref{j_bounded_from_0}}} + |D_L| + \bar{w} \le \frac{2+c_{\mathrm{L}\ref{j_bounded_from_0}}}{c_{\mathrm{L}\ref{j_bounded_from_0}}}\norm{D}_{\infty} + \bar{w} = O_T(1).
        \end{equation}
       By Equation \eqref{eq:eventually1}, conditional on $E_2^s \cap \neg E_{\mathrm{E}\ref{eq:first_case}} \cap X_j$ the state will increase by $\epsilon^*/8$ at each step $\ell$ if $D_L - \bar{w} \le x_\ell \le \frac{D_U}{a^*-b^*K_{\mathrm{opt}}(\hat{\theta}_s)}$ and $u_\ell \ne u_{\ell}^{\mathrm{safeU}}$. Furthermore, by Equation \eqref{eq:distance_upperbound}, if the state increases by at least $\frac{2+c_{\mathrm{L}\ref{j_bounded_from_0}}}{c_{\mathrm{L}\ref{j_bounded_from_0}}}\norm{D}_{\infty} + 2\bar{w}$ from $x'_j$, then the state will be greater than $\frac{D_U}{a^*-b^*{K_{\mathrm{opt}}(\hat{\theta}_s)}}$. Increasing $\frac{2+c_{\mathrm{L}\ref{j_bounded_from_0}}}{c_{\mathrm{L}\ref{j_bounded_from_0}}}\norm{D}_{\infty} + 2\bar{w}$ state in increments of at least $\epsilon^*/8$ takes at most $\left\lceil \frac{8(\frac{2+c_{\mathrm{L}\ref{j_bounded_from_0}}}{c_{\mathrm{L}\ref{j_bounded_from_0}}}\norm{D}_{\infty} + 2\bar{w})}{\epsilon^*} \right \rceil = \tau$ steps. Putting this all together, conditional on $E_2^s \cap \neg E_{\mathrm{E}\ref{eq:first_case}} \cap X_j$, either $u_\ell = u_\ell^{\mathrm{safeU}}$ for some $\ell \in [j: j+\tau]$ or $x'_{\ell} \ge \frac{D_U}{a^*-b^*K_{\mathrm{opt}}(\hat{\theta}_s)}$ for some $\ell \in [j: j+\tau]$. Both of these alternatives imply that $u_{\ell} = u_{\ell}^{\mathrm{safeU}}$ for some $\ell \in [j:j+\tau]$, because if $x_\ell' \ge \frac{D_U}{a^*-b^*K_{\mathrm{opt}}(\hat{\theta}_s)}$, then by construction of the algorithm, $u_{\ell} = u_{\ell}^{\mathrm{safeU}}$. This is the desired result for this lemma.
\end{proof}
        The next key result is the following lemma.

        \begin{lemma}\label{lemma:second_keyequation}
            Using the notation and assumptions of the proof of Lemma \ref{eventually_go_to_boundary}, for sufficiently large $T$ and any $\ell \in [j: j+\tau]$, conditional on $\{u_{\ell} = u_{\ell}^{\mathrm{safeU}}\} \cap E_2^s \cap \neg E_{\mathrm{E}\ref{eq:first_case}} \cap X_j$, $\ell+1 \in S''_{T_{s+1}}$.
        \end{lemma}

        \begin{proof}
        Suppose $\ell \in \left[j: j +\tau \right]$. Under event $E_2^s$ the control at step $\ell-1$ is safe, and therefore by the same logic as in Equation \eqrefgen{eq:enforcingsafety3}, for sufficiently large $T$ we have that 
        \begin{equation}\label{eq:xell_bound}
            D_U - \epsilon^*/8 \le D_U - \tilde{O}_T(T^{-1/4}) \le D_U - 4B_x\epsilon_s \le a^*x'_\ell + b^*u_\ell^{\mathrm{safeU}}.
        \end{equation}
        Therefore, if $u_{\ell} = u_{\ell}^{\mathrm{safeU}}$, then
        \begin{equation}\label{eq:xell_bound2}
            a^*x'_\ell + b^*u_\ell \ge D_U - \epsilon^*/8.
        \end{equation}
        Therefore, conditional on $\{u_{\ell} = u_{\ell}^{\mathrm{safeU}}\} \cap E_2^s \cap\neg E_{\mathrm{E}\ref{eq:first_case}} \cap X_j$,
        \begin{align*}
            x'_{\ell+1} &= a^*x'_\ell + b^*u_\ell + w_\ell \\
            &\ge D_U - \epsilon^*/8 + w_\ell  && \text{Equation \eqref{eq:xell_bound2}} \\
            &\ge \frac{D_U}{a^*-b^*K_{\mathrm{opt}}(\hat{\theta}_s)}  + 3\epsilon^*/8 + w_\ell - \bar{w} && \text{Lemma \ref{good_chosen_K}} \\
            &= \frac{D_U}{a^*-b^*K_{\mathrm{opt}}(\hat{\theta}_s)} + w_\ell - (\bar{w}- 3\epsilon^*/8) \\
            &\ge \frac{D_U}{a^*-b^*K_{\mathrm{opt}}(\hat{\theta}_s)} && \text{Event $X_j$} \numberthis \label{eq:second_sdoubleprime}
        \end{align*}

        We also recall again that if $x'_{\ell+1} \ge \frac{D_U}{a^*-b^*{K_{\mathrm{opt}}(\hat{\theta}_s)}}$, then $u_{\ell+1} = u_{\ell+1}^{\mathrm{safeU}}$. Therefore, we have shown that conditional on $\{u_{\ell} = u_{\ell}^{\mathrm{safeU}}\} \cap E_2^s \cap\neg E_{\mathrm{E}\ref{eq:first_case}} \cap X_j$, $u_{\ell+1} = u_{\ell+1}^{\mathrm{safeU}}$. Furthermore, we have  for any $G_{\ell+1}$ that satisfies $ \{u_{\ell} = u_{\ell}^{\mathrm{safeU}}\} \cap E_2^s \cap\neg E_{\mathrm{E}\ref{eq:first_case}}$,
        \begin{align*}
            &\P\left( u_{\ell+1} = u_{\ell+1}^{\mathrm{safeU}} \cond  G_{\ell+1} \right) \\
            &\ge \P\left(x'_{\ell+1} \ge \frac{D_U}{a^*-b^*K_{\mathrm{opt}}(\hat{\theta}_s)} \cond G_{\ell+1}\right)\\ 
            &= \P\left(a^*x'_\ell + b^*u_\ell + w_\ell \ge \frac{D_U}{a^*-b^*K_{\mathrm{opt}}(\hat{\theta}_s)} \cond G_{\ell+1}\right) \\ 
            &\ge \P\left(D_U - \epsilon^*/8 + w_\ell \ge \frac{D_U}{a^*-b^*K_{\mathrm{opt}}(\hat{\theta}_s)} \cond G_{\ell+1}\right)  && \text{Equation \eqref{eq:xell_bound2}}\\ 
            &= \P\left(w_\ell \ge \frac{D_U}{a^*-b^*K_{\mathrm{opt}}(\hat{\theta}_s)} - D_U + \epsilon/8  \cond G_{\ell+1}\right) \\ 
            &\ge \P\left(w_\ell \ge \bar{w} - \epsilon^*/2 + \epsilon^*/8  \cond G_{\ell+1}\right) && \text{Lemma \ref{good_chosen_K}} \\ 
            &= \P_{w \sim \mathcal{D}}(w \ge \bar{w} - 3\epsilon^*/8). \numberthis \label{eq:utou}
        \end{align*}

        By Definition of $S''_t$, Equations \eqref{eq:second_sdoubleprime} and \eqref{eq:utou} imply the desired result that conditional on $ \{u_{\ell} = u_{\ell}^{\mathrm{safeU}}\} \cap E_2^s \cap\neg E_{\mathrm{E}\ref{eq:first_case}} \cap X_j$, we have that $\ell+1 \in S''_{T_{s+1}}$.

        \end{proof}

        Putting together the two lemmas, we have that conditional on $E_2^s \cap \neg E_{\mathrm{E}\ref{eq:first_case}} \cap X_j$, there exists an $\ell \in [j:j+\tau]$ such that $u_\ell = u_\ell^{\mathrm{safeU}}$, and for any $\ell \in [j: j+\tau]$, conditional on $\{u_{\ell} = u_{\ell}^{\mathrm{safeU}}\} \cap E_2^s \cap \neg E_{\mathrm{E}\ref{eq:first_case}} \cap X_j$, $\ell+1 \in S''_{T_{s+1}}$. Combining these two lemmas gives that conditional on $E_2^s \cap \neg E_{\mathrm{E}\ref{eq:first_case}} \cap X_j$, there exists an $\ell \in [j: j+\tau+1]$ such that $\ell \in S''_{T_{s+1}}$. For sufficiently large $T$, $\tau+1 \le \lceil \log(T) \rceil $, and therefore this is  exactly the desired result.
    \end{proof}
\subsection{Proof of Lemma \ref{lemma:convexity}}\label{sec:proof_of_lemma:convexity}

\begin{proof}
    Let $x_0,x_1,...,$ be the series of states when using controller $C_K^{\mathrm{unc}}$ under dynamics $\theta$ with $x_0 = 0$. Then we have the recursive relationship that $x_0 = 0$ and $x_{i+1} = (a-bK)x_{i} + w_i$ for all $i \ge 0$.  Using this recursive relationship, we have that
    \begin{equation}\label{eq:xexpr}
        x_t = \sum_{i=0}^{t-1} w_i(a-bK)^{t-1-i}.
    \end{equation}
    If $K = \frac{a-1}{b}$, then $a-bK = 1$. This implies that $ x_t^2 \longrightarrow \infty$, and therefore $\bar{J}(\theta, C_K^{\mathrm{unc}}, T) = \infty$.
    
     For the rest of this proof, assume $K \in (\frac{a-1}{b}, \frac{a}{b}]$. Recall that $u_i = -Kx_i$ for all $i \ge 0$. Define $\rho = (a-bK)^2$. Using the above expression for $x_t$, we have that
    \begin{align*}
        \bar{J}(\theta, C_K^{\mathrm{unc}}, T) &= \frac{1}{T}\E\left[qx_T^2 + \sum_{t=0}^{T-1} qx_t^2 + ru_t^2\right] \\
        &= \frac{1}{T}\left(q\E[x_T^2] + \sum_{t=1}^{T-1} (q+rK^2)\E[x_t^2] \right) && \text{[$x_0 = u_0 = 0$]}\\
        &= -\frac{rK^2\E[X_T^2]}{T} +   \frac{1}{T}\left(\sum_{t=1}^{T} (q+rK^2)\E[x_t^2] \right).
    \end{align*}
    Furthermore, we have
    \begin{align*}
        &\frac{1}{T}\left(\sum_{t=1}^{T} (q+rK^2)\E[x_t^2] \right) \\
        &= \frac{1}{T}\sum_{t=1}^T (q+rK^2) \E\Big[\Big(\sum_{i=0}^{t-1} w_i(a-bK)^{t-1-i}\Big)^2\Big] && \text{Equation \eqref{eq:xexpr}}\\
        &= \frac{1}{T}\sum_{t=1}^T (q+rK^2) \E\Big[\sum_{i=0}^{t-1} \sum_{j=0}^{t-1} w_iw_j (a-bK)^{t-1-i}(a-bK)^{t-1-j}\Big]\\
        &= \frac{1}{T}\sum_{t=1}^T (q+rK^2) \sum_{i=0}^{t-1}\sigma_{\mathcal{D}}^2 (a-bK)^{2(t-1-i)}\\
        &= \frac{\sigma_{\mathcal{D}}^2}{T}\sum_{t=1}^T (q + rK^2)\sum_{i=0}^{t-1} (a-bK)^{2i} \\
        &= \frac{\sigma_{\mathcal{D}}^2(q + rK^2)}{T}\sum_{t=1}^T \sum_{i=0}^{t-1} \rho^i \\ 
        &= \frac{\sigma_{\mathcal{D}}^2(q + rK^2)}{T}\sum_{t=1}^T \frac{1 -  \rho^t}{1- \rho}\\ 
        &= \frac{\sigma_{\mathcal{D}}^2(q + rK^2)}{T(1- \rho)}\left(T - \sum_{t=0}^{T-1} \rho^t\right)\\ 
        &= \frac{\sigma_{\mathcal{D}}^2(q + rK^2)}{1- \rho}\Big(1 - \frac{1- \rho^T}{T(1- \rho)}\Big). 
    \end{align*}
    By the same logic, we have that
    \begin{align*}
        \frac{rK^2\E[X_T^2]}{T} = \frac{rK^2\sigma_\mathcal{D}^2\frac{1-\rho^T}{1-\rho}}{T}.
    \end{align*}
    Therefore,
    \begin{align*}
        \bar{J}(\theta, C_K^{\mathrm{unc}}) &= \lim_{T \rightarrow \infty}   \bar{J}(\theta, C_K^{\mathrm{unc}},T) \\
        &=  \lim_{T \rightarrow \infty} - \frac{rK^2\sigma_\mathcal{D}^2\frac{1-\rho^T}{1-\rho}}{T} + \frac{\sigma_{\mathcal{D}}^2(q + rK^2)}{1- \rho}\Big(1 - \frac{1- \rho^T}{T(1- \rho)}\Big) \\
        &= \frac{\sigma_{\mathcal{D}}^2(q + rK^2)}{1-(a-bK)^2}.
    \end{align*}
    Now, we note the following derivatives:
\[
    \frac{d}{dK} \left(\frac{1}{1 - (a-bK)^2} \right)= \frac{2b(a-bK)}{(1 - (a-bK)^2)^2} 
\]
and
\[
     \frac{d}{dK} \left(\frac{K^2}{1-(a-bK)^2} \right) = \frac{2aK(1 - (a-bK))}{(1 - (a-bK)^2)^2}  .
\]
For $K \in (\frac{a-1}{b}, \frac{a}{b}]$, if $1-(a-bK) = c > 0$, then $1-(a-bK)^2 > c > 0$, and therefore these derivatives imply that
\begin{align*}
\left| \frac{d}{dK} \bar{J}(\theta, C_K^{\mathrm{unc}})\right| &= \left|\frac{d}{dK}  \frac{\sigma_{\mathcal{D}}^2(q + rK^2)}{1-(a-bK)^2} \right|\\
&= \left| \sigma_{\mathcal{D}}^2\left(q  \frac{2b(a-bK)}{(1 - (a-bK)^2)^2}+ r \frac{2aK(1 - (a-bK))}{(1 - (a-bK)^2)^2} \right) \right| \\
&\le \sigma_{\mathcal{D}}^2\left(q \frac{2b(a-bK)}{c^2}+ r \frac{2a|K|(1 - (a-bK))}{c^2} \right) \\
& < \infty.
\end{align*}
For all $K \in (\frac{a-1}{b}, \frac{a}{b}]$, we also have that
\[
    \frac{d^2}{dK^2} \left(\frac{1}{1 - (a-bK)^2} \right) =b^2 \Big(\frac{1}{(1 - (a-bK))^3} + \frac{1}{(1 + (a-bK))^3}\Big) > 0
\]
and
\[
    \frac{d^2}{dK^2} \left(\frac{K^2}{1-(a-bK)^2} \right) =b^2\Big(\frac{(a-1)^2}{(1 - (a-bK))^3} + \frac{(a+1)^2}{(1 + (a-bK))^3}\Big) > 0
\]
This implies that
\[
    \frac{d^2}{dK^2} \bar{J}(\theta, C_K^{\mathrm{unc}}) > 0.
\]
If $a-bK = 1-c < 1$, we also have that
\[
    \frac{d^2}{dK^2} \left(\frac{1}{1 - (a-bK)^2} \right) =b^2 \Big(\frac{1}{(1 - (a-bK))^3} + \frac{1}{(1 + (a-bK))^3}\Big) \le b^2\left(\frac{1}{c^3} + 1\right) < \infty
\]
and
\[
    \frac{d^2}{dK^2} \left(\frac{K^2}{1-(a-bK)^2} \right) =b^2\Big(\frac{(a-1)^2}{(1 - (a-bK))^3} + \frac{(a+1)^2}{(1 + (a-bK))^3}\Big) \le b^2\left(\frac{(a-1)^2}{c^3} + (a+1)^2\right)  < \infty.
\]
These two equations imply that for  $K \in (\frac{a-1}{b}, \frac{a}{b}]$,
\[
    \frac{d^2}{dK^2} \bar{J}(\theta, C_K^{\mathrm{unc}})  < \infty.
\]

\end{proof}

\subsection{Proof of Lemma \ref{close_to_KDU}}\label{sec:proof_of_close_to_KDU} 
\begin{proof}
    We first note the following bounds on $K'$ that we will use throughout this proof that come from the assumptions on $\epsilon$. For any $\theta' \in \hat{\theta}_{\mathrm{L}\ref{close_to_KDU}} \pm \beta$,
    \begin{align*}
       a' - b'K' &= a - bK_{D_U}^{\theta} + (a' - a) + b(K_{D_U}^\theta - K') + K'(b-b') \\
       &\le a-bK_{D_U}^\theta + b\epsilon +  \beta + \beta K'  \\
        &\le \frac{D_U}{\bar{w} + D_U} + \upsilon && \text{Def of $K_{D_U}^\theta$}\\
        &\le \frac{D_U + \bar{w}/2}{\bar{w} + D_U} &&\text{Equation \eqref{eq:def_of_upsilon}} \\
        &< 1. \numberthis \label{eq:boundfrom12}
    \end{align*}
    \begin{align*}
       a' - b'K' &\ge a-bK_{D_U}^\theta - b\epsilon -  \beta - \beta K'  \\
        &\ge \frac{D_U}{\bar{w} + D_U} -\upsilon && \text{Def of $K_{D_U}^\theta$}\\
        &\ge \frac{D_U/2}{\bar{w} + D_U} && \text{Equation \eqref{eq:def_of_upsilon}} \\
        &> 0.  \numberthis \label{eq:boundfrom0}
    \end{align*}

   Let $y_t$ be the state at time $t$ when using controller $C$ and starting at state $y_0 = x_0$ and $x_t$ be the state at time $t$ when using controller $C_{K'}^{\mathrm{unc}}$ and starting at state $x_0$. Define $d_t := |y_t - x_t|$. Define
    \begin{equation}
        \theta_m := \arg\max_{\norm{\theta' - \hat{\theta}_{\mathrm{L}\ref{close_to_KDU}}}_{\infty} \le \beta } a' - b'K'.
    \end{equation}
    Importantly, note that $\theta_m = \arg \min_{\theta' \in \hat{\theta}_{\mathrm{L}\ref{close_to_KDU}} \pm \beta} \frac{D_U}{a'-b'K'} = \arg \max_{\theta' \in \hat{\theta}_{\mathrm{L}\ref{close_to_KDU}} \pm \beta} \frac{D_L}{a'-b'K'}$. By construction this means that $C(y_t) = v_t^{\mathrm{safeU}}$ is used if and only if $y_t \ge \frac{D_{\mathrm{U}}}{a_m-b_mK'}$, and similarly $v_t^{\mathrm{safeL}}$ is used if and only if $y_t \le \frac{D_{\mathrm{L}}}{a_m-b_mK'}$.

    \begin{lemma}\label{lemma:prob_of_using_v}
        Define $H_t = (y_0,y_1,...,y_{t-1})$. Using the notation and assumptions in the proof of Lemma \ref{close_to_KDU}, for any $H_t$,
        \begin{equation}
        \P\left(C(y_t) = v_t^{\mathrm{safeU}} \cond H_t\right) = O_T(\upsilon) \cdot 1_{K' - K_{D_U}^\theta \le \frac{(|K'|+1)\beta}{b}}.
        \end{equation} 
        Furthermore, 
        \begin{equation}
            \P\left(C(y_{t}) = v_t^{\mathrm{safeL}} \cond H_t\right) = O_T(\upsilon) \cdot  1_{K' - K_{D_U}^\theta \le \frac{(|K'|+1)\beta}{b}}.
        \end{equation}
    \end{lemma}
        The proof of Lemma \ref{lemma:prob_of_using_v} can be found in Appendix \ref{app:proof_of_prob_of_using_v}.
        Because the equations in Lemma \ref{lemma:prob_of_using_v} hold for any $H_t$, this lemma implies that
        \begin{equation}\label{eq:lowprobexc}
        \P\left(C(y_t) = v_t^{\mathrm{safeU}} \right) = O_T(\upsilon) \cdot 1_{K' - K_{D_U}^\theta \le \frac{(|K'|+1)\beta}{b}}
        \end{equation} 
        and
        \begin{equation}\label{eq:lowprobexc2}
            \P\left(C(y_{t}) = v_t^{\mathrm{safeL}} \right) = O_T(\upsilon) \cdot  1_{K' - K_{D_U}^\theta \le \frac{(|K'|+1)\beta}{b}}.
        \end{equation}

        By Lemma \ref{lemma:prob_of_using_v}, if $K' - K_{D_U}^\theta > \frac{(|K'|+1)\beta}{b}$, then for all $t$,
        \[
            \P\left(C(y_t) = v_t^{\mathrm{safeU}}\text{ or } C(y_{t}) = v_t^{\mathrm{safeL}} \right) = 0.
        \]
        Therefore in this case, the controllers $C$ and $C_{K'}^{\mathrm{unc}}$ are equivalent, which implies all of the desired results. For the rest of the proof, we will address the case when $K' - K_{D_U}^\theta \le \frac{(|K'|+1)\beta}{b}$. This combined with the definition of $\epsilon$ gives that 
        \begin{equation}\label{eq:assum_on_K}
             |K' - K_{D_U}^\theta| \le \min\left(\frac{(|K'|+1)\beta}{b}, \epsilon\right) =  O_T(\upsilon).
        \end{equation}

        \begin{lemma}\label{lemma:dt_and_control}
            Using the notation and assumptions in the proof of Lemma \ref{close_to_KDU}, if Equation \eqref{eq:assum_on_K} holds then for all $t \ge 0$,
            \begin{equation}\label{eq:recurs_dt}
            d_{t+1} =  \begin{cases}
           (a-bK')d_{t} & \text{if } \frac{D_L}{a_m-b_mK'} \le y_{t} \le \frac{D_U}{a_m-b_mK'}\\
            (a-bK')d_{t} + O_T(\upsilon) & \text{otherwise},
        \end{cases}
            \end{equation}
                    
            and
        \begin{equation}\label{eq:diff_in_controls_K}
            \left|C_{K'}^{\mathrm{unc}}(x_t) - C(y_t)\right| =  \begin{cases}
           |K'|d_{t}   & \text{if } \frac{D_L}{a_m-b_mK'} \le y_{t} \le \frac{D_U}{a_m-b_mK'}\\
          O_T(\upsilon) & \text{otherwise}. 
        \end{cases}
        \end{equation}
        \end{lemma}
        The proof of Lemma \ref{lemma:dt_and_control} can be found in Appendix \ref{app:proof_of_dt_and_control}.
        
         This recursive relationship for $d_t$ in Lemma \ref{lemma:dt_and_control} implies  that
        \begin{align*}
            d_t &= |x_t - y_t| \\
            &\le \sum_{i=1}^{t} (a-bK')^{i-1} O_T(\upsilon)  1_{y_{t-i} \ge \frac{D_U}{a_m-b_mK'} \text{ or } y_{t-i} \le \frac{D_L}{a_m-b_mK'}}  && \text{Lemma \ref{lemma:dt_and_control}}\\
            &\le O_T(\upsilon)\sum_{i=0}^\infty (a-bK')^i \\
            &\le \frac{O_T(\upsilon)}{1-(a-bK')}  \\
            &\le O_T(\upsilon). &&  \text{Equation \eqref{eq:boundfrom12}}\numberthis \label{eq:dt_bound} 
            \end{align*}
        Note that $y_t$ is by construction safe with respect to dynamics $\theta_m$. Therefore, $|a_my_t + b_mC(y_t)| \le \norm{D}_{\infty}$ and $|y_t| \le \norm{D}_{\infty} + \bar{w}$, which together imply that
        \begin{equation}\numberthis \label{eq:cyt_bound}
            |C(y_t)| \le \frac{\norm{D}_{\infty} + a_m|y_t|}{b_m} = O_T(1). 
        \end{equation}
        
        Now we can bound the difference in cost at time $t \ge 0$ as follows:
        \begin{align*}
            &|qx_t^2 - qy_t^2| + |rC_{K'}^{\mathrm{unc}}(x_t)^2 - rC(y_t)^2| \\
            &\le 2q|y_t|d_t + qd_t^2 + \left(2r|C(y_t)| \left|C_{K'}^{\mathrm{unc}}(x_t) - C(y_t)\right| + r \left|C_{K'}^{\mathrm{unc}}(x_t) - C(y_t)\right|^2 \right)\\
            &\le 2q|y_t|d_t + qd_t^2 + \left(2rO_T(1)| \left|C_{K'}^{\mathrm{unc}}(x_t) - C(y_t)\right| + r \left|C_{K'}^{\mathrm{unc}}(x_t) - C(y_t)\right|^2 \right) && \text{Equation \eqref{eq:cyt_bound}}\\
            &\le 2q(\norm{D}_{\infty} + \bar{w})d_t + qd_t^2 + \Big(2rO_T(1)\left(|K'|d_t + O_T(\upsilon)1_{y_{t} \ge \frac{D_U}{a_m-b_mK'} \text{ or } y_{t} \le \frac{D_L}{a_m-b_mK'}}\right) \\
            &\quad \quad \quad \quad \quad \quad+ r\left(|K'|d_t + O_T(\upsilon)1_{y_{t} \ge \frac{D_U}{a_m-b_mK'} \text{ or } y_{t} \le \frac{D_L}{a_m-b_mK'}}\Big)^2\right) && \text{Equation \eqref{eq:diff_in_controls_K}}\\
            &= O_T\left(d_t + \upsilon^2 + \upsilon 1_{y_{t} \ge \frac{D_U}{a_m-b_mK'} \text{ or } y_{t} \le \frac{D_L}{a_m-b_mK'}}\right). && \text{ Equation \eqref{eq:dt_bound}} \numberthis \label{eq:onestepbound1} 
        \end{align*}
        We will now show that $\E[d_t] \le O_T(\upsilon^2)$. Importantly, we use that the event 
        \[
        \{{y_{i-1} \ge \frac{D_U}{a_m-b_mK'} \text{ or } y_{i-1} \le \frac{D_L}{a_m-b_mK'}}\}
        \]
        is equivalent to the event that $C(y_{i-1}) \in \{v_t^{\mathrm{safeU}}, v_t^{\mathrm{safeL}}\}$, which allows us to apply Lemma \ref{lemma:prob_of_using_v} in the second line.
        \begin{align*}
            \E[d_t] &\le O_T(\upsilon)\sum_{i=1}^t (a-bK')^{t-i}  \E[1_{y_{i-1} \ge \frac{D_U}{a_m-b_mK'} \text{ or } y_{i-1} \le \frac{D_L}{a_m-b_mK'}}] && \text{Lemma \ref{lemma:dt_and_control}}\\
            &\le O_T(\upsilon) \sum_{i=1}^t (a-bK')^{t-i} O_T(\upsilon) && \text{Lemma \ref{lemma:prob_of_using_v}}\\ 
            &\le O_t(\upsilon^2) \sum_{i=0}^{\infty} (a-bK')^{t-i} \\
            &\le \frac{O_T(\upsilon^2)}{1 - (a -bK')} \\
            &\le O_T(\upsilon^2). && \text{Equation \eqref{eq:boundfrom12}} \numberthis \label{eq:onestepbound2}
        \end{align*}
        Therefore,
        \begin{align*}
            &|\bar{J}(\theta, C, \tau, x_0) - \bar{J}(\theta, C_{K'}^{\mathrm{unc}}, \tau, x_0)|\\
            &\le \E\left[\frac{1}{\tau}\left(q|x_\tau^2- y_\tau^2| + \sum_{t=0}^{\tau-1} |qx_t^2 - qy_t^2| + |rC_{K'}^{\mathrm{unc}}(x_t)^2 - rC(y_t)^2|\right)\right] \\
            &\le \E\left[\frac{1}{\tau}\sum_{t=0}^{\tau} O_T\left(d_t + \upsilon^2 + \upsilon 1_{y_{t} \ge \frac{D_U}{a_m-b_mK'} \text{ or } y_{t} \le \frac{D_L}{a_m-b_mK'}}\right)\right] && \text{Equation \eqref{eq:onestepbound1} } \\
            &\le \frac{1}{\tau}\sum_{t=0}^{\tau} O_T\left(\E[d_t] + \upsilon^2 + \upsilon \E\left[1_{y_{t} \ge \frac{D_U}{a_m-b_mK'} \text{ or } y_{t} \le \frac{D_L}{a_m-b_mK'}}\right]\right)  \\
            &\le \frac{1}{\tau} \sum_{t=0}^{\tau} O_T(\upsilon^2) && \text{Equation \eqref{eq:onestepbound2}, Lemma \ref{lemma:prob_of_using_v}} \\
            &\le O_T(\upsilon^2).
        \end{align*}
        Taking a limit as $ \tau \rightarrow \infty$ of the above equation (where nothing on the right side depends on $\tau$) gives the first desired  equation that
        \[
                        |\bar{J}(\theta, C, x_0) - \bar{J}(\theta, C_{K'}^{\mathrm{unc}}, x_0)| \le O_T(\upsilon^2).
                        \]
        Now we want to bound the difference in cost with high probability instead of in expectation. Let $X$ be the set of times $t \in [0:\tau]$ such that $C(y_t) \ne -K'y_t$ (i.e. $C(y_t) \in \{v_t^{\mathrm{safeL}}, v_t^{\mathrm{safeU}}\}$). Note that the event $\{t \in X\}$ is the same as the event $\{y_{t} \ge \frac{D_U}{a_m-b_mK'} \text{ or } y_{t} \le \frac{D_L}{a_m-b_mK'}\}$.
        
        By Lemma \ref{lemma:prob_of_using_v}, $\P(t \in X \mid H_t) \le c\upsilon$ for some constant $c > 0$ for all $t$. Therefore, $M_k = \sum_{t=0}^{\tau} \left(1_{t \in X} - c\upsilon \right)$ is a supermartingale. By Azuma--Hoeffding's inequality, with probability $1-o_T(1/T^{10})$,
        \[
            |X| \le  O_T(\upsilon \tau) + \log(T)\sqrt{\tau}.
        \]
        Define $A$ as the event that $|X| \le O_T(\upsilon \tau) + \log(T)\sqrt{\tau}$. Define $\kappa = \lceil \log_{a-bK'}(\upsilon) \rceil$. Note that
        \begin{align*}
            \kappa &= \lceil \log_{a-bK'}(\upsilon) \rceil \\
            &\le \left\lceil \frac{\log(\upsilon)}{\log(a-bK')} \right\rceil\\
            &= O(\log(\upsilon)) && \text{Lemma \ref{eq:boundfrom12}} \numberthis \label{eq:bound_on_upsilon}
        \end{align*}
        
        Define
        \[
            G = \left\{ t \in [0:\tau] : \exists i \in [t-\kappa: t] \text{ such that } C(y_i) \ne -K'y_i \right\}.
        \]
        Under event $A$, 
        \begin{equation}\label{eq:under_event_A_G_bounded}
            \left| G \right| \le |X| \cdot (\kappa+1) \le (O_T(\upsilon \tau) + \log(T)\sqrt{\tau})(\kappa+1).
        \end{equation}
        By Lemma \ref{lemma:dt_and_control}, if $t \not\in G$, then 
        \begin{align*}
            d_t &\le  O_T(\upsilon)\sum_{i=1}^t (a-bK')^{t-i}  1_{y_{i-1} \ge \frac{D_U}{a_m-b_mK'} \text{ or } y_{i-1} \le \frac{D_L}{a_m-b_mK'}} && \text{Lemma \ref{lemma:dt_and_control}}  \\
             &\le  O_T(\upsilon)\sum_{i=1}^{t-\kappa} (a-bK')^{t-i}  1_{y_{i-1} \ge \frac{D_U}{a_m-b_mK'} \text{ or } y_{i-1} \le \frac{D_L}{a_m-b_mK'}}  && \text{$t \not\in G$}\\
             &\le  O_T(\upsilon)(a-bK')^\kappa\sum_{i=1}^{t-\kappa} (a-bK')^{t-i-\kappa}  1_{y_{i-1 } \ge \frac{D_U}{a_m-b_mK'} \text{ or } y_{i-1} \le \frac{D_L}{a_m-b_mK'}} \\
            &\le O_T(\upsilon)(a-bK')^\kappa \sum_{i=0}^\infty (a-bK')^i \\
            &\le O_T(\upsilon^2) \sum_{i=1}^\infty (a-bK')^i && \text{Definition of $\kappa$} \\
            &= \frac{O_T(\upsilon^2)}{1-(a-bK')} \\
            &= O_T(\upsilon^2). && \text{Equation \eqref{eq:boundfrom12}} \numberthis \label{eq:highprobrec}
        \end{align*}
        Recall that by Equation \eqref{eq:dt_bound}, for any $t \in G$,  $d_t \le O_T(\upsilon)$, therefore Equation \eqref{eq:highprobrec} implies that
        \begin{equation}\label{eq:highprobrec2}
            d_t = O_T\left(\upsilon 1_{t \in G} + \upsilon^2\right).
        \end{equation}
        Using that $t \in G$  for all $t$ satisfying $y_{t} \ge \frac{D_U}{a_m-b_mK'} \text{ or } y_{t} \le \frac{D_L}{a_m-b_mK'}$, we have that under event $A$,
        \begin{align*}
            &|J(\theta, C, \tau,x_0,W') - J(\theta, C_{K'}^{\mathrm{unc}}, \tau,x_0,W')|\\
            &\le \frac{1}{\tau}\sum_{t=0}^{\tau} |qx_t^2 - qy_t^2| + |rC_{K'}^{\mathrm{unc}}(x_t)^2 - rC(y_t)^2| \\
            &= \frac{1}{\tau}\sum_{t=0}^{\tau} O_T\left(d_t + \upsilon^2 + \upsilon 1_{y_{t} \ge \frac{D_U}{a_m-b_mK'} \text{ or } y_{t} \le \frac{D_L}{a_m-b_mK'}}\right)&& \text{Equation \eqref{eq:onestepbound1} } \\
            &= \frac{1}{\tau}\sum_{t=0}^{\tau} O_T\left( \upsilon \cdot 1_{t\in G} + \upsilon^2 + \upsilon 1_{y_{t} \ge \frac{D_U}{a_m-b_mK'} \text{ or } y_{t} \le \frac{D_L}{a_m-b_mK'}}\right)&& \text{Equation \eqref{eq:highprobrec2} } \\
            &= O_T(\upsilon^2) + \frac{1}{\tau} \sum_{t=0}^{\tau} O_T\left(\upsilon \right)\cdot  1_{t\in G} && \\
            &= O_T(\upsilon^2) + O_T\left(\frac{\upsilon \cdot (O_T(\upsilon \tau) + \log(T)\sqrt{\tau})(\kappa+1)}{\tau}\right) && \text{Equation \eqref{eq:under_event_A_G_bounded}} \\
            &= O_T\left(\upsilon\log(1/\upsilon)\left(\upsilon + \frac{\log(T)}{\sqrt{\tau}}\right)\right). && \text{Equation \eqref{eq:bound_on_upsilon}}
        \end{align*}
        Since this holds under event $A$ and $\P(A) \ge 1-o_T(1/T^{10})$, this completes the proof.
\end{proof}

    \subsection{Proof of Lemma \ref{good_chosen_K}}\label{sec:proof_of_good_chosen_K}  
    \begin{proof}
        In Algorithm \ref{alg:cap3}, $\hat{\theta}_s$ satisfies 
        \[
          \hat{\theta}_s = \argmax_{\norm{\theta - \hat{\theta}^{\text{pre}}_s} \le \epsilon_s}  a - bK_{\mathrm{opt}}(\theta).
        \]
        Under event $E_2^s$, we have that $\norm{\hat{\theta}_s^{\mathrm{pre}} - \theta^*} \le \epsilon_s$, which implies that
        \[
                \hat{a}_s - \hat{b}_sK_{\mathrm{opt}}(\hat{\theta}_s) \ge a^* -b^*K_{\mathrm{opt}}(\theta^*).
        \]
        Therefore, we have that (using Lemma \ref{cswitch_prop} in the equality)
        \begin{equation}\label{eq:good_chosen}
            \frac{D_U}{\hat{a}_s - \hat{b}_sK_{\mathrm{opt}}(\hat{\theta}_s)} - D_U \le \frac{D_U}{a^*-b^*K_{\mathrm{opt}}(\theta^*)} - D_U =  \bar{w} - \epsilon^*.
        \end{equation}
        Under event $E_2^s$, we also have that $\norm{\hat{\theta}_s - \theta^*}_{\infty} \le \tilde{O}_T(T^{-1/4})$, therefore
    \begin{align*}
           & \frac{D_U}{a^*-b^*K_{\mathrm{opt}}(\hat{\theta}_s)} - D_U \\
           &=  \frac{D_U}{\hat{a}_s - \hat{b}_sK_{\mathrm{opt}}(\hat{\theta}_s)} - D_U + \frac{D_U}{a^*-b^*K_{\mathrm{opt}}(\hat{\theta}_s)}  -  \frac{D_U}{\hat{a}_s - \hat{b}_sK_{\mathrm{opt}}(\hat{\theta}_s)}  \\
            &\le \bar{w} - \epsilon^* + D_U\frac{(\hat{a}_s - a^*) + (b^* - \hat{b}_s)K_{\mathrm{opt}}(\hat{\theta}_s)}{(a^*-b^*K_{\mathrm{opt}}(\hat{\theta}_s))(\hat{a}_s - \hat{b}_sK_{\mathrm{opt}}(\hat{\theta}_s))}  && \text{Eq \eqref{eq:good_chosen}}\\
            &= \bar{w} -\epsilon^* + D_U\frac{(\hat{a}_s - a^*) + (b^* - \hat{b}_s)K_{\mathrm{opt}}(\hat{\theta}_s)}{\left(\hat{a}_s-\hat{b}_sK_{\mathrm{opt}}(\hat{\theta}_s) - (\hat{a}_s - a^*) - (b^* - \hat{b}_s)K_{\mathrm{opt}}(\hat{\theta}_s)\right)(\hat{a}_s - \hat{b}_sK_{\mathrm{opt}}(\hat{\theta}_s))} \\
            &\le \bar{w} -\epsilon^* + D_U\frac{\norm{\theta^* - \hat{\theta}_s}_{\infty} \left(1 +|K_{\mathrm{opt}}(\hat{\theta}_s)|\right)}{(\hat{a}_s-\hat{b}_sK_{\mathrm{opt}}(\hat{\theta}_s) - \norm{\theta^* - \hat{\theta}_s}_{\infty} (1+|K_{\mathrm{opt}}(\hat{\theta}_s)|))(\hat{a}_s - \hat{b}_sK_{\mathrm{opt}}(\hat{\theta}_s))} \\
            &\le \bar{w} -\epsilon^* + \frac{D_U\tilde{O}_T(T^{-1/4})\left(1 + |K_{\mathrm{opt}}(\hat{\theta}_s)|\right)}{(\hat{a}_s-\hat{b}_sK_{\mathrm{opt}}(\hat{\theta}_s) - \tilde{O}_T(T^{-1/4})(1+|K_{\mathrm{opt}}(\hat{\theta}_s)|))(\hat{a}_s - \hat{b}_sK_{\mathrm{opt}}(\hat{\theta}_s))} \\
            &\le \bar{w} - \epsilon^*/2. && \text{Eq \eqref{eq:to_use_above}} \numberthis \label{eq:du_hat_star}
    \end{align*}
    To see the last inequality, note that Lemma \ref{j_bounded_from_0} gives that $1 > \hat{a}_s - \hat{b}_sK_{\mathrm{opt}}(\hat{\theta}_s) \ge c_{\mathrm{L}\ref{j_bounded_from_0}}$. This implies that $|K_{\mathrm{opt}}(\hat{\theta}_s)| = O_T(1)$, and therefore for sufficiently large $T$ we have that
    \begin{align*}
        &\frac{D_U\tilde{O}_T(T^{-1/4})\left(1 + K_{\mathrm{opt}}(\hat{\theta}_s)\right)}{(\hat{a}_s-\hat{b}_sK_{\mathrm{opt}}(\hat{\theta}_s) - \tilde{O}_T(T^{-1/4})(1+K_{\mathrm{opt}}(\hat{\theta}_s)))(\hat{a}_s - \hat{b}_sK_{\mathrm{opt}}(\hat{\theta}_s))} \\
        &\le \frac{\tilde{O}_T(T^{-1/4})D_U\left(1 + O_T(1)\right)}{(c_{\mathrm{L}\ref{j_bounded_from_0}} - \tilde{O}_T(T^{-1/4})(1+O_T(1)))c_{\mathrm{L}\ref{j_bounded_from_0}}}\\
        &\le \epsilon^*/2. \numberthis \label{eq:to_use_above}
    \end{align*}
    Finally, rearranging Equation \eqref{eq:du_hat_star} gives exactly the desired result.
    \end{proof}

    \subsection{Proof of Lemma \ref{lemma:prob_of_using_v}}\label{app:proof_of_prob_of_using_v}

    \begin{lemma}\label{lemma:safety_of_C}
        Using the same notation and assumptions of Lemma \ref{lemma:prob_of_using_v}, for all $\theta' \in \hat{\theta}_{\mathrm{L}\ref{close_to_KDU}}  \pm \beta$, the controls used by controller $C$ are safe for dynamics $\theta'$ for all $t \in [0:T-1]$.
    \end{lemma}
    The proof of Lemma \ref{lemma:safety_of_C} can be found in Appendix \ref{app:proof_of_safety_of_C}

    By definition, $C(y_t) = v_t^{\mathrm{safeU}}$ if and only if there exists a $\theta' \in \hat{\theta}_{\mathrm{L}\ref{close_to_KDU}} \pm \beta$ such that $y_t \ge \frac{D_U}{a'-b'K'}$.  Equivalently, $C(y_t) = v_t^{\mathrm{safeU}}$ if and only if $y_t \ge \frac{D_U}{a_m-b_mK'}$. We also note that
    \begin{align*}
      (a-bK_{D_U}^\theta) - (a_m -b_mK') &= (a - a_m) +b (K' - K_{D_U}^{\theta}) + K'(b_m - b) \\
        &\ge -\beta + b(K'- K_{D_U}^{\theta}) - |K'| \beta \\
        &= b(K'- K_{D_U}^{\theta}) - (1 + |K'|) \beta \\
        &\ge  \Big( b(K'- K_{D_U}^{\theta}) - (|K'| + 1) \beta \Big)1_{K' - K_{D_U}^\theta \le \frac{(|K'|+1)\beta}{b}} \\
        &\ge  -(b\epsilon  + (|K'|+1)\beta)1_{K' - K_{D_U}^\theta \le \frac{(|K'|+1)\beta}{b}}        \\
        &= -\upsilon 1_{K' - K_{D_U}^\theta \le \frac{(|K'|+1)\beta}{b}} .\numberthis \label{eq:a_to_m}
    \end{align*}
Therefore,    
    \begin{align*}
            &\P\left(C(y_t) = v_t^{\mathrm{safeU}} \mid H_t\right)\\
            &= \P\left(y_t \ge \frac{D_U}{a_m-b_mK'} \mid H_t\right) \\
            &= \P\left(ay_{t-1} + bC(y_{t-1}) + w_{t-1} \ge \frac{D_U}{a_m-b_mK'} \mid H_t\right) \\
            &\le \P\left(|w_{t-1}| \ge \frac{D_U}{a_m-b_mK'} - D_U \mid H_t\right) &&  \text{Lemma \ref{lemma:safety_of_C}}\\
             &= \P\left(|w_{t-1}| \ge \frac{D_U}{a_m-b_mK'} + \bar{w} - \frac{D_U}{a - bK_{D_U}^\theta}\right) && \text{Definition of $K_{D_U}^\theta$} \\
            &= \P\left(|w_{t-1}| \ge \bar{w} + \frac{D_U(a-bK_{D_U}^\theta) - D_U(a_m - b_mK')}{(a-bK_{D_U}^\theta)(a_m-b_mK')}\right) \\
            &\le \P\left(|w_{t-1}| \ge \bar{w} - \frac{D_U\upsilon }{(a-bK_{D_U}^\theta)(a_m-b_mK')}\right)1_{K' - K_{D_U}^\theta \le \frac{(|K'|+1)\beta}{b}} 
 && \text{Equation \eqref{eq:a_to_m}}\\
            &\le \P\left(|w_{t-1}| \ge \bar{w} - \frac{ D_U\upsilon}{(a-bK_{D_U}^\theta)(a-bK_{D_U}^\theta - \upsilon)}\right) 1_{K' - K_{D_U}^\theta \le \frac{(|K'|+1)\beta}{b}} && \text{Equation \eqref{eq:a_to_m}}\\
            &\le \P\left(|w_{t-1}| \ge \bar{w} - \frac{ D_U \upsilon}{(\frac{D_U}{\bar{w} + D_U})(\frac{D_U}{\bar{w} + D_U} - \frac{D_U}{2(\bar{w} + D_U)})}\right) 1_{K' - K_{D_U}^\theta \le \frac{(|K'|+1)\beta}{b}} && \text{Def \ref{def:Kdu_theta}, $\upsilon \le \frac{(D_U/2)}{(\bar{w} + D_U)}$}\\
            &= \P\left(|w_{t-1}| \ge \bar{w} - \frac{2\upsilon(\bar{w} + D_U)^2}{D_U}\right)1_{K' - K_{D_U}^\theta \le \frac{(|K'|+1)\beta}{b}}  \\
            &\le \frac{4B_P\upsilon(\bar{w} + D_U)^2}{D_U} 1_{K' - K_{D_U}^\theta \le \frac{(|K'|+1)\beta}{b}} &&  \text{$\mathcal{D}$ pdf bounded by $B_P$}\\
            &= O_T(\upsilon) \cdot 1_{K' - K_{D_U}^\theta \le \frac{(|K'|+1)\epsilon}{b}}. \numberthis \label{eq:lowprobexc_pre}
        \end{align*}
        Therefore, the safety truncation $v_{t}^{\mathrm{safeU}}$ is only applied with probability at most $O_T(\upsilon) $ at every time step. By definition, $C(y_t) = v_t^{\mathrm{safeL}}$ if and only if there exists a $\theta' \in \hat{\theta}_{\mathrm{L}\ref{close_to_KDU}} \pm \beta$ such that $y_t \le \frac{D_L}{a'-b'K'}$. This only happens if and only if $y_t \le \frac{D_L}{a_m-b_mK'}$. We also have by Equations \eqref{eq:boundfrom12} and \eqref{eq:boundfrom0} that because $D_L < 0$,
        \begin{equation}\label{eq:DL_to_DUa}
            \left| \frac{D_L}{a_m-b_mK'} - D_L\right| = \frac{|D_L|}{a_m-b_mK'} - |D_L|.
        \end{equation}
        Also by Equations \eqref{eq:boundfrom12} and \eqref{eq:boundfrom0}, we have because $D_U \le |D_L|$ that 
        \begin{equation}\label{eq:DL_to_DUb}
            \frac{|D_L|}{a_m-b_mK'} - |D_L| \ge \frac{D_U}{a_m-b_mK'} - D_U.
        \end{equation}
        Therefore, 
        \begin{align*}
            &\P\left(C(y_{t}) = v_t^{\mathrm{safeL}} \cond H_t\right) \\
            &= \P\left(y_t \le \frac{D_L}{a_m-b_mK'}\cond H_t\right) \\
            &= \P\left(ay_{t-1} + bC(y_{t-1}) + w_{t-1} \le \frac{D_L}{a_m-b_mK'}\cond H_t\right) \\
            &\le \P\left(w_{t-1} \le \frac{D_L}{a_m-b_mK'} - D_L\cond H_t\right) && \text{Lemma \ref{lemma:safety_of_C}} \\
            &\le \P\left(|w_{t-1}| \ge \frac{|D_L|}{a_m-b_mK'} - |D_L|\cond H_t\right) && \text{Equation \eqref{eq:DL_to_DUa}} \\
            & \le \P\left(|w_{t-1}| \ge \frac{D_U}{a_m-b_mK'} - D_U\cond H_t\right)   && \text{Equation \eqref{eq:DL_to_DUb}} \\
            &\le O_T(\upsilon) \cdot  1_{K' - K_{D_U}^\theta \le \frac{(|K'|+1)\epsilon}{b}} \cdot  . && \text{Equation \eqref{eq:lowprobexc_pre}} \numberthis \label{eq:lowprobexc2_pre}
        \end{align*}     
        This is exactly the second result we need and therefore we are done.
        \subsection{Proof of Lemma \ref{lemma:dt_and_control}}\label{app:proof_of_dt_and_control}
        If $\frac{D_L}{a_m-b_mK'} \le y_{t} \le \frac{D_U}{a_m-b_mK'}$, then $C(y_{t}) = -K'y_{t}$, and therefore
        \begin{equation}\label{eq:diff_in_contr_bound_middle}
            \left|C(y_{t}) - C_{K'}^{\mathrm{unc}}(x_{t}) \right| = |K'| d_t
        \end{equation}
        and
        \begin{equation}\label{eq:recurs1}
            d_{t+1} = |ay_t +bC(y_t) + w_t - (ax_t + bC_{K'}^{\mathrm{unc}}(x_t) +w_t)| = (a-bK')d_{t}.
        \end{equation}
        This proves the first case of both equations in Lemma \ref{lemma:dt_and_control}. Now we will prove the second case of both equations.

        Under Equation \eqref{eq:assum_on_K}, we have that for any $\theta' \in \hat{\theta}_{\mathrm{L}\ref{close_to_KDU}}  \pm \beta$
        \begin{align*}
                  \left| (a-bK_{D_U}^\theta) - (a' -b'K') \right| &\le  |a - a'| +b |K' - K_{D_U}^{\theta}| + |K'||b' - b| \\
                  &\le (\beta+bO_T(\upsilon)+|K'|\beta) \\
                  &= O_T(\upsilon). \numberthis \label{eq:a_to_m_abs}
        \end{align*}
        If $y_{t} > \frac{D_U}{a_m - b_mK'}$, then for some $\theta' \in \hat{\theta}_{\mathrm{L}\ref{close_to_KDU}}  \pm \beta$, $C(y_t) = \frac{D_U - a'y_t}{b'}$. Therefore,
        \begin{align*}
            &\left|C(y_{t}) - C_{K'}^{\mathrm{unc}}(y_{t}) \right| \\
            &=\left|C(y_{t}) + K'y_{t} \right| \\
            &= \left|\frac{D_U - a'y_{t}}{b'} + K'y_{t}\right| \\
            &=\frac{1}{b'} \left| D_U - \left(a' - b'K'\right)y_{t}\right| \\
            &=\frac{a'-b'K'}{b'} \left| \frac{D_U}{a' - b'K'} - y_{t}\right| && \text{Equations \eqref{eq:boundfrom12}, \eqref{eq:boundfrom0}} \\
            &\le \frac{a'-b'K'}{b'} \left| \frac{D_U}{a'-b'K'} - (D_U + \bar{w})\right|  &&  \text{$\frac{D_U}{a' - b'K'} \le y_{t} \le D_U + \bar{w}$ by Lemma \ref{lemma:safety_of_C}}\\
            &=  \frac{a'-b'K'}{b'} \left| \frac{D_U}{a'-b'K'} - \frac{D_U}{a-bK_{D_U}^\theta}\right|\\
            &= \frac{D_U}{b'} \left| \frac{(a-bK_{D_U}^\theta) - \left(a' - b'K'\right)}{a-bK_{D_U}^\theta}\right|\\
            &\le \frac{D_U}{b'} \left( \frac{O_T(\upsilon)  }{a-bK_{D_U}^\theta}\right) && \text{Equation \eqref{eq:a_to_m_abs}, Equation \eqref{eq:boundfrom0}}\\
            &=  \frac{(D_U + \bar{w})O_T(\upsilon)}{b'}  && \text{$a-bK_{D_U}^\theta = \frac{D_U}{D_U + \bar{w}}$} \\
            &= O_T(\upsilon). \numberthis \label{eq:diff_in_contr_bound}
        \end{align*}
        Because the controls used by $C$ are safe with respect to $\theta$ by Lemma \ref{lemma:safety_of_C}, if $D_L - \frac{D_L}{a_m - b_mK'} > \bar{w}$, then $\P\left(y_t \le \frac{D_L}{a_m -b_mK'}\right) = 0$. Therefore, if $y_t \le \frac{D_L}{a_m -b_mK'}$ then it also must be the case that $D_L - \frac{D_L}{a_m - b_mK'} \le \bar{w}$. By Equations \eqref{eq:boundfrom12} and \eqref{eq:boundfrom0}, we have that $a_m - b_mK' \le \frac{D_U}{D_U + \bar{w}} + O_T(\upsilon)$ and $a_m - b_mK' \ge \frac{D_U}{D_U + \bar{w}} - O_T(\upsilon)$. Therefore, if  $y_t \le \frac{D_L}{a_m -b_mK'}$, then $D_L - \frac{D_L}{a_m - b_mK'} \le \bar{w}$, which implies that
        \begin{align*}
            D_L &\ge \frac{\bar{w}}{ 1- \frac{1}{a_m - b_mK'}}\\
            &= \bar{w}\frac{a_m- b_mK'}{a_m- b_mK' - 1} \\
            &\ge  \bar{w}\frac{\frac{D_U}{D_U + \bar{w}} - O_T(\upsilon) }{\frac{D_U}{D_U + \bar{w}} + O_T(\upsilon) - 1} \\
            &= \bar{w} \left(\frac{\frac{D_U}{D_U+\bar{w}}}{\frac{-\bar{w}}{D_U+\bar{w}}} - O_T(\upsilon)\right)\\
            &= -D_U - O_T(\upsilon).
        \end{align*}
        This combined with the fact that $D_U \le |D_L|$  by Assumption \ref{assum_thm4}, we have that if $y_t \le \frac{D_L}{a_m -b_mK'}$, then
        \begin{equation}\label{eq:DL_to_DU}
        ||D_L|-D_U| \le O_T(\upsilon).
        \end{equation} 
        Therefore, if $y_t < \frac{D_L}{a_m- b_mK'}$, then for some $\theta' \in \hat{\theta}_{\mathrm{L}\ref{close_to_KDU}}  \pm \beta$, $C(y_t) = \frac{D_L - a'y_t}{b'}$. Therefore, 
        {\fontsize{10}{10}
        \begin{align*}
             &\left|C(y_{t}) - C_{K'}^{\mathrm{unc}}(y_{t}) \right|  \\
             &=\left|C(y_{t}) + K'y_{t} \right| \\
            &= \left|\frac{D_L - a'y_{t}}{b'} + K'y_{t}\right| \\
            &=\frac{1}{b'} \left| D_L - \left(a'- b'K'\right)y_{t}\right| \\
            &\le \frac{1}{b'} \left| D_L - \left(a'- b'K'\right)(D_L - \bar{w})\right|   && \text{$D_L - \bar{w} \le y_t \le \frac{D_L}{a'- b'K'}$, Eq \eqref{eq:boundfrom0}}\\
            &=  \frac{1}{b'} \left| |D_L| - \left(a'- b'K'\right)(|D_L| + \bar{w})\right| \\
            &\le  \frac{1}{b'} \left| D_U - \left(a'- b'K'\right)(D_U + \bar{w})\right| + |D_U - |D_L||\\
            &\quad \quad + |(a' - b'K')(D_U - |D_L|)|\\
            &\le  \frac{1}{b'} \left| D_U - \left(a'- b'K'\right)(D_U + \bar{w})\right| + |D_U - |D_L|| + |D_U - |D_L|| && \text{Equation \eqref{eq:boundfrom0}, \eqref{eq:boundfrom12}}\\
            &\le  \frac{1}{b'} \left| D_U - \left(a'- b'K'\right)(D_U + \bar{w})\right| + O_T(\upsilon) && \text{Equation \eqref{eq:DL_to_DU}}\\
            &\le  \frac{\left(a'- b'K'\right)}{b'} \left| \frac{D_U}{\left(a'- b'K'\right)} - (D_U + \bar{w})\right| + O_T(\upsilon) \\
            &= O_T(\upsilon). && \text{As in Equation \eqref{eq:diff_in_contr_bound}} \numberthis \label{eq:diff_in_contr_bound2}        
            \end{align*}
            }
        Combining Equations \eqref{eq:diff_in_contr_bound} and \eqref{eq:diff_in_contr_bound2} gives that if $y_{t} > \frac{D_U}{a_m-b_mK'} \text{ or } y_{t} < \frac{D_L}{a_m-b_mK'}$,
        \begin{align*}
            \left|C_{K'}^{\mathrm{unc}}(x_t) - C(y_t)\right| &= \left|-K'x_t + K'y_t - K'y_t -  C(y_t)\right| \\
            &=  \left|K'x_t - K'y_t\right|  + \left|K'y_t +  C(y_t)\right|  \\
            &\le K'd_t + O_T(\upsilon) \\
            &\le O_T(\upsilon). && \text{Equation \eqref{eq:dt_bound}} \numberthis \label{eq:diff_in_controls_K_proof}
        \end{align*}
        
         Now we can use this to bound the value of $d_{t+1}$  as follows:
        \begin{align*}
            d_{t+1}  &= |(a-bK')x_{t} - (ay_{t} - bC(y_{t_1})| \\
            &= |(a-bK')x_{t} - (a - bK')y_{t} +  bK'y_{t}  - bC(y_{t})| \\
            &\le |(a-bK')x_{t} - (a - bK')y_{t}| + |bK'y_{t}  - bC(y_{t})|   \\
            &\le (a-bK')d_{t} + bO_T(\upsilon) & \text{Equations \eqref{eq:diff_in_contr_bound} and \eqref{eq:diff_in_contr_bound2}} \\
            &\le (a-bK')d_{t} + O_T(\upsilon).  \numberthis \label{eq:recurs2}
        \end{align*}

        Equations \eqref{eq:diff_in_controls_K_proof} and \eqref{eq:recurs2} give the second half of both desired piecewise equations.
\subsection{Proof of Lemma \ref{lemma:safety_of_C}}\label{app:proof_of_safety_of_C}
    \begin{proof}
        We will proceed by induction. For the base case, we have that $y_0 = x_0$ satisfies $|y_0| \le \norm{D}_{\infty} + \bar{w}$. Define $z := \frac{D_U - ay_0 - 2\beta(\norm{D}_{\infty} + \bar{w} + \log(T))}{b}$. For sufficiently large $T$, because $\beta \le 1/\log^2(T)$ and $\norm{D}_{\infty} = O_T(1)$, we have that
        \[
            |z| \le \frac{D_U + a(\norm{D}_{\infty} + \bar{w})+2\beta(\norm{D}_{\infty} + \bar{w} + \log(T))}{b}  \le \frac{D_U + a(\norm{D}_{\infty} + \bar{w})+\frac{2(\norm{D}_{\infty} + \bar{w} + \log(T))}{\log^2(T)}}{b} \le \log(T).
        \]
        Because $\theta \in \hat{\theta}_{\mathrm{L}\ref{close_to_KDU}}  \pm \beta$, 
        \begin{align*}
            \max_{\theta' \in \hat{\theta}_{\mathrm{L}\ref{close_to_KDU}}  \pm \beta} a'y_0 + b'z &\le ay_0 + bz + 2\beta |y_0| + 2\beta|z| \\
            &\le ay_0 + bz + 2\beta(\norm{D}_{\infty} + \bar{w} + \log(T))\\
            &=  ay_0 + D_U - ay_0 - 2\beta(\norm{D}_{\infty} + \bar{w} + \log(T)) + 2\beta(\norm{D}_{\infty} + \bar{w} + \log(T)) \\
            &= D_U.
        \end{align*}
        Therefore, 
        \[
            v_t^{\mathrm{safeU}} \ge z = \frac{D_U - ay_0 - 2\beta(D_U + \bar{w} + \log(T))}{b}.
        \]
        By similar logic, we have that
        \[
            v_t^{\mathrm{safeL}} \le \frac{D_L - ay_0 + 2\beta(\norm{D}_{\infty} + \bar{w} + \log(T))}{b}.
        \]
        For sufficiently large $T$, $4\beta(\norm{D}_{\infty} + \bar{w} + \log(T)) \le \frac{1}{\log(T)}$. Therefore, because $D_U \ge D_L + \frac{1}{\log(T)}$, we have that 
        \[
            v_t^{\mathrm{safeL}} \le v_t^{\mathrm{safeU}}.
        \]
        Finally, this implies by construction of the controller $C$ that the control $C(y_0)$ will be safe for all $\theta' \in \hat{\theta}_{\mathrm{L}\ref{close_to_KDU}}  \pm \beta$. This completes the base case. 

        For the inductive step, we note that if $C(y_{t-1})$ is safe for all $\theta' \in \hat{\theta}_{\mathrm{L}\ref{close_to_KDU}}  \pm \beta$, then it is safe for $\theta$. This implies that $D_L \le ay_{t-1} + bC(y_{t-1}) \le D_U$, which implies that $|y_t| \le \norm{D}_{\infty} + \bar{w}$. We can therefore use the exact same logic as in the base case to get that $C(y_t)$ will be safe for all $\theta' \in \hat{\theta}_{\mathrm{L}\ref{close_to_KDU}}  \pm \beta$. This completes the proof by induction.
    \end{proof}
\section{Proofs from Appendix \ref{sec:proof_of_prop_Kcase}}
\subsection{Proof of Proposition \ref{close_J2}}\label{sec:proof_of_close_J2} 
\begin{proof}
    By Lemma \ref{j_bounded_from_0}, $a^*-b^*F_{\mathrm{opt}}(\theta^*) < 1 - c_{\mathrm{L}\ref{j_bounded_from_0}}$, which implies by Lemma \ref{lemma:convexity} that $\bar{J}(\theta^*, C_F^{\mathrm{unc}})$ is twice differentiable at the point $F = F_{\mathrm{opt}}(\theta^*)$ with first and second derivatives that are both finite and independent of $T$.  We also have by Lemma \ref{close_J} that  $|F_{\mathrm{opt}}(\hat{\theta}_{\mathrm{wu}}) - F_{\mathrm{opt}}(\theta^*)| \le \tilde{O}_T(T^{-1/4})$ conditional on event $E_2^0$. Therefore,  conditional on event $E_2^0$ and for sufficiently large $T$, we can do a second order Taylor expansion of $\bar{J}(\theta^*, C_F^{\mathrm{unc}})$ around $F = F_{\mathrm{opt}}(\theta^*)$ to get that
    \begin{equation}\label{close_J2a}
        \left|T \cdot \bar{J}\left(\theta^*, C_{F_{\mathrm{opt}}(\hat{\theta}_{\mathrm{wu}})}^{\mathrm{unc}}\right) -  T \cdot \bar{J}\left(\theta^*, C_{F_{\mathrm{opt}}(\theta^*)}^{\mathrm{unc}}\right)\right| = \tilde{O}_T(\sqrt{T}).
    \end{equation}
    Because the lowest-cost unconstrained linear controller $C_{F_{\mathrm{opt}}(\theta^*)}^{\mathrm{unc}}$ has the lowest cost among all unconstrained controllers \cite{anderson2007optimal}, 
    \begin{equation}\label{close_J2b}
        T \cdot \bar{J}(\theta^*, C_{F_{\mathrm{opt}}(\theta^*)}^{\mathrm{unc}}) -  T \cdot \bar{J}(\theta^*, C_{K_{\mathrm{opt}}(\theta^*, T)}^{\theta^*}) \le 0.
    \end{equation}
    Combining Equations \eqref{close_J2a} and \eqref{close_J2b} and multiplying by $(T-T_0)/T$, we have
    \begin{equation}\label{close_J2c}
      (T-T_0) \cdot \bar{J}(\theta^*, C_{F_{\mathrm{opt}}(\hat{\theta}_{\mathrm{wu}})}^{\mathrm{unc}}) -  (T-T_0) \cdot \bar{J}(\theta^*, C_{K_{\mathrm{opt}}(\theta^*, T)}^{\theta^*}) = \tilde{O}_T(\sqrt{T}).
    \end{equation}
    Now we just need to convert this to a result about finite time cost rather than infinite cost which requires the following lemma.    
    \begin{lemma}\label{split2}
         Under Assumptions \ref{assum:constraints}--\ref{assum:initial} and \ref{assum_thm4}, for any $\theta \in \Theta$ and $K$ satisfying $1-(a-bK) = \epsilon > 0$ for some $\epsilon = \Omega_T(1)$,
        \[
            |\bar{J}(\theta, C_K^{\mathrm{unc}}, T) - \bar{J}(\theta, C_K^{\mathrm{unc}})|  = O_T\left(\frac{1}{T}\right).
        \]
    \end{lemma}
           The proof of Lemma \ref{split2} can be found in Appendix \ref{sec:proof_of_split2}.

    For sufficiently large $T$, conditional on event $E_2^0$,
    \begin{align*}
      1 -  ( a^*-b^*F_{\mathrm{opt}}(\hat{\theta}_{\mathrm{wu}})) &\ge 1  - \left(a^*-b^*F_{\mathrm{opt}}(\theta^*) - \tilde{O}_T(T^{-1/4})  \right)&& \text{Lemma \ref{close_J}} \\
        &>c_{\mathrm{L}\ref{j_bounded_from_0}}/2. && \text{Lemma \ref{j_bounded_from_0}}
    \end{align*}
    Therefore, we can apply Lemmas \ref{split} and \ref{split2} to Equation \eqref{close_J2c} to get the desired result that conditional on event $E_2^0$,
    \[
      (T-T_0) \cdot \bar{J}(\theta^*, C_{F_{\mathrm{opt}}(\hat{\theta}_{\mathrm{wu}})}^{\mathrm{unc}}, T-T_0) -  (T-T_0) \cdot \bar{J}(\theta^*, C_{K_{\mathrm{opt}}(\theta^*, T)}^{\theta^*}, T-T_0) = \tilde{O}_T(\sqrt{T}).
    \]
\end{proof}
\subsection{Proof of Proposition \ref{close_safe_J}}\label{sec:proof_of_close_safe_J} 
\begin{proof}
We will apply the standard McDiarmid's inequality to the function 
\[
f(\{w_{t}\}_{t=T_0}^{T-1}) = (T-T_0)J(\theta^*, C_{F_{\mathrm{opt}}(\hat{\theta}_{\mathrm{wu}})}^{\mathrm{unc}}, T-T_0, 0, W').
\]
To do this, we need a bounded difference inequality. We will show the following.

\begin{lemma}\label{lemma:mcdiarmid_standard}
    For $i \in [T_0:T-1]$, let $\{w'_{t}\}_{t=T_0}^{T-1}$ be such that $w'_t = w_t$ for $t \ne i$ and $w'_i \sim \mathcal{D}$ is independent of $\{w_{t}\}_{t=T_0}^{T-1}$. If $|F_{\mathrm{opt}}(\hat{\theta}_{\mathrm{wu}}) - F_{\mathrm{opt}}(\theta^*)| \le \tilde{O}_T(T^{-1/4})$, then for sufficiently large $T$,
    \[
        |(T-T_0) \cdot J(\theta^*, C_{F_{\mathrm{opt}}(\hat{\theta}_{\mathrm{wu}})}^{\mathrm{unc}}, T-T_0, 0, \{w_{t}\}_{t=T_0}^{T-1}) - (T-T_0) \cdot J(\theta^*, C_{F_{\mathrm{opt}}(\hat{\theta}_{\mathrm{wu}})}^{\mathrm{unc}}, T-T_0,0, \{w'_{t}\}_{t=T_0}^{T-1})| \le c.
    \]
    for some $c = \tilde{O}_T(1)$.
\end{lemma}
The proof of Lemma \ref{lemma:mcdiarmid_standard} can be found in Appendix \ref{sec:proof_of_lemma:mcdiarmid_standard}.

Under event $E_2^0$, by Lemma \ref{close_J} we have that $|F_{\mathrm{opt}}(\hat{\theta}_{\mathrm{wu}}) - F_{\mathrm{opt}}(\theta^*)| \le \tilde{O}_T(T^{-1/4})$. Furthermore, conditional on $E_2^0$ and $\hat{\theta}_{\mathrm{wu}}$ the random variables $\{w_{t}\}_{t=T_0}^{T-1}$ are still i.i.d. because the noise random variables are independent of the history. Therefore, conditional on event $E_2^0$, we can use Lemma \ref{lemma:mcdiarmid_standard} with the standard McDiarmid's inequality \cite{mcdiarmid1989method} and get
{\fontsize{10}{10}
\begin{align*}
    &\P\left(|(T-T_0) \cdot J(\theta^*, C_{F_{\mathrm{opt}}(\hat{\theta}_{\mathrm{wu}})}^{\mathrm{unc}}, T-T_0, 0, \{w_{t}\}_{t=T_0}^{T-1}) - \E[(T-T_0) \cdot J(\theta^*, C_{F_{\mathrm{opt}}(\hat{\theta}_{\mathrm{wu}})}^{\mathrm{unc}}, T-T_0, 0, \{w_{t}\}_{t=T_0}^{T-1})]|  \ge \epsilon \cond \hat{\theta}_{\mathrm{wu}} \right) \\
    &\le 2\exp\left(-2\frac{\epsilon^2}{c^2(T-T_0)}\right).
\end{align*}
}
Because 
\[
\E[(T-T_0) \cdot J(\theta^*, C_{F_{\mathrm{opt}}(\hat{\theta}_{\mathrm{wu}})}^{\mathrm{unc}}, T-T_0, 0, \{w_{t}\}_{t=T_0}^{T-1})] = (T-T_0) \cdot \bar{J}(\theta^*, C_{F_{\mathrm{opt}}(\hat{\theta}_{\mathrm{wu}})}^{\mathrm{unc}}, T-T_0),
\]
taking $\epsilon = \sqrt{T}c\log(T)$ gives conditional on $E_2^0$, 
{\fontsize{10}{10}
\begin{align*}
    &\P\left(|(T-T_0) \cdot J(\theta^*, C_{F_{\mathrm{opt}}(\hat{\theta}_{\mathrm{wu}})}^{\mathrm{unc}}, T-T_0, 0, \{w_{t}\}_{t=T_0}^{T-1}) - T-T_0) \cdot \bar{J}(\theta^*, C_{F_{\mathrm{opt}}(\hat{\theta}_{\mathrm{wu}})}^{\mathrm{unc}}, T-T_0)|  \ge \sqrt{T}c\log(T) \cond \hat{\theta}_{\mathrm{wu}} \right) \\
    &= o_T(1/T). \numberthis \label{eq:E20csqrtT}
\end{align*}
}
Define 
{\fontsize{10}{10}
\[
    E_{\mathrm{P}\ref{close_safe_J}} := \left\{ |(T-T_0) \cdot J(\theta^*, C_{F_{\mathrm{opt}}(\hat{\theta}_{\mathrm{wu}})}^{\mathrm{unc}}, T-T_0, 0, \{w_{t}\}_{t=T_0}^{T-1}) - T-T_0) \cdot \bar{J}(\theta^*, C_{F_{\mathrm{opt}}(\hat{\theta}_{\mathrm{wu}})}^{\mathrm{unc}}, T-T_0)|  < \sqrt{T}c\log(T) \right\}.
\]
}
By the law of total expectation, Equation \eqref{eq:E20csqrtT} implies that
\[
    \P(E_{\mathrm{P}\ref{close_safe_J}} \mid E_2^0) = \E[\P(E_{\mathrm{P}\ref{close_safe_J}} \mid \hat{\theta}_{\mathrm{wu}}, E_2^0) \mid E_2^0] \ge 1 - o_T(1/T).
\]

Because $\P(E_2^0) \ge \P(E) = 1-o_T(1/T)$, we therefore have that 
\[
\P(E_{\mathrm{P}\ref{close_safe_J}}) \ge \P(E_{\mathrm{P}\ref{close_safe_J}} \mid E_2^0)\P(E_2^0) = 1 - o_T(1/T)
\]
as desired.

\end{proof}

\subsection{Proof of Proposition \ref{F_start}}\label{sec:proof_of_F_start} 
\begin{proof}   
    \begin{lemma}\label{unc_start_invariant}
        Under Assumptions \ref{assum:constraints}--\ref{assum:initial} and \ref{assum_thm4}, for any $\theta \in \Theta$ and any $K \in [\frac{a-1}{b}, \frac{a}{b}]$, when using controller $C_K^{\mathrm{unc}}$ under dynamics $\theta$ where $1 - (a-bK) = \epsilon = \Omega_T(1)$ and starting at state $x_0 = x$, then for all $t$, the state $x_t$ at time $t$ satisfies
        \[
            |x_t| \le |x| + \frac{\bar{w}}{\epsilon}.
        \]
        Furthermore, for any $x,y,W'$ and $\tau \le T$,
        \begin{align*}
            |\tau J(\theta, C_K^{\mathrm{unc}}, \tau, x, W') - \tau J(\theta, C_K^{\mathrm{unc}}, \tau, y, W')| &\le \frac{(q+rK^2)(x-y)^2+2(q+rK^2)\left(|x| + \frac{\bar{w}}{\epsilon}\right)|x-y|}{\epsilon } \\
            &= O_T\left((x-y)^2 + |x(x-y)|\right).
        \end{align*}
    \end{lemma}
           The proof of Lemma \ref{unc_start_invariant} can be found in Appendix \ref{sec:proof_of_unc_start_invariant}.
           
    By Lemma \ref{close_J}, under event $E \subseteq E_2^0$, we have $|F_{\mathrm{opt}}(\hat{\theta}_{\mathrm{wu}}) - F_{\mathrm{opt}}(\theta^*)| \le O_T(T^{-1/4})$. Therefore, by Lemma \ref{j_bounded_from_0}, under event $E$ and for large enough $T$, 
    \[
    1 - (a^*-b^*F_{\mathrm{opt}}(\hat{\theta}_{\mathrm{wu}})) \ge c_{\mathrm{L}\ref{j_bounded_from_0}} - b^*|F_{\mathrm{opt}}(\hat{\theta}_{\mathrm{wu}}) - F_{\mathrm{opt}}(\theta^*)| \ge c_{\mathrm{L}\ref{j_bounded_from_0}}/2.
    \]
    Conditional on event $E$, $C^{\mathrm{alg}}$ is safe for dynamics $\theta^*$, and therefore by Lemma \ref{bounded_approx}, the  state of $C^{\mathrm{alg}}$ at time $T_0$ satisfies $|x'_{T_0}| \le B_x = \tilde{O}_T(1)$ conditional on $E$. Therefore, by Lemma \ref{unc_start_invariant}, conditional on $E$,
    \[
        (T-T_0)\cdot J(\theta^*, C_{F_{\mathrm{opt}}(\hat{\theta}_{\mathrm{wu}})}^{\mathrm{unc}}, T-T_0, x'_{T_0}, \{w_{t}\}_{t=T_0}^{T-1}) - (T-T_0) \cdot J(\theta^*, C_{F_{\mathrm{opt}}(\hat{\theta}_{\mathrm{wu}})}^{\mathrm{unc}}, T-T_0, 0, \{w_{t}\}_{t=T_0}^{T-1}) = \tilde{O}_T(1).
    \]
\end{proof}

\subsection{Proof of Proposition \ref{safety_is_cheap}}\label{sec:proof_of_safety_is_cheap} 
\begin{proof}
    Under event $E_2^0$, we have that for sufficiently large $T$,
    \begin{equation}\label{eq:bound_from_0_hat}
    \hat{a}-\hat{b}F_{\mathrm{opt}}(\hat{\theta}_{\mathrm{wu}}) \ge a^*-b^*F_{\mathrm{opt}}(\theta^*) - \tilde{O}_T(T^{-1/4}) > 0
    \end{equation}
    by Lemma \ref{j_bounded_from_0} and Lemma \ref{close_J}. Conditional on event $E_2^0 \cap E_{\mathrm{E}\ref{eq:first_case}}$ and for sufficiently large $T$ we have the following result:
    \begin{align*}
        &\tilde{O}_T(T^{-1/4}) \\
        &\ge D_U + \bar{w}  - \frac{D_U}{\hat{a}- \hat{b}F_{\mathrm{opt}}(\hat{\theta}_{\mathrm{wu}})} && \text{Equation \eqref{eq:first_case}} \\
        &= \frac{D_U}{a^* -b^*K_{D_U}^{\theta^*}} -  \frac{D_U}{\hat{a}- \hat{b}F_{\mathrm{opt}}(\hat{\theta}_{\mathrm{wu}})} && \text{Definition of $K_{D_U}^{\theta^*}$}\\
        &\ge \frac{D_U}{a^* -b^*K_{D_U}^{\theta^*}} - \frac{D_U}{a^* - b^*F_{\mathrm{opt}}(\theta^*) - \tilde{O}_T(T^{-1/4})}  && \text{Equation \eqref{eq:bound_from_0_hat}} \\
        &= \frac{-D_U\tilde{O}_T(T^{-1/4})+b^*D_U(K_{D_U}^{\theta^*}-F_{\mathrm{opt}}(\theta^*))}{(a^* - b^*F_{\mathrm{opt}}(\theta^*) - \tilde{O}_T(T^{-1/4}))(a^* - b^*K_{D_U}^{\theta^*})} \\
        &\ge \frac{-D_U\tilde{O}_T(T^{-1/4})+b^*D_U(K_{D_U}^{\theta^*}-F_{\mathrm{opt}}(\theta^*))}{(a^* - b^*F_{\mathrm{opt}}(\theta^*))(a^* - b^*K_{D_U}^{\theta^*})}   \\
        &= (K_{D_U}^{\theta^*}-F_{\mathrm{opt}}(\theta^*))\frac{b^*D_U}{(a^* - b^*F_{\mathrm{opt}}(\theta^*))(a^* - b^*K_{D_U}^{\theta^*})} \\
        &\quad \quad - \frac{D_U\tilde{O}_T(T^{-1/4})}{(a^* - b^*F_{\mathrm{opt}}(\theta^*))(a^* - b^*K_{D_U}^{\theta^*})}. \numberthis \label{eq:hatF_to_Kdu}
    \end{align*}
    Because $\theta^*, D_U, F_{\mathrm{opt}}(\theta^*), K_{D_U}^{\theta^*}$ are all independent of $T$, we can rearrange Equation \eqref{eq:hatF_to_Kdu} to get
    \[
        K_{D_U}^{\theta^*} - F_{\mathrm{opt}}(\theta^*) \le  \tilde{O}_T(T^{-1/4}).
    \]
    Combining this with Lemma \ref{close_J} which states that $|F_{\mathrm{opt}}(\hat{\theta}_{\mathrm{wu}}) - F_{\mathrm{opt}}(\theta^*)| = \tilde{O}_T(T^{-1/4})$ we have that
    \begin{equation}\label{eq:def_of_epsilonstar}
         K_{D_U}^{\theta^*} - F_{\mathrm{opt}}(\hat{\theta}_{\mathrm{wu}}) \le \tilde{O}_T(T^{-1/4}).
    \end{equation}
    Conditional on event $E_2^0 \cap E_{\mathrm{safe}}^{\mathrm{wu}}$, $\norm{\hat{\theta}_{\mathrm{wu}}- \theta^*}_{\infty} \le \epsilon_0 = \tilde{O}_T(T^{-1/4})$ and $|x'_{T_0}| \le \norm{D}_{\infty} + \bar{w}$. Conditional on $E_2^0 \cap E_{\mathrm{safe}}^{\mathrm{wu}}$, we can apply Lemma \ref{close_to_KDU} with $\theta = \theta^*$, $\hat{\theta}_{\mathrm{L}\ref{close_to_KDU}} = \hat{\theta}_{\mathrm{wu}}$, $K' = F_{\mathrm{opt}}(\hat{\theta}_{\mathrm{wu}})$, $\epsilon$ as the right hand side of Equation \eqref{eq:def_of_epsilonstar}, $\beta = \epsilon_0$, $\tau = T-T_0$, and $x_0 = x'_{T_0}$. With this choice of parameters, the controller $C$ in Lemma \ref{close_to_KDU} is exactly equivalent to $C^{\mathrm{alg}'}$ under event $E_{\mathrm{E}\ref{eq:first_case}}$. Conditional on $E_2^0$, $\epsilon$ and $\beta$ satisfy the necessary inequality for Lemma \ref{close_to_KDU} as both are $\tilde{O}_T(T^{-1/4})$. 
    
   The event $E_2^0 \cap E_{\mathrm{safe}}^{\mathrm{wu}} \cap E_{\mathrm{E}\ref{eq:first_case}}$ depends only on noise random variables before time $T_0$, which means we can apply Lemma \ref{close_to_KDU} conditional on these events. Equation \eqref{eq:close_to_KDU_second_eq} of Lemma \ref{close_to_KDU} gives that for sufficiently large $T$,  conditional on $E_2^0 \cap E_{\mathrm{safe}}^{\mathrm{wu}} \cap E_{\mathrm{E}\ref{eq:first_case}}$, and with conditional probability $1-o_T(1/T)$,
    \begin{align*}
        &\left|(T-T_0) \cdot J(\theta^*, C^{\mathrm{alg}'}, T-T_0, x'_{T_0}, W') - (T-T_0) \cdot J(\theta^*, C_{K'}^{\mathrm{unc}}, T-T_0, x'_{T_0}, W') \right| \\
        &\le (T-T_0) O_T \Big((b\epsilon + \epsilon_0 + |F_{\mathrm{opt}}(\hat{\theta}_{\mathrm{wu}})|\epsilon_0) \log\left(\frac{1}{b\epsilon  + \epsilon_0 + |F_{\mathrm{opt}}(\hat{\theta}_{\mathrm{wu}})|\epsilon_0}\right) \\
        &\quad \quad \times \left((b\epsilon  + \epsilon_0 + |F_{\mathrm{opt}}(\hat{\theta}_{\mathrm{wu}})|\epsilon_0) + \frac{\log(T)}{\sqrt{T-T_0}}\right)\Big) \\
        &\le (T-T_0) O_T \left(\tilde{O}_T(T^{-1/4})\log(1/\tilde{\Omega}_T(T^{-1/4}))\left(\tilde{O}_T(T^{-1/4}) + \frac{\log(T)}{\sqrt{T-T_0}}\right)\right) \quad \quad \quad \quad \quad \quad \text{[$\epsilon_0 = \tilde{\Omega}_T(1)$]}\\
        &= \tilde{O}_T(\sqrt{T}). \numberthis \label{eq:def_of_E24} 
    \end{align*}
    Taking $E_{\mathrm{P}\ref{safety_is_cheap}}$ to be the event that Equation \eqref{eq:def_of_E24} holds gives the desired result that $\P(E_{\mathrm{P}\ref{safety_is_cheap}} \mid E_2^0 \cap E_{\mathrm{safe}}^{\mathrm{wu}} \cap E_{\mathrm{E}\ref{eq:first_case}})  = 1-o_T(1/T)$.
\end{proof}

\subsection{Proof of Lemma \ref{split2}}\label{sec:proof_of_split2}
\begin{proof}
    By Lemma \ref{unc_start_invariant}, when starting at $x_0 = 0$ and using controller $C_K^{\mathrm{unc}}$ we have that 
    \begin{equation}\label{eq:xT_bound}
        |x_T| \le \frac{\bar{w}}{\epsilon}.
    \end{equation}
    Therefore, we can conclude that (for $W' = \{w_t\}_{t=0}^{T-1}$):
    \begin{align*}
        &\left| \bar{J}(\theta, C_K^{\mathrm{unc}}, 2T) - \bar{J}(\theta, C_K^{\mathrm{unc}},T)\right| \\
        &= \left| \frac{T\cdot \bar{J}(\theta, C_K^{\mathrm{unc}}, T) + T\cdot \E\left[\bar{J}(\theta,C_K^{\mathrm{unc}},T, x_T)\right]}{2T} - \bar{J}(\theta, C_K^{\mathrm{unc}},T)\right| \\
        &=\left| \frac{ \E\left[\bar{J}(\theta,C_K^{\mathrm{unc}},T, x_T)\right]}{2} - \frac{1}{2}\bar{J}(\theta, C_K^{\mathrm{unc}},T)\right| \\
        &=\frac{1}{2T}\left| \E\left[T\bar{J}(\theta,C_K^{\mathrm{unc}},T, x_T)\right] - T\bar{J}(\theta, C_K^{\mathrm{unc}},T)\right|  \\
        &=\frac{1}{2T}\left| \E\Big[\E\big[TJ(\theta,C_K^{\mathrm{unc}},T, x_T, W') - TJ(\theta, C_K^{\mathrm{unc}},T,0,W') \cond x_T \big]\Big]\right|  \\
        &=\frac{1}{2T}\left| \E\Big[O_T\left(x_T^2\right)\Big]\right|  && \text{Lemma \ref{unc_start_invariant}} \\
        &= O_T\left(\frac{1}{T}\E[x_T^2] \right) \\
        &= O_T\left( \frac{1}{T} \E\left[\frac{\bar{w}^2}{\epsilon^2}\right]\right)  && \text{Equation \eqref{eq:xT_bound}} \\
        &= O_T\left(\frac{1}{T}\right).
    \end{align*}

    Furthermore, we have that
    \begin{align*}
         |\bar{J}(\theta, C_K^{\mathrm{unc}}, T) -\bar{J}(\theta, C_K^{\mathrm{unc}})| &=  \left| \sum_{i=0}^{\infty} \bar{J}(\theta, C_K^{\mathrm{unc}}, 2^{i}T) - \bar{J}(\theta, C_K^{\mathrm{unc}}, 2^{i+1}T) \right| \\
         &\le \sum_{i=0}^{\infty}  \left|  \bar{J}(\theta, C_K^{\mathrm{unc}}, 2^{i}T) -  \bar{J}(\theta, C_K^{\mathrm{unc}}, 2^{i+1}T) \right| \\
         &= \sum_{i=0}^{\infty} O_T \left(\frac{1}{T2^i}\right) \\
         &= O_T \left(\frac{1}{T}\right).
    \end{align*}
\end{proof}
\subsection{Proof of Lemma \ref{lemma:mcdiarmid_standard}}\label{sec:proof_of_lemma:mcdiarmid_standard} 
\begin{proof}
    Define $x_{T_0},..,x_{T}$ as the states with noise $\{w_{t}\}_{t=T_0}^{T-1}$ when using controller $C^{\mathrm{unc}}_{F_{\mathrm{opt}}(\hat{\theta}_{\mathrm{wu}})}$ starting at $x_{T_0} = 0$ and define $y_{T_0},...,y_T$ as the states with noise $\{w'_{t}\}_{t=T_0}^{T-1}$ when using controller $C^{\mathrm{unc}}_{F_{\mathrm{opt}}(\hat{\theta}_{\mathrm{wu}})}$starting at $y_{T_0} = 0$. By construction, the cost up until time $i$ is the same for both trajectories. At time $i+1$, we have that
    \begin{equation}\label{eq:bound_diff_in_x_y}
        |y_{i+1} - x_{i+1}| = |w_i - w_i'| \le 2\bar{w} = O_T(1).
    \end{equation}
    The remaining difference in cost is simply the difference in cost of two length $T' = T-i-1$ trajectories using controller $C_{F_{\mathrm{opt}}(\hat{\theta}_{\mathrm{wu}})}^{\mathrm{unc}}$ starting at states $y_{i+1}$ and $x_{i+1}$ respectively. By the assumption of this lemma on $F_{\mathrm{opt}}(\hat{\theta}_{\mathrm{wu}})$ and Lemma \ref{j_bounded_from_0}, we have that for sufficiently large $T$,
    \[
    1-(a^*-b^*F_{\mathrm{opt}}(\hat{\theta}_{\mathrm{wu}})) \ge 1-(a^*-b^*F_{\mathrm{opt}}(\theta^*)) - \tilde{O}_T(T^{-1/4}) \ge c_{\mathrm{L}\ref{j_bounded_from_0}}/2.
    \]
    Therefore we can combine Lemma \ref{unc_start_invariant} and Equation \eqref{eq:bound_diff_in_x_y} to get that the difference in the cost from time $i+1$ onward is upper bounded by
    \begin{equation}\label{eq:app_of_start_invariant}
       |T' \cdot J(\theta^*, C_{F_{\mathrm{opt}}(\hat{\theta}_{\mathrm{wu}})}^{\mathrm{unc}},  T', x_{i+1}, \{w_{t}\}_{t=i+1}^{T-1}) - T' \cdot J(\theta^*, C_{F_{\mathrm{opt}}(\hat{\theta}_{\mathrm{wu}})}^{\mathrm{unc}}, T', y_{i+1}, \{w_{t}\}_{t=i+1}^{T-1})| = O_T(1).
    \end{equation}
    Therefore, we have that (see below for justification)
    \begin{align*}
        &|(T-T_0) J(\theta^*, C_{F_{\mathrm{opt}}(\hat{\theta}_{\mathrm{wu}})}^{\mathrm{unc}}, T-T_0,0, \{w_{t}\}_{t=T_0}^{T-1}) - (T-T_0) J(\theta^*, C_{F_{\mathrm{opt}}(\hat{\theta}_{\mathrm{wu}})}^{\mathrm{unc}}, T-T_0,0, \{w'_{t}\}_{t=T_0}^{T-1})| \\
        &=  |(i- T_0) J(\theta^*, C_{F_{\mathrm{opt}}(\hat{\theta}_{\mathrm{wu}})}^{\mathrm{unc}}, i-T_0, \{w_{t}\}_{t=T_0}^{i-1})  - (i-T_0) J(\theta^*, C_{F_{\mathrm{opt}}(\hat{\theta}_{\mathrm{wu}})}^{\mathrm{unc}}, i - T_0, \{w'_{t}\}_{t=T_0}^{i-1})| +\\
        & \quad \quad \quad \quad  |T' J(\theta^*, C_{F_{\mathrm{opt}}(\hat{\theta}_{\mathrm{wu}})}^{\mathrm{unc}}, T', x_{i+1}, \{w_{t}\}_{t=i+1}^{T-1}) - T' J(\theta^*, C_{F_{\mathrm{opt}}(\hat{\theta}_{\mathrm{wu}})}^{\mathrm{unc}}, T', y_{i+1}, \{w_{t}\}_{t=i+1}^{T-1})| \\
        &=  |J(\theta^*, C_{F_{\mathrm{opt}}(\hat{\theta}_{\mathrm{wu}})}^{\mathrm{unc}}, T-i-1, x_{i+1}, \{w_{t}\}_{t=i+1}^{T-1}) - J(\theta^*, C_{F_{\mathrm{opt}}(\hat{\theta}_{\mathrm{wu}})}^{\mathrm{unc}}, T-i-1, y_{i+1}, \{w_{t}\}_{t=i+1}^{T-1})| \\
        &= O_T(1).
    \end{align*}
    Note that in the first equality we also cancelled out the controls at time $i$ which are the same for both trajectories. In the second equality, we used the fact that $\{w_{t}\}_{t=T_0}^{i-1} = \{w'_{t}\}_{t=T_0}^{i-1}$, and in the final line we used Equation \eqref{eq:app_of_start_invariant}.
\end{proof}

  \subsection{Proof of Lemma \ref{unc_start_invariant}}\label{sec:proof_of_unc_start_invariant} 
    \begin{proof}
        
    By construction, when using $C_K^{\mathrm{unc}}$ we have the recursive relationship that $x_t = (a-bK)x_{t-1} + w_{t-1}$. Because we assume that $a-bK = 1-\epsilon < 1$, we have that
    \[
        |x_t| \le |x| + \sum_{i=0}^\infty (a-bK)^i\bar{w} = |x| + \frac{\bar{w}}{1-(a-bK)} = |x| + \frac{\bar{w}}{\epsilon} = \beta,
    \]
    where we define $\beta =  |x| + \frac{\bar{w}}{\epsilon}$.
    This proves the first part of the lemma. Furthermore, this implies that the magnitude of the control is never greater than
    \[
        |u_t| = |K||x_t| \le |K|\beta.
    \]
    Using controller $C_K^{\mathrm{unc}}$, let $x_0,x_1,...,x_T$ be the sequence of states starting at $x_0 = x$ and let $y_0,y_1,...,y_T$ be the series of states starting at $y_0 = y$. Define $d_t = |x_t - y_t|$. Note that $d_0 = |x-y|$. Furthermore, for all $t$,
    \[
        d_t = (a-bK)d_{t-1}.
    \]
    and
    \begin{align*}
        |C_K^{\mathrm{unc}}(x_t) - C_K^{\mathrm{unc}}(y_t)| = Kd_t.
    \end{align*}
    Therefore, we have the following bound.
    \begin{align*}
        &\left|  TJ(\theta, C_K^{\mathrm{unc}},T, x, W') - T J(\theta, C_K^{\mathrm{unc}},T, y, W')\right| \\
        &= \left|(qx_T^2 - qy_T^2) + \sum_{t=0}^{T-1} qx_t^2 - qy_t^2 + r(Kx_t)^2 - r(Ky_t)^2 \right| \\
        &\le \sum_{t=0}^{T} 2q|x_t|d_t + qd_t^2 + 2r|Kx_t||Kd_t| + rK^2d_t^2 \\
        &\le (2q + 2rK^2)\beta \sum_{t=0}^{T} d_t + (q+rK^2)\sum_{t=0}^{T} d_t^2  && \text{$|x_t| \le \beta$}\\
        &\le (2q + 2rK^2)\beta \sum_{t=0}^{\infty}  (a-bK)^{t}d_0 + (q+rK^2)\sum_{t=0}^{\infty} (a-bK)^{2t}d_0^2\\
        &= 2(q+rK^2)\beta\frac{|x-y|}{1-(a-bK)} + \frac{(q+rK^2)(x-y)^2}{1-(a-bK)^2}   \\
        &\le 2(q+rK^2)\beta\frac{|x-y|}{1-(a-bK)} + \frac{(q+rK^2)(x-y)^2}{1-(a-bK)} && \text{$a-bK < 1$}  \\
        &\le   \frac{2(q+rK^2)\left(|x| + \frac{\bar{w}}{\epsilon}\right)|x-y| + (q+rK^2)(x-y)^2}{\epsilon }. \numberthis \label{eq:Jstartdiffpos}
    \end{align*}
    This is exactly the desired result of the second equation of Lemma  \ref{unc_start_invariant}.
    \end{proof}

\section{General Lemmas}

The following four lemmas are used throughout the appendix and follow directly from results in \cite{schiffer2024stronger}.
    \begin{lemma}[Lemma \refgen{lemma:subgaussian_tail}]\label{lemma:subgaussian_tail}
        Suppose $w_t$ for $t < T$ are sub-Gaussian and $F$ is an event such that $\P(F) = 1-o_T(1/T^{11})$. Then 
        \[
            \E[\max_{i \le t} w_i^2 \mid \neg F]\P( \neg F) = o_T\left(\frac{1}{T^{10}}\right).
        \]
    \end{lemma}
    \begin{lemma}[Lemma \refgen{offbyepsilon_exp}]\label{offbyepsilon_exp}
 Let $x,y$ be two random variables independent of noises $W' = \{w_{i}'\}_{i=0}^{t-1}$ such that for some $L = \tilde{O}_T(1)$, both $\P(|x| \ge L)\E[x^2 \mid |x| \ge L] = o_T\left(\frac{1}{T^{10}}\right)$ and $ \P(|y| \ge L)\E[y^2 \mid |y| \ge L] = o_T\left(\frac{1}{T^{10}}\right)$ and $\P(|x| \le 4\log^2(T)) = 1-o_T(1/T^{11})$ and $\P(|y| \le 4\log^2(T))=1-o_T(1/T^{11})$. Then under Assumptions \ref{assum:constraints}--\ref{assum:initial}, if $\norm{\theta - \theta^*}_{\infty} = \epsilon \le  \epsilon_{\mathrm{L}\ref{parameterization_assum3}}$, then for any $K \in (\frac{a-1}{b}, \frac{a}{b})$ and $t \le T$,
    \begin{align*}
        &\left|\E\left[t \cdot J(\theta^*,C_{K}^\theta,t,x, W') - t \cdot J(\theta^*,C_{K}^\theta,t,y, W')\right] \right| =  \tilde{O}_T \left(\E[|x-y|]+\epsilon  + \frac{1}{T^2}\right). \numberthis \label{eq:ofbyepsilon2}
    \end{align*}
\end{lemma}
\begin{proof}
    This follows directly by Lemma \refgen{offbyepsilon_exp} and Lemmas \ref{parameterization_assum2} and \ref{parameterization_assum3}.
\end{proof}

            \begin{lemma}[Lemma \refgen{bounded_pos_cont}]\label{bounded_pos_cont}
    Let $x_0,x_1,...x_T$ be the sequences of states when starting at state $x_0 = x$ and using controller $C_t$ at time $t$. Suppose that the control $C_t(x_t)$ is safe for dynamics $\theta_t$ and $\norm{\theta_t - \theta^*} \le \frac{1}{\log(T)}$ for all $t < T$. For sufficiently large $T$,
    \[
        \forall t \le T, \: |x_t| = O_T(|x| + \norm{D}_{\infty} + \max_{i \le t-1} |w_i|).
    \]
    \[
        \forall t < T, \: |C_t(x_t)| = O_T(|x| + \norm{D}_{\infty} + \max_{i \le t-1} |w_i|).
    \]
\end{lemma}

\begin{lemma}[Lemma \refgen{bounded_approx}]\label{bounded_approx}
    Let $|x_0| \le 4\log^2(T)$. Suppose for all $t < T$, the control used by controller $C_t$ at time $t$ is safe for fixed dynamics $\theta_t$ and for all $t \le T$,
    \begin{equation}\label{eq:bounded_approx1}
        \norm{\theta^* - \theta_t}_{\infty} \le \frac{1}{\log(T)}.
    \end{equation}
    Then under Assumptions \ref{assum:constraints}--\ref{assum:initial},  for sufficiently large $T$ and conditioned on event $E_1$, using this controller $C_t$ with dynamics $\theta^*$ for $T$ steps starting at $x_0$ will give states ($x_0,...,x_{T}$) and controls ($u_0,...,u_{T-1}$) satisfying the following equations.
    \begin{equation}\label{eq:bounded_approx2a}
        |x_t|\stackrel{\text{a.s.}}{\le} 4\log^2(T) < \log^3(T) := B_x
    \end{equation}
    \begin{equation}\label{eq:bounded_approx2b}
        |u_t| \stackrel{\text{a.s.}}{\le} O_T(\log^2(T)) < \log^3(T) := B_x.
    \end{equation}
    Furthermore, if $x_0$ and the controller $C_t$ are deterministic, then the  states ($x_0,...,x_{T}$) and controls ($u_0,...,u_{T-1}$) satisfy
\begin{equation}\label{eq:bounded_approx2c}
        \E[|x_t|] \le 4\log^2(T) < \log^3(T) := B_x
    \end{equation}
    \begin{equation}\label{eq:bounded_approx2d}
        \E[|u_t|] \le O_T(\log^2(T)) < \log^3(T) := B_x.
    \end{equation}

\end{lemma}
\begin{equation}\label{eq:e1bound}
    \P(E_1) \ge 1 - \sum_{t=0}^{T-1} 2\exp\left(-\log^4(T)/\alpha\right) = 1 - o_T\left(\frac{1}{T^{\log(T)}}\right).
\end{equation}

\end{document}